\documentclass[lettersize,journal]{IEEEtran}
\usepackage{amsmath,amsfonts}
\usepackage{array}
\usepackage{caption}        %
\usepackage{subcaption}     %
\usepackage{textcomp}
\usepackage{stfloats}
\usepackage{url}
\usepackage{verbatim}
\usepackage{graphicx}
\usepackage{booktabs}
\usepackage{multirow}
\usepackage{multicol}
\usepackage{array}
\usepackage{caption}
\usepackage{makecell}
\usepackage{soul}
\usepackage{xcolor,colortbl}
\usepackage[shortlabels]{enumitem}
\usepackage{diagbox}
\def\BibTeX{{\rm B\kern-.05em{\sc i\kern-.025em b}\kern-.08em
    T\kern-.1667em\lower.7ex\hbox{E}\kern-.125emX}}
\usepackage{balance}
\usepackage{xcolor}
\usepackage{bm}
\usepackage{algorithm}
\usepackage{algpseudocode}
\usepackage{dsfont}
\usepackage{graphicx}
\usepackage{pifont}
\usepackage{bbding}
\usepackage[table]{xcolor}
\usepackage{float}
\usepackage{subcaption}
\usepackage{amsthm}
\theoremstyle{definition}

\newtheorem{theorem}{Theorem}

\definecolor{dark_green}{RGB} {0, 140, 0}

	\newif\ifcomments
	
	\commentstrue

	\ifcomments
\newcommand{\MB}[1]{\textcolor{red}{#1}}
\newcommand{\fjw}[1]{\textcolor{blue}{#1}}
\newcommand{\BT}[1]{\textcolor{orange}{#1}}
\newcommand{\BTcomm}[1]{\textcolor{dark_green}{#1}}

	\else
\newcommand{\MB}[1]{}
\newcommand{\fjw}[1]{}
\newcommand{\BT}[1]{}
\newcommand{\BTcomm}[1]{}

	\fi
	
\definecolor{ForestGreen}{RGB}{34,139,34}

\begin{document}

\title{Efficient, Robust, and Anti-Collusion Fingerprinting of Image Diffusion Models}

\author{Jianwei Fei,~\IEEEmembership{Member,~IEEE}, Yunshu Dai,~\IEEEmembership{Student Member,~IEEE} Zhihua Xia,~\IEEEmembership{Member,~IEEE}, \\
Xiaochun Cao,~\IEEEmembership{Senior Member,~IEEE}, Jiantao Zhou,~\IEEEmembership{Senior Member,~IEEE}, Alessandro Piva,~\IEEEmembership{Fellow,~IEEE}, \\
and Benedetta Tondi,~\IEEEmembership{Senior Member,~IEEE}
\thanks{
Jianwei Fei and Alessandro Piva are with the University of Florence, Florence, Italy (e-mail: fei\_jianwei@163.com; alessandro.piva@unifi.it);
Yunshu Dai and Xiaochun Cao are with the Shenzhen Campus of Sun Yat-sen University, Shenzhen, China (e-mail: daiyunshu0102@163.com; caoxiaochun@mail.sysu.edu.cn);
Zhihua Xia is with the College of Cyber Security, Jinan University, Guangzhou, China (e-mail: xia\_zhihua@163.com);
Jiantao Zhou is with the State Key Laboratory of Internet of Things for Smart City, and also with the Department of Computer and Information Science, Faculty of Science and Technology, University of Macau (e-mail:  jtzhou@um.edu.mo);
Benedetta Tondi is with the University of Siena, Siena, Italy (e-mail: benedetta.tondi@unisi.it).

This work was supported in part by Macau Science and Technology Development Fund under 001/2024/SKL, 0119/2024/RIB2 and 0110/2025/R1B2; in part by Research Committee at University of Macau under MYRG-CRG2025-00031-FST and MYRG-GRG2025-00086-FST; in part by the Guangdong Basic and Applied Basic Research Foundation under Grant 2024A1515012536. This work was also supported in part by the National Natural Science Foundation of China under Grant 625B2187, U23B2023, and 62472199, in part by Guangdong S\&T Program under 2026B0101100003, by Guangdong Key Laboratory of Data Security and Privacy Preserving under 2023B1212060036, Guangdong Hong Kong Joint Laboratory for Data Security and Privacy Protection under 2023B1212120007, by the basic and Applied Basic Research Foundation of Guangdong Province under 2025A1515011097, and by the Outstanding Youth Project of Guangdong Basic and Applied Basic Research Foundation under 2023B1515020064.

Corresponding author: Jiantao Zhou
}
}


\maketitle                                                                                                                                                                                                                                       

\begin{abstract}
Model fingerprinting, embedding user-specific identifiers (fingerprints) into generated outputs, has recently emerged as a popular solution to protect the intellectual property rights (IPR) of generative text-to-image (T2I) models and prevent unauthorized redistribution. In this work, we reveal a previously unexplored systematic vulnerability in existing generative model fingerprinting methods: they lack robustness against collusion attacks, where multiple attackers combine their models to remove or obscure the fingerprints. To address this issue, we take the first step towards a robust fingerprinting method for T2I models with anti-collusion capabilities. The proposed method encodes strings of bits, namely fingerprints, into the coefficients of a personalized normalization module (PNM) incorporated into T2I models, so that fingerprints can be reliably recovered from any generated image. To defend against collusion attacks and prevent unauthorized model redistribution, we introduce an anti-collusion mechanism based on lossless function-invariant parameter transformations. This mechanism significantly degrades the image generation quality of colluded models, making them effectively unusable. Moreover, our method allows developers to efficiently create multiple copies of fingerprinted T2I models by reparameterizing the PNM without the need for retraining. We also introduce a worst-case optimization strategy to improve robustness against model-level attacks. Our experiments demonstrate that the proposed method achieves high fidelity and robustness across multiple T2I image generation and editing tasks, with fingerprint extraction accuracy exceeding 99.5\%. Compared with existing methods, our method demonstrates, for the first time, a notable proactive robustness to collusion attack by significantly increasing the FID of colluded models.
\end{abstract}


\begin{figure}[h]
\centering
    \includegraphics[width=0.48\textwidth]{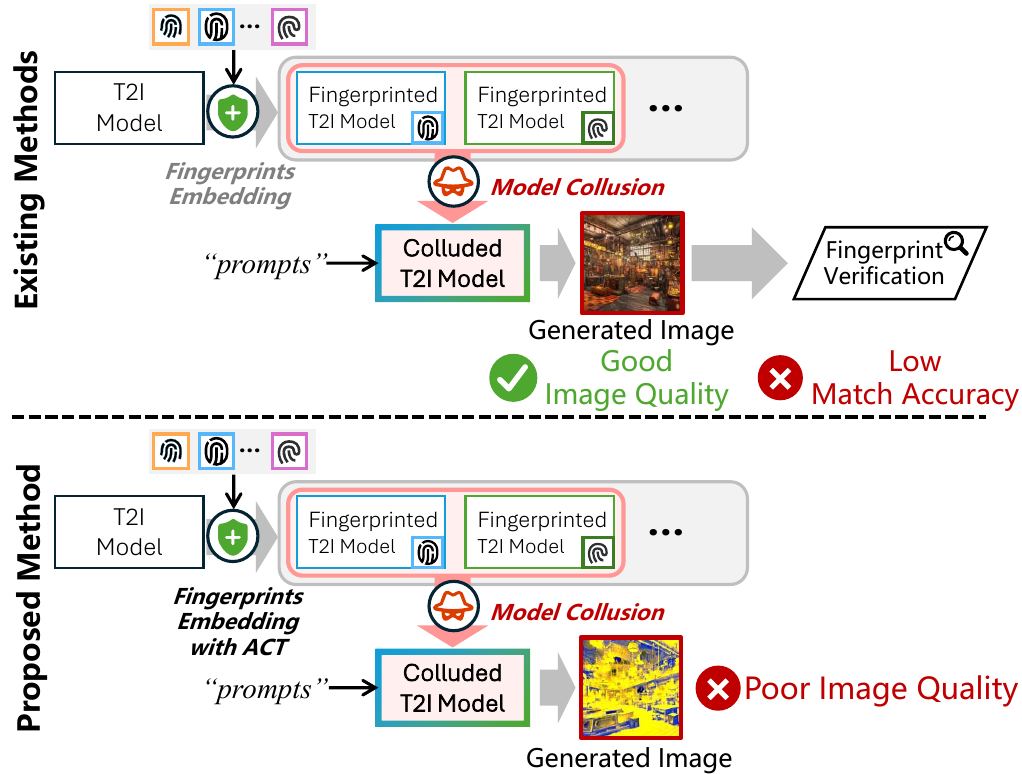}
    \caption{
    T2I model fingerprinting under collusion attacks. While existing methods allow colluders to obtain functional models that retain good image quality but exhibit low fingerprint matching accuracy, our proposed method effectively disables the collusion. By integrating the Anti-Collusion Transformation (ACT), any attempt to collude results in bad image quality, rendering the colluded model unusable.
    }
\label{fig: compare_0}
\end{figure}

\section{Introduction}
\label{sec: Introduction}

Recent progress in deep image generative models, particularly latent-diffusion-based text-to-image (T2I) models~\cite{rombach2022high}, has achieved remarkable breakthroughs in realism. These models enable high-quality text-guided image generation and editing, thus driving widespread application and leading to a continuous evolution of commercial tools~\cite{nichol2021glide,saharia2022photorealistic,rombach2022high}. As these models are increasingly integrated into commercial products and services, the substantial resources required for their development render them intellectual property (IP) with valuable economic worth.  Consequently, there is a strong demand for IP rights (IPR) protection against infringements such as unauthorized model redistribution.

To this end, recent studies have proposed generative model watermarking~\cite{yu2021artificial,fei2022supervised,fernandez2023stable,fei2024wide} and fingerprinting~\cite{fei2023robust,kim2024wouaf} to support IPR protection and attribution. Watermarking methods primarily fine-tune the model with additional optimization objectives to ensure that every image generated by the watermarked model contains a prescribed watermark for IPR protection. 
In contrast, fingerprinting methods embed user-specific identifiers, i.e., fingerprints, in different model copies to ensure user traceability. Fingerprinting methods can efficiently produce individualized fingerprinted models for each user~\cite{yuresponsible,fei2023robust,kim2024wouaf} without requiring retraining. In modern commercial settings, models are often developed and distributed as licensed copies to multiple users. Therefore, fingerprinting methods are particularly suitable for IPR protection and responsibility attribution due to their flexibility.

However, distributing distinct fingerprinted model copies to users introduces the collusion vulnerability. Due to the need for data privacy and local customization, users may have access to these distributed model weights. This empowers two or more malicious users to easily combine their models to produce a colluded copy that either contains no registered fingerprint~\cite{boneh1998collusion} or is erroneously attributed to non-participating parties. A simple yet highly realistic and effective collusion strategy is to average the parameters of the involved models~\cite{wortsman2022model}. Malicious users can exploit this parameter-level collusion to completely wipe out fingerprints with zero computational cost, all while preserving the model's generation capabilities.
As noticed in recent works~\cite{fernandez2023stable} and verified in Sec.~\ref{Sec: Anti-collusion Evaluations}, the collusion attack yields a colluded model that can generate high-quality images, while the extracted fingerprint shows low matching accuracy with any individual colluder. To the best of our knowledge, this threat has not been addressed so far in existing generative model fingerprinting research. 
%

To address this limitation, we propose, for the first time, a T2I model fingerprinting framework explicitly designed to resist collusion attacks. Our method ensures that, even when fingerprinted models are colluded, the resulting model fails to generate high-quality images, effectively preventing unauthorized use. Specifically, we introduce a personalized normalization module (PNM) for model fingerprinting, inspired by Fei~\textit{et al.}~\cite{fei2023robust}, who proposed the use of personalized normalization layers for fingerprinting generative adversarial networks (GANs). The PNM is integrated into the decoder of the variational autoencoder (VAE) of a pre-trained T2I model. The parameters of the normalization layer of the PNM are generated by fingerprint encoders conditioned on the assigned fingerprints. The VAE decoder is then fine-tuned to ensure that the generated images encode the fingerprints while maintaining high image quality. Simultaneously, a fingerprint decoder is trained to extract the fingerprints from the images generated by the VAE. During deployment, each user is assigned a unique fingerprint, from which the PNM parameters are generated via the fingerprint encoders, to create a personalized model copy without retraining or fine-tuning.

To defend against collusion attacks, we design a lossless anti-collusion transformation (ACT) that is integrated into the proposed PNM. The ACT applies user-specific transformations to the parameters of the PNM layers, ensuring that the collusion of two or more models compromises the image generation capability of the colluded model. Specifically, three operations, namely, parameter permutation, scaling, and sign flip, are applied to modify the PNM parameters in a user-dependent manner, while preserving the original output. 
These transformations guarantee that when multiple models are colluded, the image generation capability of the colluded model is disrupted, as shown in Fig.~\ref{fig: compare_0}. 
To further enhance fingerprint robustness against model-level modifications such as pruning and fine-tuning, we introduce an optimization strategy that minimizes fingerprinting loss under the worst-case parameter perturbation. At each fine-tuning iteration, the model is optimized over parameter perturbations in a neighborhood that maximizes the fingerprinting loss. Experimental results show that our method significantly increases the Fréchet Inception Distance (FID). For example, under a 2-party collusion attack on the COCO dataset using Stable Diffusion v2, the FID rises from 23 to 79, substantially reducing the usability of colluded models. Moreover,  our method maintains good performance under model-level and image-level attacks, achieving a high fingerprint extraction accuracy even when the generated images suffer from severe quality degradation.

The contributions of our work can be outlined as follows.
\begin{itemize}[itemsep=2pt,topsep=0pt,parsep=0pt]
    \item We develop a fingerprinting method for T2I models that enables distributors to efficiently create model copies with distinct fingerprints without additional training. To enhance robustness against collusion attacks, we propose the ACT, ensuring that colluding multiple models disrupts the generation capability of the colluded model. 
    \item We propose a worst-case embedding strategy that increases the fingerprint robustness to model-level attacks, particularly against fine-tuning.
    \item We introduce a rigorous framework for fingerprint verification and use it to evaluate the performance, considering two different scenarios: i) the claim-based verification scenario, where a claim on the specific fingerprinted model/user has to be verified; and ii) the identification scenario, where ownership is established and then 
    the specific model/user is identified.
    \item We conduct extensive experiments on both T2I generation and text-based image editing tasks, covering diverse datasets, COCO and ImageNet for generation, MagicBrush and InstructPix2Pix for editing, with multiple diffusion-based models. The results demonstrate that our method consistently improves the state-of-the-art in terms of security against collusion attacks as well as robustness to both image-level and model-level attacks.
\end{itemize}

\section{Related Work}
\label{sec: Related Work}

\subsection{Generative Model Watermarking}
\label{sec: Generative Model Watermarking}
Watermarking has been adopted as a solution to protect the IPR of deep neural network (DNN)~\cite{li2021survey,fan2021deepipr,zhang2021deep,li2022fedipr}. Generative model watermarking is a branch of DNN watermarking that focuses on generative models, including GANs~\cite{yu2021artificial,fei2022supervised,fei2024wide} and latent diffusion models (LDMs)~\cite{fernandez2023stable}. 
Most methods were designed for a black-box and, in particular, a box-free watermarking scenario, which represents the most interesting category. In black-box watermarking, the watermark is embedded in the behavior of the model output in correspondence to some specific inputs, referred to as triggers~\cite{zhao2023recipe,yuan2024watermarking}. The watermark is then extracted by querying the model through its API in a black-box setting. In box-free watermarking, the model is modified so that every output contains an extractable watermark, enabling ownership verification without accessing the model. 
Note that, unlike generative models designed for image watermarking~\cite{baluja2019hiding}, the carrier of generative model watermarking methods is the model parameters, consequently, their objectives and methodologies differ.
Yu~\emph{et al.}~\cite{yu2021artificial} introduced the first multi-bit, box-free GAN watermarking method via dataset watermarking. This method embedded a watermark into the training dataset so that any GAN trained with the watermarked dataset produced images containing the watermark. Zhao~\emph{et al.}~\cite{zhao2023recipe} extended the dataset watermarking approach to diffusion models and also found that the diffusion model trained on a watermarked dataset generated images containing the watermark. Fei~\emph{et al.}~\cite{fei2022supervised} proposed the first supervised box-free GAN watermarking method by fine-tuning GANs with an additional loss function using a pre-trained watermarking decoder and a predefined watermark. Similarly, Lin~\emph{et al.}~\cite{lin2024cyclegan} proposed an improved bi-directional supervised embedding strategy for image translation models, which integrated a pre-trained watermarking decoder to supervise the watermark embedding for both the source domain and the translated domain. Zhang~\emph{et al.}~\cite{zhang2024robust} proposed a structure-consistent watermarking method for image processing networks, which aligned watermarks with image structures to improve robustness against augmentation attacks. Recent works have focused on diffusion models. In~\cite{fernandez2023stable}, the supervised embedding approach was extended to LDM~\cite{fernandez2023stable}, by fine-tuning the VAE component. In addition, some works also proposed to fine-tune the U-Net~\cite{min2024watermark,feng2024aqualora}. 
Beyond 2D images, pre-trained watermarking decoders were also used to watermark 3D generative models for IPR protection~\cite{jang2024waterf,song2025protecting,luo2025nerf}.

\subsection{Generative Model Fingerprinting}
\label{sec: Generative Model Fingerprinting}
Beyond watermarking, recent research has explored methods for model fingerprinting. The goal of fingerprinting is to insert unique identifiers (fingerprints) into copies of a model, thus creating different fingerprinted copies that can be distributed to different users. The user-specific fingerprint can then be used for user traceability and to track unauthorized re-distribution~\cite{luh2006digital}.
In recent work, fingerprints were used to modulate the parameters of specific/purposely added layers~\cite{yuresponsible,kim2024wouaf,fei2025omnimark} or architectures~\cite{fei2023robust,fei2025scalable}.
For efficient generation, the creation of new fingerprinted model copies typically does not require re-training or fine-tuning, but only that the parameters are adjusted based on the fingerprints. Yu~\emph{et al.}~\cite{yuresponsible} introduced a pioneering GAN fingerprinting method. The method embedded binary fingerprints into model parameters via weight modulation~\cite{karras2019style}. The GAN was fine-tuned to enable fingerprint extraction from generated images via a decoder. This method can efficiently produce new fingerprinted models by altering the fingerprint and adjusting the parameters via modulation. A similar method for diffusion models by modulating the weights of the VAE decoder has been proposed in~\cite{kim2024wouaf}. Fei~\emph{et al.}~\cite{fei2023robust} also developed a retraining-free GAN fingerprinting method that resorted to a personalized normalization layer. The parameters of the normalization layer were modulated by the fingerprints to create model copies. 
Feng~\emph{et al.}~\cite{feng2024aqualora} introduced a LoRA-based fingerprinting method for T2I models that also enabled flexible fingerprint updates.

\subsection{Collusion Attacks Challenges}
\label{sec: Challenges Against Collusion Attacks}
The collusion attack, where two or more colluders combine their fingerprinted copies to remove the fingerprint~\cite{feng2010optimal}, is a well-known threat in traditional media fingerprinting~\cite{zhao2005forensic}. This threat also exists in the context of model fingerprinting, and recent studies have demonstrated that model collusion can degrade fingerprint verification performance~\cite{fernandez2023stable,fei2024wide}, while retaining high-quality image generation.

Motivated by this, we aim to address the following unexplored problem: \textbf{Can collusion attacks be proactively prevented in the context of generative model fingerprinting?} Traditional image fingerprinting methods rely on traitor-tracing codes or group testing~\cite{wu2004collusion,vskoric2015tally} to identify colluders. 
However, applying these methods to generative model fingerprinting has two limitations:
(1) the limited bit capacity of existing model fingerprinting methods might be insufficient to guarantee the redundancy required for traitor-tracing encoding; and, more importantly, (2) traitor-tracing schemes are designed for post-hoc attribution, and can be used to identify colluders only {\em after} a colluded model has been created. They also require the application of a tracing algorithm.

In this work, we address for the first time the collusion attack threat in generative model fingerprinting. We propose a fundamentally different solution that leverages the new paradigm of generative model fingerprinting, where fingerprints are embedded in the functionals (i.e., models) rather than directly in the signals (i.e., generated images) as in traditional media fingerprinting~\cite{barni2021dnn}. Our method proactively mitigates collusion attacks by disabling model functionality, rendering any colluded model unusable.

\section{Preliminaries}
\label{sec: Preliminaries}

\subsection{Text-to-Image Diffusion Models}
\label{sec: Text to Image Diffusion Models}
Diffusion models are a class of generative models capable of synthesizing high-quality images by gradually denoising random noise~\cite{ho2020denoising,croitoru2023diffusion}. These models can accept various conditioning inputs, such as textual descriptions (prompts) and other constraints, enabling text-to-image and controlled image generation~\cite{zhang2023adding,mou2024t2i}. 
Early methods perform the diffusion process in the pixel space~\cite{songscore,ho2020denoising,ho2022cascaded,bie2024renaissance} with U-Net as the denoising backbone, which is computationally demanding.

To improve generation efficiency, Rombach~\textit{et al.}~\cite{rombach2022high} introduced the LDM, which performs denoising in the latent space of a pre-trained VAE using the U-Net backbone. 
During inference, the reverse diffusion process is initialized from pure noise in the latent space and progressively denoises it to obtain the latent representation of the image, denoted as $\bm{z}_0$, conditioned on the input prompt. The final image is subsequently generated by passing $\bm{z}_0$ through the VAE decoder, i.e., $\hat{\bm{x}} = \mathcal{D}(\bm{z}_0)$.
Due to the superior scalability and efficiency, LDM has become the mainstream framework for modern diffusion models~\cite{podellsdxl}, including state-of-the-art variants that replace the U-Net backbone with DiT~\cite{peebles2023scalable}. Therefore, our work focuses on latent diffusion models, which, without exception, rely on a VAE for final image synthesis.

\subsection{Fingerprint Verification}
\label{sec: Fingerprint Verification}
We consider a scenario where a model provider distributes distinct fingerprinted generative models (copies) to $M$ users. Each user $i \in \{1, \ldots, M\}$ is associated with a unique fingerprint $\bm{m}^{(i)}$.
The distributor uses the user-specific fingerprint $\bm{m}^{(i)}$ to instantiate a new model $\mathcal{M}^{(i)}$ that is released to user $i$. 
The goal is to verify whether a suspicious image $\bm{x}$ is generated by a claimed fingerprinted model, or by one of the $M$ models of the distributor, and identify the responsible model.\footnote{A similar approach can be used to verify the identity of a suspicious model, assuming that the model can be queried to generate at least one image.}
A binary fingerprint $\bm{m}$ is extracted from the image $\bm{x}$ via the fingerprint decoder $\mathcal{W}$, i.e., $\bm{m} = \mathcal{W}(\bm{x})$. 
We define the bit-wise matching accuracy (Bit Acc) between the extracted fingerprint and a given fingerprint $\bm{m}^{(i)}$ as:\footnote{To keep the notation light, we omit the explicit dependence on $\bm{x}$ in $p^{(i)}$ (and also $\mathcal{S}$, defined in Eq. \eqref{eq.S}).}
\begin{equation}
\label{eq: fingerprint verification}
p^{(i)} = \text{Acc}(\bm{m},\bm{m}^{(i)}) = \frac{1}{d} \sum_{j=1}^{d} \mathds{1}\left(\bm{m}_j = \bm{m}^{(i)}_j\right),
\end{equation}
where $d$ is the number of bits, $\mathds{1}(\cdot)$ is the indicator function that outputs 1 if the condition is true and 0 otherwise. A threshold $\tau\in [0,1]$ is used to determine whether a match is considered valid. We consider two fingerprint verification scenarios detailed in the following:

\textbf{1) Claim-based verification}: “{\em Was $\bm{x}$ generated by the claimed user/model?}”. 
In this scenario, a claim is made about the model that generated the image, i.e., model $\mathcal{M}^{(\text{claim})}$, or equivalently, the model associated with fingerprint $\bm{m}^{(claim)}$.
The binary hypothesis testing problem can be defined as:
\begin{itemize}
  \item $H_1$: The image was generated by the claimed model.
  \item $H_0$: The image was generated by a non-fingerprinted model or a different fingerprinted model.
\end{itemize}
We verify the claim by comparing the extracted fingerprint $\bm{m}$ with $\bm{m}^{(claim)}$.
If the match is above $\tau$, i.e., $\text{Acc}(\bm{m}, \bm{m}^{(\text{claim})}) > \tau$, the claim is accepted and the image is attributed to $\mathcal{M}^{(\text{claim})}$ (decision for $H_1$). To evaluate performance, we compute the true positive rate (TPR) and false positive rate (FPR) as: 
\begin{equation}
\label{eq: detection}
\begin{aligned}
  \text{TPR} &= 
  N(p^{(\text{claim})} > \tau \mid  H_1)/N(H_1), \\
  \text{FPR} &= N( p^{(\text{claim})} > \tau \mid H_0)/N(H_0),
\end{aligned}
\end{equation}
where $N(\cdot)$ denotes the number of samples $\bm{x}$ satisfying the the specified event.
In practice, $\tau$ is chosen to achieve a desired TPR while keeping FPR below a desired threshold.
$\tau$ is estimated empirically using a validation dataset containing both positive and negative samples.

\textbf{2) Identification (with rejection)}: 
“{\em Was $\bm{x}$ generated from one of the models from the distributor? If yes, which specific model generated it?}”
In this scenario, no claim is provided. The authority aims to determine {\em if any} of the fingerprinted models generated the image $\bm{x}$, and, in case of a positive answer, which is the specific model.
The problem of detecting whether the image comes from a model of the distributor can be modeled as a binary hypothesis test:
\begin{itemize}
  \item $H_1$: $\bm{x}$ was generated by one of the fingerprinted models.
  \item $H_0$: $\bm{x}$ comes from a non-fingerprinted model.
\end{itemize}
If $H_1$ is accepted, the image is attributed to the model with the highest Bit Acc. More formally, let us define the set of candidate models $\mathcal{M}^{(\text{i})}$, i.e., those models whose fingerprints match the extracted fingerprint $\bm{m}$ for a given threshold $\tau$:
\begin{equation}
\begin{aligned}
\mathcal{S} = \left\{\mathcal{M}^{(\text{i})}, i \in \{1, \ldots, M\} \mid p^{(i)} > \tau \right\}.
\end{aligned}
\label{eq.S}
\end{equation}
If $\mathcal{S} \neq \emptyset$, a decision for $\mathcal{H}_{1}$ is made. The image is attributed to the model in $\mathcal{S}$ with the highest matching accuracy:
\begin{equation}
\begin{aligned}
\mathcal{M}^{(i^*)} \text{, with }  i^* = \arg\max_{\mathcal{M}^{(i)} \in \mathcal{S}} p^{(i)}.
\end{aligned}
\end{equation}
The performance of identification with rejection can be evaluated using the following metrics: 

1) the TPR and FPR of the binary hypothesis test (i.e., the rejection test), where
\begin{equation}
\small
\label{eq: identification}
\begin{aligned}  
& \text{TPR}  = \frac{N\left(\mathcal{S} \neq \emptyset \mid H_1 \right)}{N(H_1)}, \quad
\text{FPR}  = \frac{N\left(\mathcal{S} \neq \emptyset \mid H_0 \right)}{N(H_0)}.
\end{aligned}
\end{equation}

2) the Correct Identification Rate (CIR), measuring the percentage of images from fingerprinted models ($H_1$) for which $H_1$ is accepted and which are attributed to the correct model. Formally:
\begin{equation}
\label{eq: identification2}
\begin{aligned}
& \text{CIR}  = N \left( (\mathcal{M}^{(t)} \in \mathcal{S})
\wedge
(p^{(t)} > p^{(i)},\ 
\forall i \neq t)
~\middle|~ H_1 \right) / N(H_1),
\end{aligned}
\end{equation}
where $\mathcal{M}^{(t)}$ denotes the ground-truth fingerprinted model.
In the following, we refer to these two scenarios as fingerprint verification and identification, respectively.

\begin{figure*}[ht]
\centering
    \includegraphics[width=0.98\textwidth]{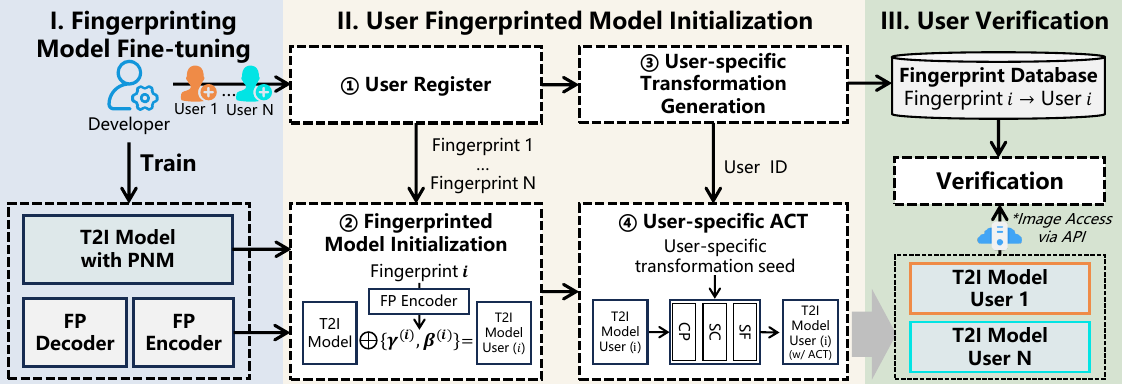}
    \caption{Application workflow of the proposed framework.}
\label{fig: overview2}
\end{figure*}


\section{Methodology}
\label{sec: Methodology}
\subsection{Framework Overview}
The workflow of the proposed framework, as shown in Fig.~\ref{fig: overview2}, is structured in 3 phases: Fingerprinting Model Fine-tuning, User Fingerprinted Model Initialization, and User Verification. 

\subsubsection{Phase I: Fingerprinting Model Fine-tuning}
The developer first integrates the Personalized Normalization Module (PNM) into the T2I model. The fingerprint encoder and decoder are jointly trained with the model to ensure accurate embedding and extraction. Detailed training objectives and loss functions are provided in \textbf{Sec.~\ref{sec: Method Overview}}.

\subsubsection{Phase II: User Fingerprinted Model Initialization}
\label{sec:Phase II}
Once the components are trained, the developer can efficiently instantiate uniquely fingerprinted model copies without retraining. As illustrated in the middle of Fig.~\ref{fig: overview2}, this process involves 4 sequential steps: 1) \textbf{User Registration}, where a user is assigned a unique ID and a binary fingerprint $\text{FP}_i$; 2) \textbf{Model Initialization}, where $\text{FP}_i$ is mapped to PNM parameters via the pre-trained encoders (details in \textbf{Sec.~\ref{sec: User Fingerprinted Model Initialization}}); 3) \textbf{Transformation Generation}, where user-specific ACT keys are deterministically derived from the User ID; and 4) \textbf{ACT Application}, where these transformations are applied to the PNM parameters to proactively prevent collusion. The final model is then released to the user.

\subsubsection{Phase III: User Verification}
When a suspicious image is encountered, the model owner uses the pre-trained \textit{Fingerprint Decoder} to extract the signature. This extracted fingerprint is then matched against the \textit{Fingerprint Database} to identify the specific user or detect potential collusion, following the verification protocols defined in \textbf{Sec.~\ref{sec: Fingerprint Verification}}.

\subsection{Fingerprint Model Fine-tuning}
\label{sec: Method Overview}

We illustrate the fine-tuning pipeline in the upper part of Fig.~\ref{fig: overview}.
To fingerprint the T2I model, the original VAE decoder $\mathcal{D}$ is modified by inserting an intermediate PNM, whose parameters $\bm{\gamma}$ and $\bm{\beta}$ are given by two (trainable) encoding networks, $\mathcal{F}_{\bm{\gamma}}$ and $\mathcal{F}_{\bm{\beta}}$, fed with the fingerprint message $\bm{m}$ as input. 
The modified decoder is fine-tuned (with the encoder $\mathcal{E}$ frozen) for fingerprint embedding, as detailed in Section \ref{sec: Fine-tuning of VAE decoder}.
The fingerprint decoder $\mathcal{W}$ is jointly trained to extract the fingerprints from the images generated by the VAE decoder. We denote the fingerprinted VAE decoder by $\mathcal{D}_m$, which replaces the original one $\mathcal{D}$ inside the T2I model, thus obtaining the fingerprinted T2I model that can be used for image generation. Note that the denoising U-Net and text encoder are not involved in the fine-tuning process.


\begin{figure}[ht]
\centering
    \includegraphics[width=0.48\textwidth]{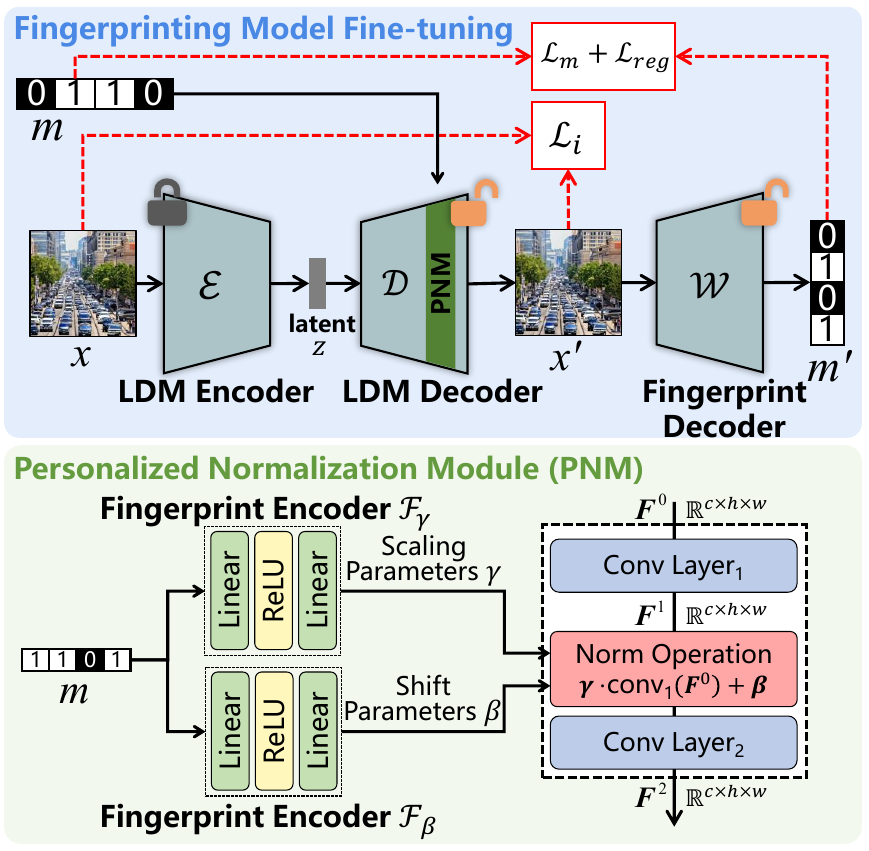}
    \caption{Pipeline of the proposed fine-tuning method.
    }
\label{fig: overview}
\end{figure}

\subsubsection{Personalized Normalization Module}
\label{sec: Personalized Normalization Module on VAE}
The details of the personalized normalization module (PNM) are shown in the lower part of Fig.~\ref{fig: overview}. The idea behind the use of personalized normalization is inspired by Fei~\emph{et al.}~\cite{fei2023robust}, where it was first introduced for GAN fingerprinting. 
In this work, we extend the idea to T2I models by designing a \textit{Conv–Norm–Conv} structured PNM, which enables the integration of the proposed ACT against collusion. The detailed operations performed inside our PNM are detailed below, while the details on the ACT are provided in Section \ref{sec: User Fingerprinted Model Initialization}.

Given a fingerprint message $\bm{m} \in \{0,1\}^{d}$, the outputs of the two encoding networks $\mathcal{F}_{\bm{\gamma}}$ and $\mathcal{F}_{\bm{\beta}}$ are used for the normalization operation, providing, respectively, the scaling and shift parameters.
In addition, two $1\times1$ convolution layers (without bias parameters) are placed before and after the normalization layer. 
Let $\bm{F}^{(0)} \in \mathbb{R}^{c\times h \times w}$ be the input feature of PNM, where $c$, $h$, and $w$ represent the number of channels, height, and width, respectively. The output feature map $\bm{F}^{(2)}$ of the PNM is computed as:
\begin{align}
    \label{eq: pn module}
    \bm{F}^{(2)}  =  \text{Conv}_2( \bm{F}^{(1)} ) = \text{Conv}_2(  {\bm{\gamma}} \cdot \text{Conv}_1(\bm{F}^{(0)})+{\bm{\beta}}),
\end{align}
where $\text{Conv}_2$ and $\text{Conv}_1$ denote the convolution operation, and $\cdot$ denotes element-wise multiplication. $\bm{\gamma} = \mathcal{F}_{\bm{\gamma}}(\bm{m})$ and $\bm{\beta} = \mathcal{F}_{\bm{\beta}}(\bm{m})$ are vectors in $\mathbb{R}^c$.
To enable element-wise operations with the intermediate feature map $\text{Conv}_1(\bm{F}^{(0)}) \in \mathbb{R}^{c \times h \times w}$, both $\bm{\gamma}$ and $\bm{\beta}$ are broadcast to match its shape. 
Specifically, broadcasting extends each $c$-dimensional vector into a 3D tensor of shape $\mathbb{R}^{c \times h \times w}$ by replicating the values of each channel across all spatial positions (i.e., along the height and width dimensions). 
The normalized intermediate feature map $\bm{F}^{(1)}$ is then passed through the second convolutional layer $\text{Conv}_2$ to produce the final output $\bm{F}^{(2)}$.

\subsubsection{Fine-tuning Objective}
\label{sec: Fine-tuning of VAE decoder}

In the following, we describe the fine-tuning loss objective in detail.
To instruct the network to embed a different fingerprint inside the output images for every different input message, we randomly sample $\bm{m}\in \{0,1\}^{d}$ at every step. We use the networks $\mathcal{F}_{\bm{\gamma}}$ and $\mathcal{F}_{\bm{\beta}}$ to get from $\bm{m}$ the scaling and shift parameters $\bm{\gamma}$ and $\bm{\beta}$ used in the PNM. Simultaneously, we randomly sample real images and encode them to get the latent representation $\bm{z}$ using the frozen VAE encoder $\mathcal{E}$, and use $\mathcal{D}_m$ to reconstruct the image from $\bm{z}$. The fingerprint decoder $\mathcal{W}$ is jointly trained to extract the fingerprint from $\mathcal{D}_m(\bm{z})$. The optimization includes an image reconstruction loss, a fingerprinting loss and a regularization loss for improved robustness.

\textbf{Image Reconstruction Loss}: The purpose of this loss is to maintain the visual quality of the images obtained from the fingerprinted $\mathcal{D}_m$. We used the original image loss employed during the training of the VAE~\cite{stabilityai2024sdvae}, including learned perceptual image patch similarity (LPIPS)~\cite{zhang2018unreasonable} and mean squared error (MSE) between input images $\bm{x}$ and the reconstructed images $\mathcal{D}_m(\mathcal{E}(\bm{x}))$.
This combination of these two loss terms ensures a balance between pixel-wise error and perceptual similarity~\cite{larsen2016autoencoding,dosovitskiy2016generating}. 
Specifically, LPIPS is measured by the high-level features extracted by a pre-trained VGG network $\phi$:
\begin{align}
\label{eq: lpips}
\mathcal{L}_{\text{LPIPS}} = \sum_l w_l \cdot \|\phi_l(\bm{x}) - \phi_l(\mathcal{D}_m(\mathcal{E}(\bm{x})))\|_2^2,
\end{align}
where $\bm{x}$ and $\mathcal{D}_m(\mathcal{E}(\bm{x}))$ are the input and reconstructed images, $\phi_l(\bm{x})$ represents the feature map of image $\bm{x}$ extracted from the $l$-th layer of $\phi$, and $w_l$ is a learned weight for layer $l$. Then, the image reconstruction loss is expressed as:
\begin{align}
    \label{eq:img_loss}
    \mathcal{L}_i = 
     \lambda_{\text{lpips}}\mathcal{L}_{\text{LPIPS}} +  \lambda_{\text{mse}} \left\| \mathcal{D}_m(\mathcal{E}(\bm{x})) - \bm{x} \right\|_2^2 .
\end{align}

\textbf{Fingerprinting Loss}: It instructs the fingerprint decoder $\mathcal{W}$ to extract the fingerprints from the images produced by $\mathcal{D}_m$. The binary cross-entropy (BCE) between the input fingerprint and the output of the fingerprint decoder is considered:
\begin{equation}
    \begin{aligned}
     \mathcal{L}_{m} = &
      \sum_{j=1}^{d} \bm{m}_j \log \sigma (\mathcal{W}(\mathcal{D}_m(\mathcal{E}(\bm{x})))_j)  \\ & + (1 -\bm{m}_j) \log (1-\sigma (\mathcal{W}(\mathcal{D}_m(\mathcal{E}(\bm{x})))_j)),
    \end{aligned}
    \label{eq:m_loss}
\end{equation}
where $\sigma(\cdot)$ denotes the sigmoid function. Minimizing $\mathcal{L}_{m}$ corresponds to minimizing the bit-wise error. Remind that $\mathcal{E}$ is frozen and only $\mathcal{D}_m$ and $\mathcal{W}$ are trained.

\textbf{Worst-Case Regularization Loss}:
To proactively enhance the robustness of the fingerprinted model against parameter perturbations, we propose a novel worst-case regularization Loss. Unlike standard optimization objectives that often lead to sharp minima sensitive to modifications, our designed loss term explicitly aims to encourage the model to converge to a wider, flatter minimum of the fingerprinting loss $\mathcal{L}_m$. In this way, the model should be less sensitive to small perturbations in the parameter space, such as those introduced during fine-tuning, thereby preserving the embedded fingerprints. To achieve this, we formulate a min-max optimization objective where we consider the perturbation of the parameters that maximizes the fingerprinting loss within a neighborhood. Formally, the proposed objective is defined as:
\begin{equation}
\begin{alignedat}{2}
&
\mathcal{L}_{\text{reg}} = \mathcal{L}_m\big(\bm{x}, \bm{m}; \theta_{\mathcal{D}_m} + \delta^*\big), \\
& \text{where } \delta^* = \underset{\delta}{\mathop{\arg\max}} \, \mathcal{L}_m\big(\bm{x}, \bm{m}; \theta_{\mathcal{D}_m} + \delta\big), \, 
\text{s.t. } \|\delta\| < \xi.
\end{alignedat}
\label{eq:Regularization}
\end{equation}
where $\theta_{\mathcal{D}_m}$ is the parameter vector of $\mathcal{D}_m$ and $\xi$ is an upper bound for the perturbation. $\delta^*$ denotes the worst-case perturbation. 
Since solving \eqref{eq:Regularization} exactly is computationally expensive, we propose to approximate $\delta^*$ via a single-step gradient ascent:
\begin{align}
\delta^* = \eta_1 \cdot \nabla_{\theta_{\mathcal{D}_m}} \mathcal{L}_m(\bm{x}, \bm{m}; \theta_{\mathcal{D}_m}),
\end{align}
where $\eta_1$ is a step size controlling the magnitude of the perturbation. While theoretically $\delta^*$ should be projected onto the $\xi$-sphere (i.e., $\xi \cdot \nabla \mathcal{L} / \|\nabla \mathcal{L}\|$), we found that a fixed step size $\eta_1$ acts as an effective proxy for the worst-case perturbation in practice and reduces computational overhead.
Specifically, at each fine-tuning step $i$, given the current parameters $\theta_{\mathcal{D}_m}^{(i)}$, we compute the perturbed parameters in the regularization loss as: 
\begin{align}
\theta_{\mathcal{D}_m}^{(i*)} =
\theta_{\mathcal{D}_m}^{(i)} + \delta^* = \theta_{\mathcal{D}_m}^{(i)} +\eta_1 \cdot \nabla_{\theta_{\mathcal{D}^{(i)}_m}} \mathcal{L}_m(\bm{x}, \bm{m}; \theta_{\mathcal{D}^{(i)}_m}).
\end{align}
This newly designed regularization term is minimized jointly with the other losses to update the model parameters $\theta_{\mathcal{D}_m}$. Compared with standard training, the regularization ensures that, under a worst-case perturbation of the VAE parameters, the fingerprinting loss remains low, thereby improving the robustness of the fingerprints against model parameter modification, introduced, for instance, via fine-tuning.

Therefore, the overall loss used to fine-tune the VAE is 
\begin{align}
\label{eq:total_loss}
\mathcal{L}(\bm{x},\bm{m};\theta_{\mathcal{D}}) = \mathcal{L}_{i} + \lambda_m \mathcal{L}_{m} + \lambda_{reg} \mathcal{L}_{reg},
\end{align}
where $\lambda_m$ and $\lambda_{reg}$ are the weights applied to each loss term. At the end of the fine-tuning procedure, the following four models are obtained: the fingerprinted VAE decoder $\mathcal{D}_m$, the fingerprint decoder $\mathcal{W}$, and fingerprint encoders $\mathcal{F}_{\bm{\gamma}}$ and $\mathcal{F}_{\bm{\beta}}$.


\subsection{User Fingerprinted Model Initialization with ACT}
\label{sec: User Fingerprinted Model Initialization}
Once a new user ($i$) comes, a model containing the user-specific fingerprint $\bm{m}^{(i)}$ is obtained by assigning to the normalization layer in the PNM layer the corresponding parameters obtained from the trained encoding networks, that is, $\bm{\gamma}^{(i)} = \mathcal{F}_{\bm{\gamma}}({{\bm{m}}^{(i)}})$ and $\bm{\beta}^{(i)} = \mathcal{F}_{\bm{\beta}}({\bm{m}}^{(i)})$. 
Furthermore, a user-dependent ACT is applied to the parameters of the PNM to enable robustness against collusion attacks (the details of the ACT are provided below). While the fingerprinted T2I model $\mathcal{M}^{(i)}$ is then distributed to the user, the two trained encoding networks $\mathcal{F}_{\bm{\gamma}}$ and $\mathcal{F}_{\bm{\beta}}$, used for setting user-dependent parameters of the PNM in $\mathcal{D}_m$, and the trained fingerprint decoder $\mathcal{W}$, used for the extraction of the fingerprint in the fingerprint verification phase, are not distributed.


\textbf{Objective and Motivation.}
Recent studies have shown that the parameter space of deep neural networks exhibits notable connectivity and flatness. Garipov~\emph{et al.}~\cite{garipov2018loss} and Draxler~\emph{et al.}~\cite{draxler2018essentially} demonstrated that independently trained models lie on a connected low-dimensional manifold. As a consequence of such mode connectivity, linear interpolation between two independently trained models often gets well-performing solutions. Izmailov~\emph{et al.}~\cite{izmailov2018averaging} further showed that averaging model parameters not only preserves model performance but often also improves generalization. These findings suggest that parameter interpolation or averaging is an effective way to maintain model fidelity. However, this property may facilitate collusion attacks, since attackers can exploit parameter-space connectivity to remove or forge embedded fingerprints. The proposed ACT allows for counteracting this by applying user-specific, function-preserving transformations to the model. By deliberately reshaping the parameter space while keeping the generation quality unchanged,  ACT breaks the mode connectivity that colluders rely on. The analysis carried out in Sec.~\ref{sec:Anti Collusion Analysis}  
provides empirical evidence of ACT’s effectiveness against collusion attacks.

\textbf{Pipeline.} Once $\mathcal{D}_m$ is trained, its PNM can be efficiently initialized with distinct fingerprints to produce different model instances, as shown in Sec.~\ref{sec:Phase II}. Before distributing the fingerprinted models to users, we apply a user-specific ACT to the PNM. As shown in Fig.~\ref{fig: overview2}, we combine three operations: parameter permutation, scaling, and sign flip. These transformations are applied sequentially and are parameterized by user-specific keys. Recall that the PNM is a Conv-Norm-Conv module, and $\bm{F}^{(0)}$ is the input, $\bm{F}^{(1)}$ and $\bm{F}^{(2)}$ are the intermediate and final outputs of PNM, respectively. We say a transformation preserves the function of the PNM if for all input $\bm{F}^{(0)}$, the output $\bm{F}^{(2)}$ remains unchanged. Below, we formally define each transformation. 
The theoretical guarantees of function preservation are provided in Appendix A.

\textbf{Channel-wise parameter Permutation (CP).}
CP aims to rearrange the parameters of the PNM layers without changing the output. 
Consider the PNM operation given by Eq.~\eqref{eq: pn module}, where $\bm{W}^{(1)}$ and $\bm{W}^{(2)}$ denote the kernels of $\text{Conv}_1$ and $\text{Conv}_2$, respectively. These are 4-dimensional tensors with shape $C \times C \times k\times k$, where we set the input and output channel dimensions equal to ensure PNM can be inserted at any position in $\mathcal{D}_m$.
Let $\bm{W}_i \in \mathbb{R}^{C \times k\times k}$ denote the $i$-th filter in kernel $\bm{W}$, and $\bm{W}_{i,j} \in \mathbb{R}^{k\times k}$ the $j$-th channel of the $i$-th filter. 
For input feature map $\bm{F}^{(0)}$ with $i$-th channel $\bm{F}^{(0)}_{i} \in \mathbb{R}^{h\times w}$, after $\text{Conv}_1$ and normalization, the $c$-th channel is:
\begin{align}
\label{eq:conv1}
\bm{F}^{(1)}_c = \bm{\gamma}_{c} \cdot \sum_{i=1}^{C} \bm{F}^{(0)}_{i}\ast \bm{W}_{c,i}^{(1)} + \bm{\beta}_{c},
\end{align}
where * represents the convolution operation, $\bm{\gamma}_{c}$ and $\bm{\beta}_{c}$ are the $c$-th normalization parameters.
After $\text{Conv}2$, the $c$-th output channel becomes:
\begin{align}
\label{eq:conv2}
\bm{F}^{(2)}_c = \sum_{j=1}^{C} \bm{F}^{(1)}_j \ast \bm{W}_{c,j}^{(2)}.
\end{align}
Let $[C] := \{1, 2, \dots, C\}$. The CP operation applies a permutation function $\pi: [C] \to [C]$ to rearrange parameters as:
\begin{equation}
\small
\label{eq:cp1}
\begin{aligned}
\footnotesize
& \widetilde{\bm{W}}_{i}^{(1)} := \bm{W}_{\pi(i)}^{(1)}, \quad 
  \widetilde{\bm{\gamma}}_{i}:=\bm{\gamma}_{\pi(i)}, \quad
  \widetilde{\bm{\beta}}_{i}:=\bm{\beta}_{\pi(i)}, \quad \forall i \in [C] \\
& \widetilde{\bm{W}}_{i,j}^{(2)}  := \bm{W}_{i,\pi(j)}^{(2)},  \quad \forall i,j \in [C]
\end{aligned}
\end{equation}
This permutation modifies filters in $\bm{W}^{(1)}$, channels in $\bm{W}^{(2)}$, and normalization parameters consistently. The permutation function $\pi$ is randomly generated for each user using their unique user ID as the random seed. This ensures reproducibility while providing distinct transformations across users.

\textbf{Parameter Scaling (SC).}
SC is the second transformation considered in the ACT framework. It modifies the magnitude of convolutional and normalization parameters in such a way that the overall function of the PNM remains unchanged. 
Specifically, given scaling vectors $\bm{\alpha}^{(1)}, \bm{\alpha}^{(2)}, \bm{\alpha}^{(3)}, \bm{\alpha}^{(4)} \in \mathbb{R}^C$, we define the scaled parameters of the PNM as follows:
\begin{equation}
\begin{aligned}
\label{eq:scal1}
\widetilde{\bm{W}}_{i}^{(1)} & = \bm{\alpha}_i^{(1)}  \cdot {\bm{W}}_{i}^{(1)}, \forall i,\\
\widetilde{\bm{W}}_{i,j}^{(2)} & = \bm{\alpha}_j^{(2)} \cdot  {\bm{W}}_{i,j}^{(2)}, \forall i,j,\\
\widetilde{\bm{\gamma}}_{i} & = \bm{\alpha}_i^{(3)} \cdot  {\bm{\gamma}}_{i}, \forall i,\\
\widetilde{\bm{\beta}}_{i} & = \bm{\alpha}_i^{(4)} \cdot  {\bm{\beta}}_{i}, \forall i,
\end{aligned}
\end{equation}
where $\bm{\alpha}^{(1)},\bm{\alpha}^{(2)},\bm{\alpha}^{(3)},\bm{\alpha}^{(4)} \in \mathbb{R}^{C}$, and such that 
\begin{equation}
\begin{aligned}
\label{eq:con1}
\bm{\alpha}_i^{(1)} \bm{\alpha}_i^{(2)} \bm{\alpha}_i^{(3)} = 1, 
\bm{\alpha}_i^{(2)} \bm{\alpha}_i^{(4)} = 1, \forall i.
\end{aligned}
\end{equation}
To obtain the scaling vectors, a random generator is used to sample $\bm{\alpha}^{(1)}$ and $\bm{\alpha}^{(2)}$ with a user-specific random seed. Then, $\bm{\alpha}^{(3)}$ and $\bm{\alpha}^{(4)}$ are derived accordingly to satisfy the constraints in Eq.~\eqref{eq:con1}. Hence, scaling is applied across the filter dimension of $\bm{W}^{(1)}$ and the channel dimension of $\bm{W}^{(2)}$, and to $\bm{\beta}$ and $\bm{\gamma}$, with the above constraints ensuring that the output does not change.

\textbf{Sign Flip (SF).}
The third transformation in the ACT framework is sign flip, which inverts the signs of parameters in the PNM while ensuring the outputs remain equivalent. By making the parameter sign-flipping user-dependent, this operation provides additional parameter obfuscation. 

The sign flip transformation applies element-wise sign changes to the convolutional kernels and normalization parameters. 
Specifically, we define sign flip vectors $\hat{\bm{\alpha}}^{(1)}, \hat{\bm{\alpha}}^{(2)}, \hat{\bm{\alpha}}^{(3)}, \hat{\bm{\alpha}}^{(4)} \in \{-1,1\}^C$, where $\hat{\bm{\alpha}}^{(1)}$ and  $\hat{\bm{\alpha}}^{(2)} \in \{-1,1\}^C$
are generated independently using user-specific random seeds, while $\hat{\bm{\alpha}}^{(3)}$ and $\hat{\bm{\alpha}}^{(4)}$ are determined by the same constraints in Eq.~\eqref{eq:con1}. The transformed parameters are then obtained as Eq.~\eqref{eq:scal1} by replacing the scaling vectors with the sign flip vectors.  Hence, sign flip is a special case of parameter scaling where scaling factors are constrained to $\{-1, 1\}$.

The security of the ACT is also analyzed in the appendix (Appendix B), where the complexity of achieving parameter alignment across different users is evaluated. 

\section{Experimental Results on Fidelity and Effectiveness}
\label{Sec: Exeriments}
\subsection{Settings}
\label{Sec: Datasets, Models, and Tasks}

\textbf{Datasets and Tasks.}
We considered both T2I image generation and text-based image editing tasks.
We fine-tuned the VAE decoder for fingerprinting using the MS-COCO-2017 train set~\cite{lin2014microsoft}. The evaluations are based on the MS-COCO-2017 val set (generation task), ImageNet~\cite{imagenet15russakovsky} (generation task), MagicBrush~\cite{Zhang2023MagicBrush} (editing task), and InstructPix2Pix~\cite{brooks2023instructpix2pix} (editing task). All images are first resized such that the shorter side is scaled to 512 pixels while maintaining the aspect ratio, followed by a random crop to obtain a final image of size $512 \times 512$. For image generation from text, following common practice~\cite{fernandez2023stable,kim2024wouaf}, we selected a caption per image in the COCO and ImageNet datasets as the input prompt. For image editing, we used: i) MagicBrush, which provides (image, mask, prompt) triplets, comprising a source image, a binary mask that indicates the region to be modified, and a prompt describing the intended modification; and ii)  InstructPix2Pix, which provides source images paired with editing prompts.

\noindent\textbf{Models.}
We considered multiple models, including Stable Diffusion v2-base (SD2)
for T2I generation, and SD2-inpainting
for the MagicBrush dataset, and the InstructPix2Pix model.
Note that they share the same VAE, so we only need to modify and fine-tune the decoder once, which can then be used across these models. The PNM was inserted before the final convolutional layer in $\mathcal{D}_m$ with $C=128$.
$\mathcal{F}_{\bm{\gamma}}$ and $\mathcal{F}_{\bm{\beta}}$ are 2-layer fully connected networks with 128 neurons and LeakyReLU activation (negative slope = 0.1)~\cite{maas2013rectifier}. 
The fingerprint decoder $\mathcal{W}$ is implemented via an EfficientNet-B0.
$\mathcal{F}_{\bm{\gamma}}$ and $\mathcal{F}_{\bm{\beta}}$ are randomly initialized, while $\mathcal{W}$ is pre-trained on ImageNet. We used 48-bit fingerprints (i.e., $d = 48$), which is the same payload used in \cite{fernandez2023stable}.

Following common practice, a noise layer is introduced in the scheme before the fingerprint decoder, performing data augmentation on the  VAE reconstructed image. This procedure forces the  VAE decoder to embed a more robust fingerprint that can survive processing. The operations we considered are random Gaussian blurring, brightness adjustment, contrast adjustment, Gaussian noise, horizontal flipping, JPEG compression, and random cropping with factors chosen from 0.5 to 0.9. Given an image, a processing in the above set is randomly chosen and applied with probability 70\%, while the image is left unchanged with probability 30\%.

With regard to the comparison of fingerprinting methods, we consider both U-Net-based methods  \textbf{AquaLoRA}~\cite{feng2024aqualora} and  \textbf{WatermarkDM}~\cite{zhao2023recipe}\footnote{As \cite{zhao2023recipe} belongs to the black-box watermarking category, comparison is done only on fidelity.}, and VAE-based methods \textbf{Per. Norm.}~\cite{fei2023robust} (applied to the  VAE decoder), \textbf{Sta. Sig.}~\cite{fernandez2023stable} and \textbf{WOUAF}~\cite{kim2024wouaf}. The hyperparameters are set as follows:  $\lambda_{\text{lpips}} = 10$, $\lambda_{\text{mse}} = 1$, $\lambda_{\text{m}} = 1$, $\lambda_{\text{reg}} = 1$, and $\eta_1 = 0.05$. The model is optimized using the Adam optimizer with a learning rate of 0.0003, $\beta_1 = 0.9$, $\beta_2 = 0.999$, and a batch size of 6.

\subsection{Evaluation Metrics}
To evaluate the impact of the fingerprint on the performance of image reconstruction, we measured the peak signal-to-noise ratio (PSNR), the structural similarity index (SSIM), and the learned perceptual image patch similarity (LPIPS).
The quality of the image generation is also measured via the FID score. Finally, we also considered the CLIP score~\cite{hessel2021clipscore} to measure the semantic similarity between the generated image and the prompt. Specifically:
\begin{description}
\item i) To evaluate image reconstruction performance, we measured PSNR, SSIM, and LPIPS between images reconstructed by the fingerprinted VAE and: 
(1) the reconstructed images produced by the non-fingerprinted
model (baseline); (2) the original (real) images. The corresponding metrics are referred to respectively as $\text{PSNR}_b$, $\text{SSIM}_b$, and $\text{LPIPS}_b$ (b = baseline), and $\text{PSNR}_r$, $\text{SSIM}_r$, and $\text{LPIPS}_r$ (r = real).
\item ii) To evaluate image generation performance, we measured the $\text{FID}$ 
between images generated by the fingerprinted model and real images, and CLIP score (scaled to [0, 1]) between generated images and the prompt.
\end{description}

The effectiveness of the fingerprinting method is evaluated via the bit-wise matching accuracy (Bit Acc, \%) between the extracted fingerprint and the ground truth fingerprint, given by Eq.~\eqref{eq: fingerprint verification}. To evaluate the fingerprint verification performance, we also measured the TPR/FPR and the CIR (in the identification case) given by Eq.~\eqref{eq: detection} and Eq.~\eqref{eq: identification}-\eqref{eq: identification2}.

\begin{figure}[H]
\centering
    \includegraphics[width=0.48\textwidth]{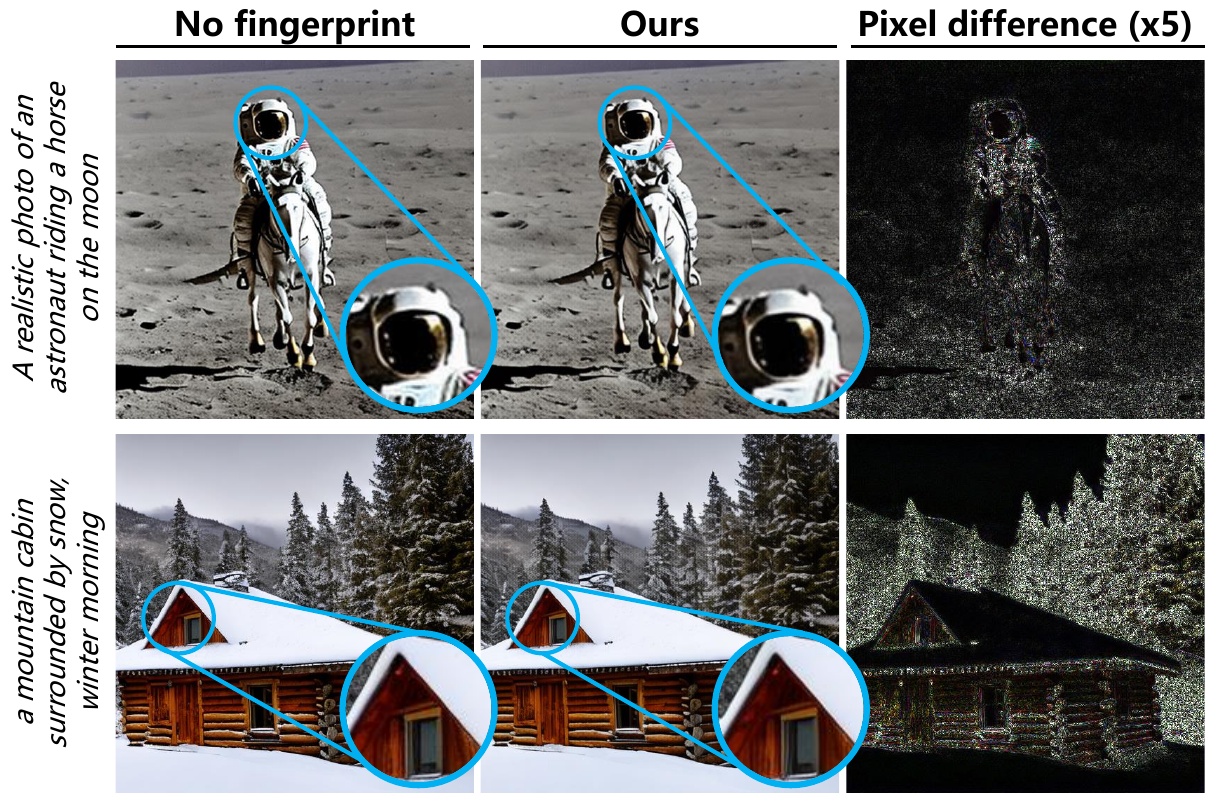}
    \caption{Images generated by SD2 (no fingerprint) and our fingerprinted model, along with their differences ($\times$5).}
\label{fig: sample}
\end{figure} 

\subsection{Performance Evaluation}
\label{Sec: Model Fidelity and Fingerprint Effectiveness}
\textbf{Fidelity.}
We first evaluate the impact of the fingerprinting on image quality.
As shown in Table~\ref{table: fid}, in terms of image reconstruction performance, our method achieves the highest $\text{PSNR}_b$ values across all models when evaluated against images generated by the non-fingerprinted model. Our method achieves competitive performance also in terms of SSIM and LPIPS, whose values are always close to those achieved by the other VAE-based methods. Since U-Net–based methods do not modify the VAE, the comparison with the non-fingerprinted model is not applicable.

However, in terms of image generation quality (FID and CLIP score), U-Net–based methods lead to a notable degradation.
For example, on COCO, for our method $\text{FID}$ is 24.03, nearly matching WOUAF (23.93), Sta. Sig. (24.06) and Per. Norm. (23.97), outperforming AquaLoRA and WatermarkDM (for which $\text{FID}$  is  24.92 and 27.45).
On ImageNet, our $\text{FID}$ is 25.46, which is 2-5 better than AquaLoRA and WatermarkDM, and better than WOUAF (26.34) and almost as good as Sta. Sig. (25.71). On MagicBrush, our method achieves 22.18, close to the best (21.76 of Per. Norm.), and better than  WOUAF (22.24) and Sta. Sig. (22.40). This advantage is consistent in image editing. Our method achieves the best or second-best $\text{FID}$ on both MagicBrush and InstructPix2Pix, and the change in FID relative to the non-fingerprinted model is within 0.5, which is 2–3 points better than U-Net–based methods and also superior to VAE–based methods. CLIP scores are stable across all VAE–based methods (0.620–0.659), showing that semantic consistency is preserved regardless of the presence of the fingerprint, whereas U-Net–based methods cause a slight degradation in semantic consistency.

In Fig.~\ref{fig: sample}, we show some images generated by the original non-fingerprinted model and by our fingerprinted model,  obtained using the same input prompt, in the case of SD2. The images produced by the fingerprinted model are almost indistinguishable from those produced by the original model in both semantic content and fine-grained details.

\begin{table*}[ht]
    \renewcommand{\arraystretch}{0.95}
    \centering
    \caption{Performance of model fingerprinting on various datasets and tasks. for T2I and also image editing (PSNR $\uparrow$; SSIM$\uparrow$; LPIPS$\downarrow$; FID$\downarrow$; CLIP$\uparrow$; Bit Acc(\%) $\uparrow$). ‘Generation’ refers to T2I image generation tasks based on SD2; ‘Editing’ refers respectively to text-based image editing tasks based on the SD2-inpainting and InstructPix2Pix models. 'retrain-free' denotes whether a new fingerprinted model necessitates retraining.
    }
    \begin{tabular}{
    m{1.9cm}<{\centering}|
    m{0.8cm}<{\centering}|
    m{2.2cm}<{\centering}|
    m{0.9cm}<{\centering}|
    m{0.75cm}<{\centering}
    m{0.75cm}<{\centering}
    m{0.75cm}<{\centering}
    m{0.75cm}<{\centering}
    m{0.75cm}<{\centering}
    m{0.75cm}<{\centering}|
    m{0.75cm}<{\centering}
    m{0.75cm}<{\centering}|
    m{0.75cm}<{\centering}}
    \toprule
    \multirow{2}*{\textbf{Dataset}} & \multirow{2}*{\textbf{Type}} &\multirow{2}*{\textbf{Method}}   & \multirow{2}*{\makecell{\textbf{Retrain}\\\textbf{-free?}}} & \multicolumn{6}{c|}{\textbf{Reconstruction}} & \multicolumn{2}{c|}{\textbf{Generation}} & \multirow{2}*{\textbf{Bit Acc}}\\
    & & &  &\textbf{$\text{PSNR}_{b}$}  & \textbf{$\text{PSNR}_{r}$} & \textbf{$\text{SSIM}_{b}$} & \textbf{$\text{SSIM}_{r}$}   & \textbf{$\text{LPIPS}_{b}$} & \textbf{$\text{LPIPS}_{r}$} & \textbf{$\text{FID}$} & \textbf{CLIP} \\
    \midrule
    
   \multirow{7}*{\makecell{\textbf{CoCo}\\ \textbf{(Generation)}}}& &No fingerprint& - & - & 28.55 & - & 0.805 & -  &0.128  & 23.26& 0.658 & - \\ \cmidrule(lr){2-13}
    & \multirow{2}*{U-net} &AquaLoRA~\cite{feng2024aqualora} & \checkmark& -&\textbf{28.55} &- & 0.805& -  & \underline{0.128}  &24.92 &0.657 & 95.43\\ 
    & & WatermarkDM~\cite{zhao2023recipe} &  \ding{55}& -&\textbf{28.55} &- & 0.805& -& \underline{0.128}  &27.45 &0.655 &-\\ \cmidrule(lr){2-13}
    & \multirow{4}*{VAE} &Per. Norm.~\cite{fei2023robust} & \checkmark & 29.97& 27.56 & 0.858 & 0.789 & \textbf{0.068} & \textbf{0.124}  &\underline{23.97}  &   0.658  &\textbf{99.60}\\
    & &Sta. Sig.~\cite{fernandez2023stable} & \ding{55} & 30.90 & 28.04& \textbf{0.897} & \textbf{0.815} & 0.084&  0.142 & 24.06 &   0.658 & 99.48\\
    & &WOUAF~\cite{kim2024wouaf} & \checkmark & \underline{31.11}  & 28.26 &0.882 &\underline{0.814} &0.078 &\underline{0.128} &\textbf{23.93} &0.658 & 98.95 \\ \cmidrule(lr){3-13}
    & &Ours&  \checkmark  &\textbf{31.80} & \underline{28.51} & \underline{0.883} & 0.808 & \underline{0.071} &0.129   & 24.03 &  0.658 & \underline{99.57} \\
    \midrule
    
   \multirow{7}*{\makecell{\textbf{ImageNet}\\ \textbf{(Generation)}}} &&No fingerprint& -  & - &29.58 & -& 0.825 &   & 0.125 & 24.23 & 0.659 &- \\ \cmidrule(lr){2-13}
    & \multirow{2}*{U-net}&AquaLoRA~\cite{feng2024aqualora} & \checkmark&  -  & \textbf{29.58}&- &\underline{0.825} & -& \textbf{0.125}&27.65 &0.659 &95.26\\ 
   & & WatermarkDM~\cite{zhao2023recipe} &  \ding{55} & -&\textbf{29.58} &- & \underline{0.825} & -& \textbf{0.125}&30.14 &0.654 & -\\  \cmidrule(lr){2-13}
    & \multirow{4}*{VAE} &Per. Norm.~\cite{fei2023robust} &  \checkmark  & 30.54& 28.13 &0.879  & 0.809 & \textbf{0.066} & \textbf{0.125}  &  \textbf{25.25} &  0.659 & \textbf{99.65} \\
    &&Sta. Sig.~\cite{fernandez2023stable} & \ding{55} & 32.10 &28.90 & \underline{0.919} & \textbf{0.831} &0.082 & 0.148  & 25.71 &   0.659 &99.48 \\
    &&WOUAF~\cite{kim2024wouaf} & \checkmark &\underline{32.05} &29.19 & \underline{0.919} &0.828 &0.081 &0.149 & 26.34 & 0.659 & 98.82\\  \cmidrule(lr){3-13}
    & &Ours& \checkmark  &\textbf{32.66} & \underline{29.30} &\textbf{0.928} &0.822 & \underline{0.067}& \underline{0.127}  & \underline{25.46} &  0.659 & \underline{99.61}\\
    \midrule
    
   \multirow{7}*{\makecell{\textbf{MagicBrush}\\ \textbf{(Editing)}}}& &No fingerprint & -  &- & 30.82& - &0.871 & -& 0.105   & 19.92 &  0.620&- \\ \cmidrule(lr){2-13}
   & \multirow{2}*{U-net} &AquaLoRA~\cite{feng2024aqualora} & \checkmark&   -  & \textbf{30.82}&- &\textbf{0.871} & -& \textbf{0.105}&23.52 &0.618 & 94.81\\   
    && WatermarkDM~\cite{zhao2023recipe} &  \ding{55} & -&\textbf{30.82} &- & \textbf{0.871}& -& \textbf{0.105}&24.95 &0.612 & -\\ \cmidrule(lr){2-13}
    & \multirow{4}*{VAE}&Per. Norm.~\cite{fei2023robust} & \checkmark  & 31.41&   29.01&0.898  &0.844  &\textbf{0.053} &\underline{0.106}   & \textbf{21.76}  &  0.620  & \textbf{99.59}\\
    &&Sta. Sig.~\cite{fernandez2023stable} & \ding{55}  & 32.63 &29.83 &\textbf{0.928}  &  0.861 & 0.066& 0.120  & 22.40 &  0.620 & \underline{99.51} \\
    &&WOUAF~\cite{kim2024wouaf} & \checkmark  &\underline{33.01} &30.22 &\underline{0.924} &0.850 &\underline{0.062} &0.109 &22.24 &0.620 & 99.17 \\ \cmidrule(lr){3-13}
     &&Ours &  \checkmark  &  \textbf{33.95} &\underline{30.56} &\underline{0.924} & \underline{0.869} & \textbf{0.053} &\underline{0.106}  & \underline{22.18} & 0.620  & 99.49\\ \midrule

   \multirow{7}*{\makecell{\textbf{InstructPix2Pix}\\ \textbf{(Editing)}}} & &No fingerprint & - & - &33.49  & - & 0.936 & -& 0.040   & 13.21 & 0.657  &- \\ \cmidrule(lr){2-13}
   & \multirow{2}*{U-net} &AquaLoRA~\cite{feng2024aqualora} & \checkmark&   -  & \textbf{33.49} &- &\textbf{0.936} & -& \textbf{0.040} &14.59 &0.655 & 95.49\\
   & & WatermarkDM~\cite{zhao2023recipe} &  \ding{55} & -&\textbf{33.49} &- & \textbf{0.936}& -& \textbf{0.040} &16.78 &0.649 &- \\  \cmidrule(lr){2-13}
   & \multirow{4}*{VAE} &Per. Norm.~\cite{fei2023robust} & \checkmark &29.59  & 29.20 & 0.841 & 0.822 & 0.063& 0.074  & 13.51  &  0.658 & \underline{99.27}\\
   & &Sta. Sig.~\cite{fernandez2023stable} & \ding{55} & 29.63 &28.96 &  0.819& 0.814 &0.063 &0.076  & 13.69& 0.658 &99.20 \\
    &&WOUAF~\cite{kim2024wouaf} & \checkmark &\textbf{30.05} &  29.17 &0.820  &0.804 &  0.067& 0.075& 13.69 &0.658 & 99.15\\  \cmidrule(lr){3-13}
    & &Ours &  \checkmark & \underline{30.02} & 29.84  & \textbf{0.879}  &\underline{0.845} &\textbf{0.060} &\underline{0.069} &\textbf{13.06} & 0.658 & \textbf{99.46}\\ 
    
    \bottomrule
    \end{tabular}
    \label{table: fid}
\end{table*}

\begin{figure*}[ht]
    \centering
    \begin{minipage}{0.25\textwidth}
        \centering
        \includegraphics[width=\linewidth]{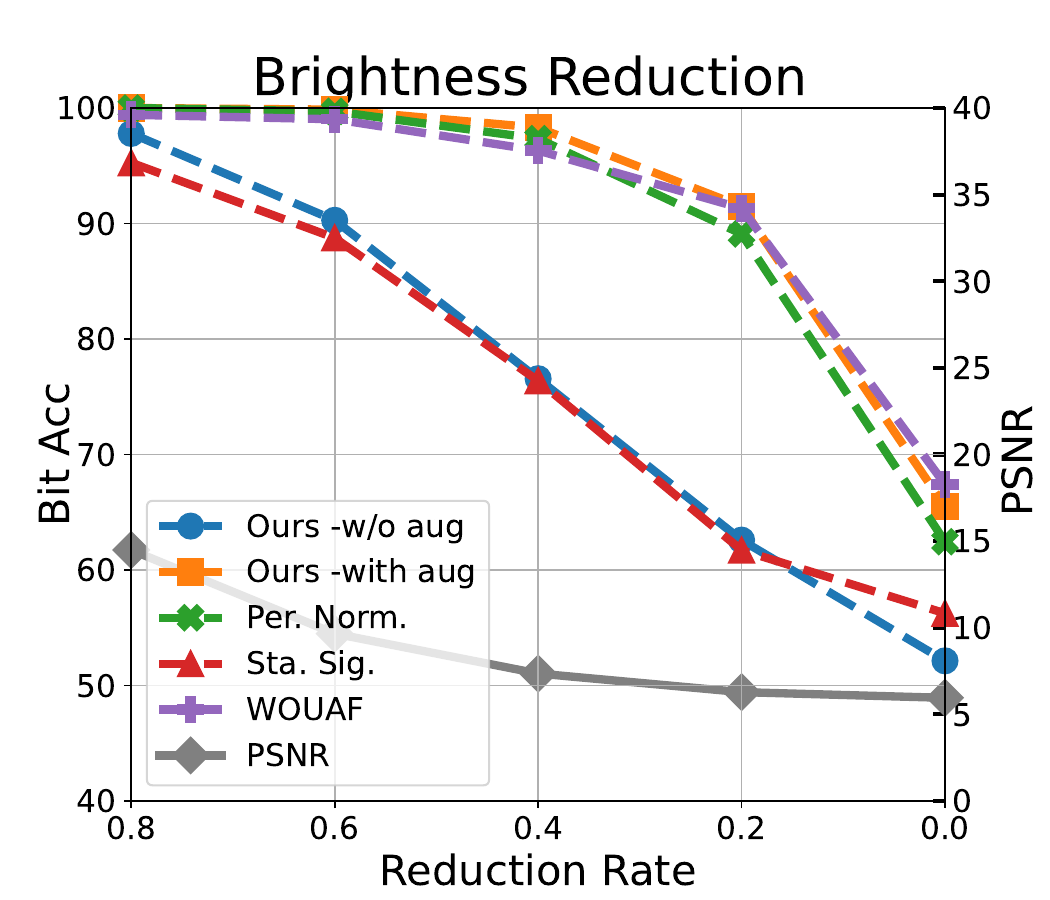}
    \end{minipage}\hfill
    \begin{minipage}{0.25\textwidth}
        \centering
        \includegraphics[width=\linewidth]{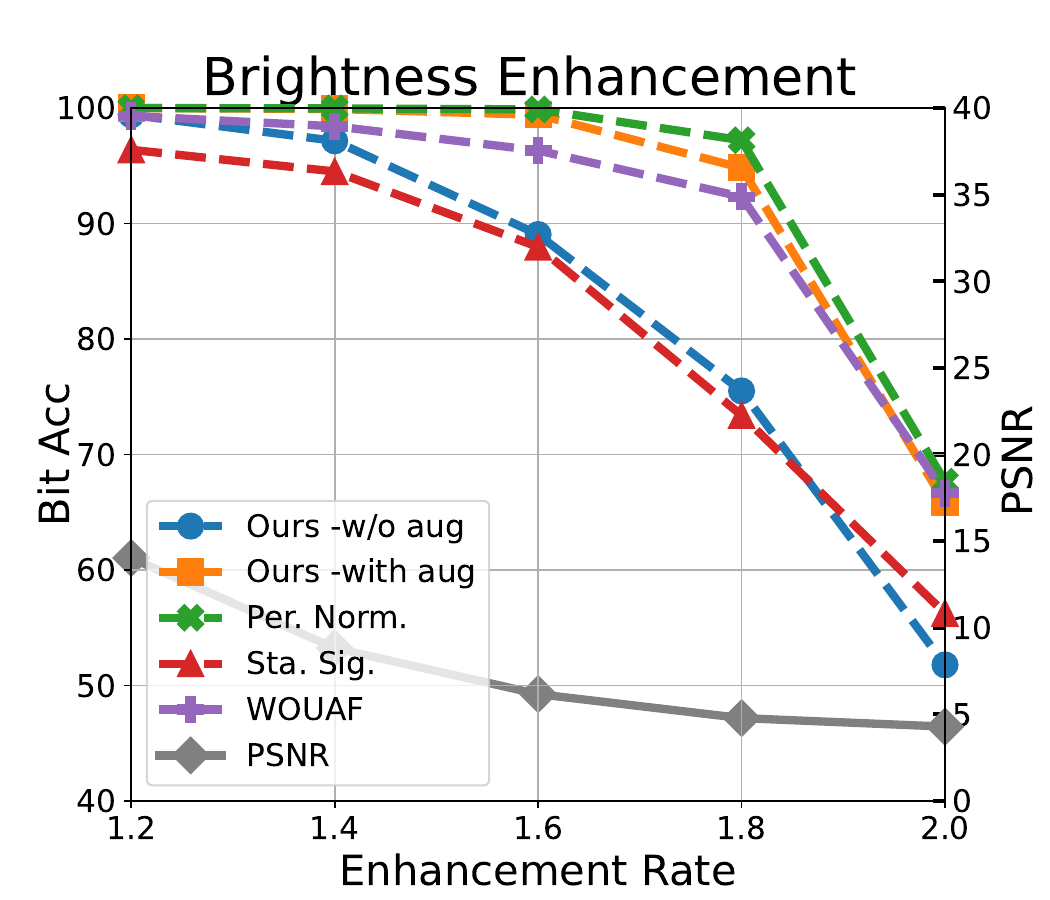}
    \end{minipage}\hfill
    \begin{minipage}{0.25\textwidth}
        \centering
        \includegraphics[width=\linewidth]{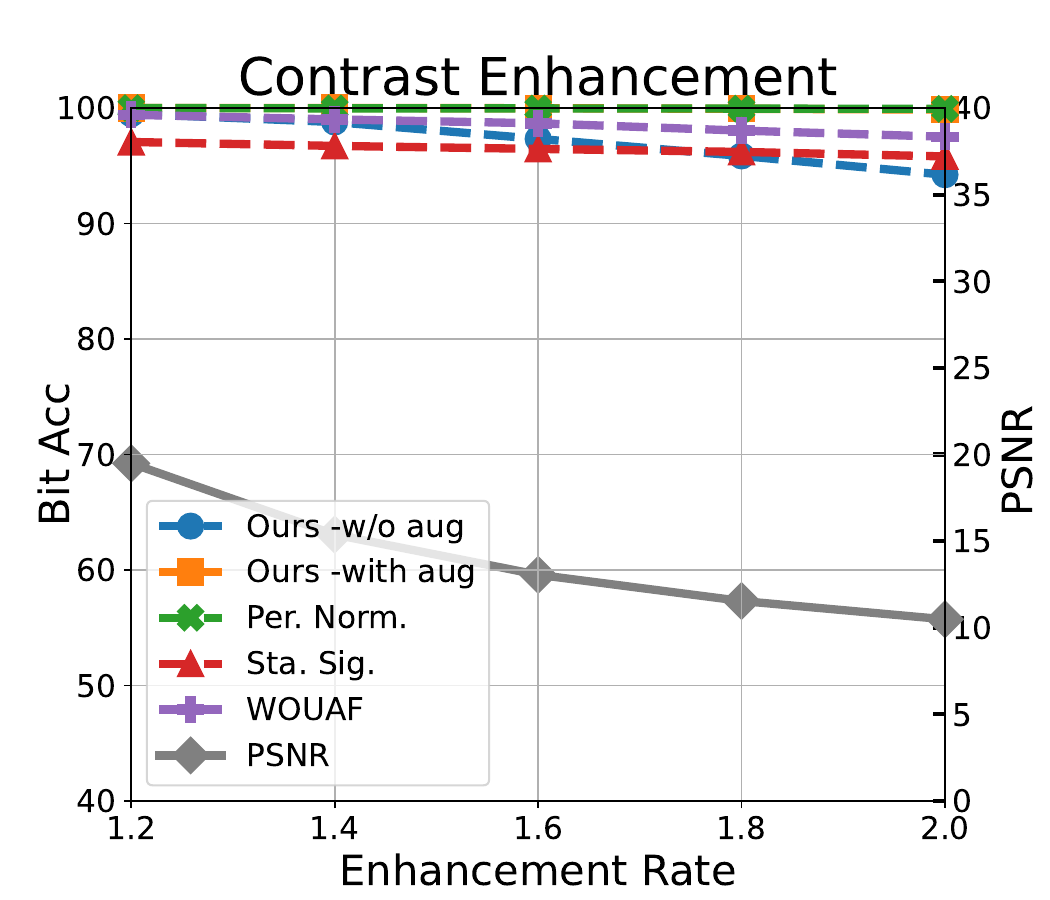}
    \end{minipage}\hfill
    \begin{minipage}{0.25\textwidth}
        \centering
        \includegraphics[width=\linewidth]{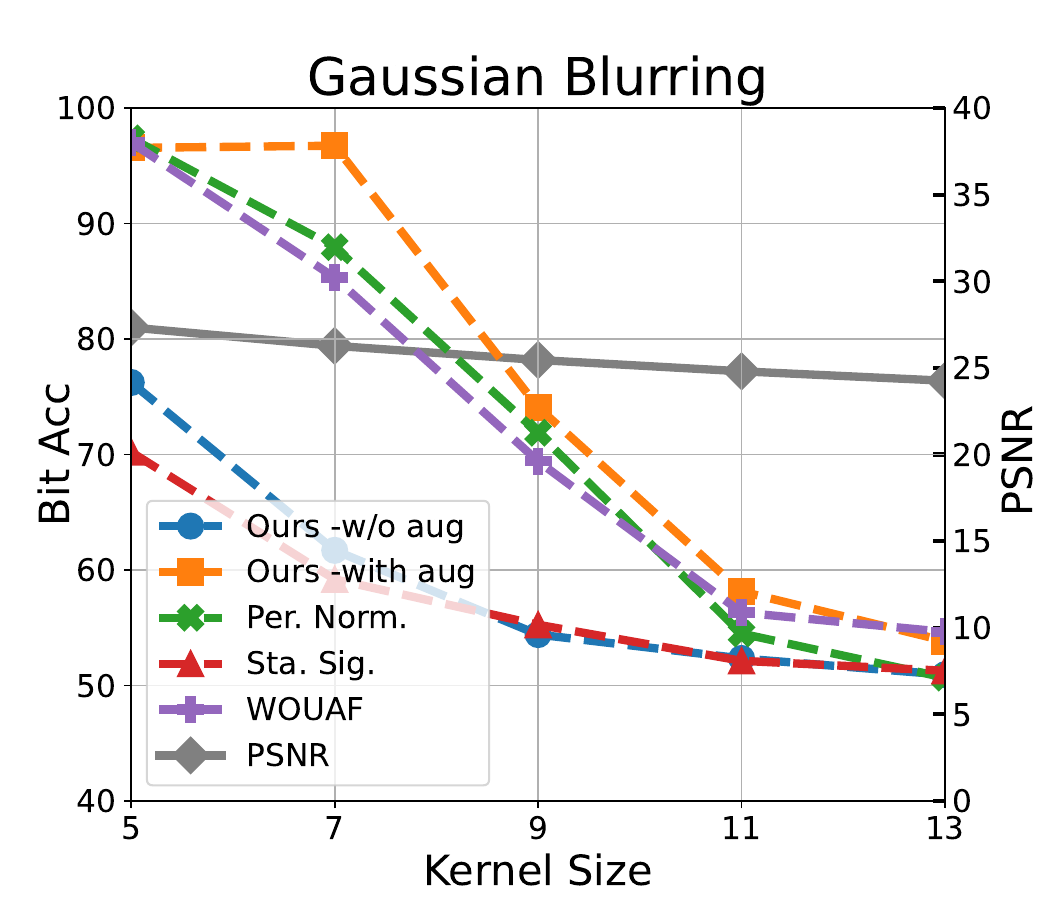}
    \end{minipage}
    \begin{minipage}{0.25\textwidth}
        \centering
        \includegraphics[width=\linewidth]{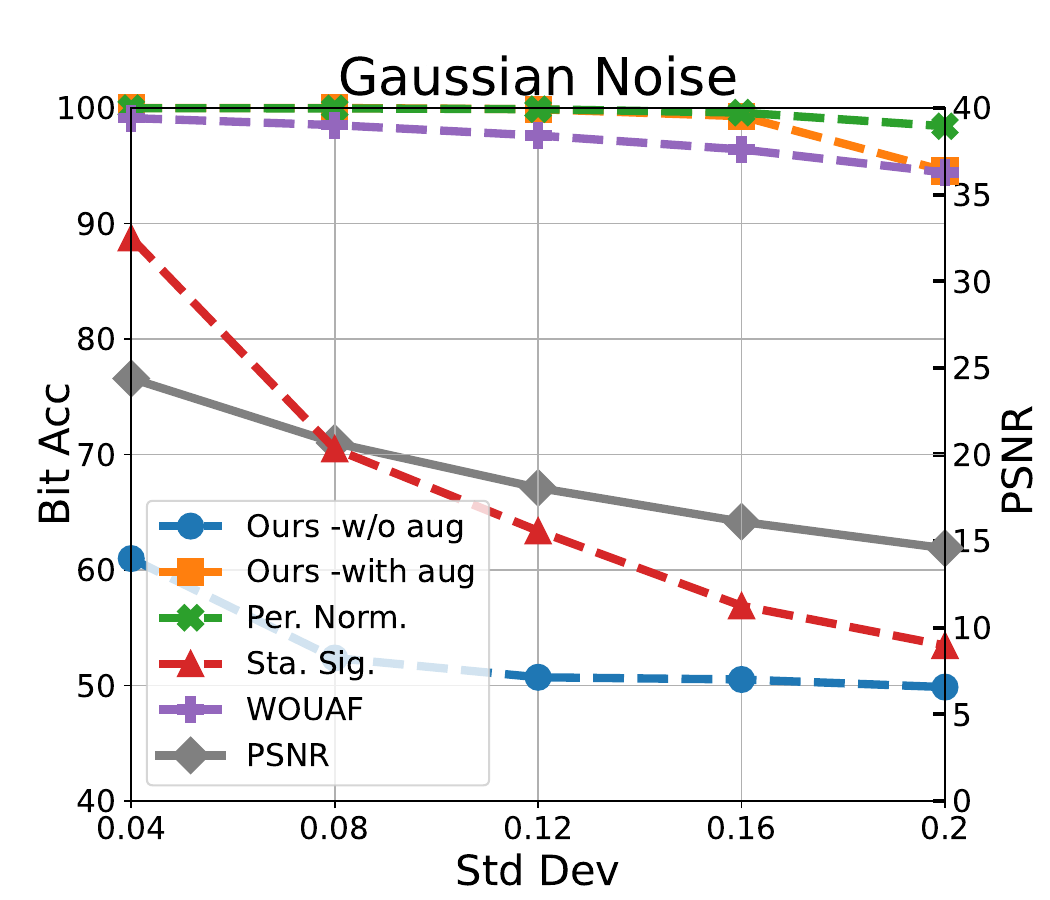}
    \end{minipage}\hfill
    \begin{minipage}{0.25\textwidth}
        \centering
        \includegraphics[width=\linewidth]{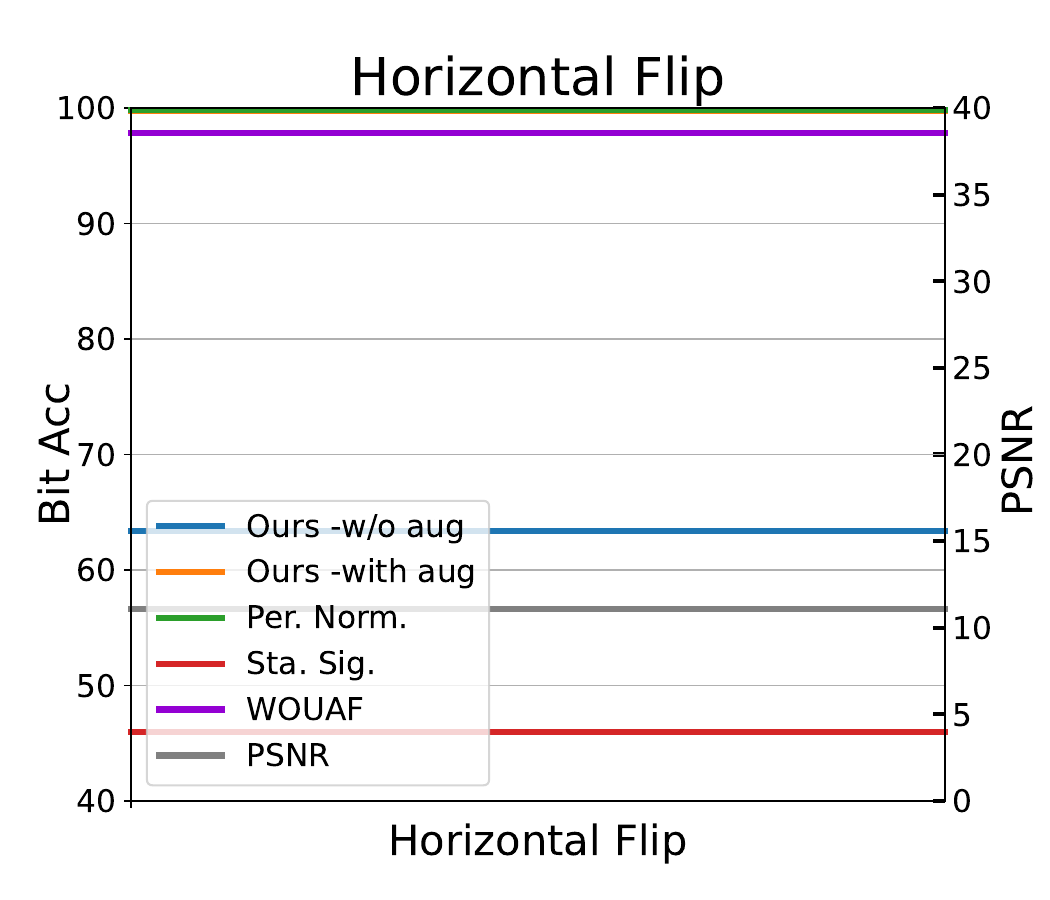}
    \end{minipage}\hfill
    \begin{minipage}{0.25\textwidth}
        \centering
        \includegraphics[width=\linewidth]{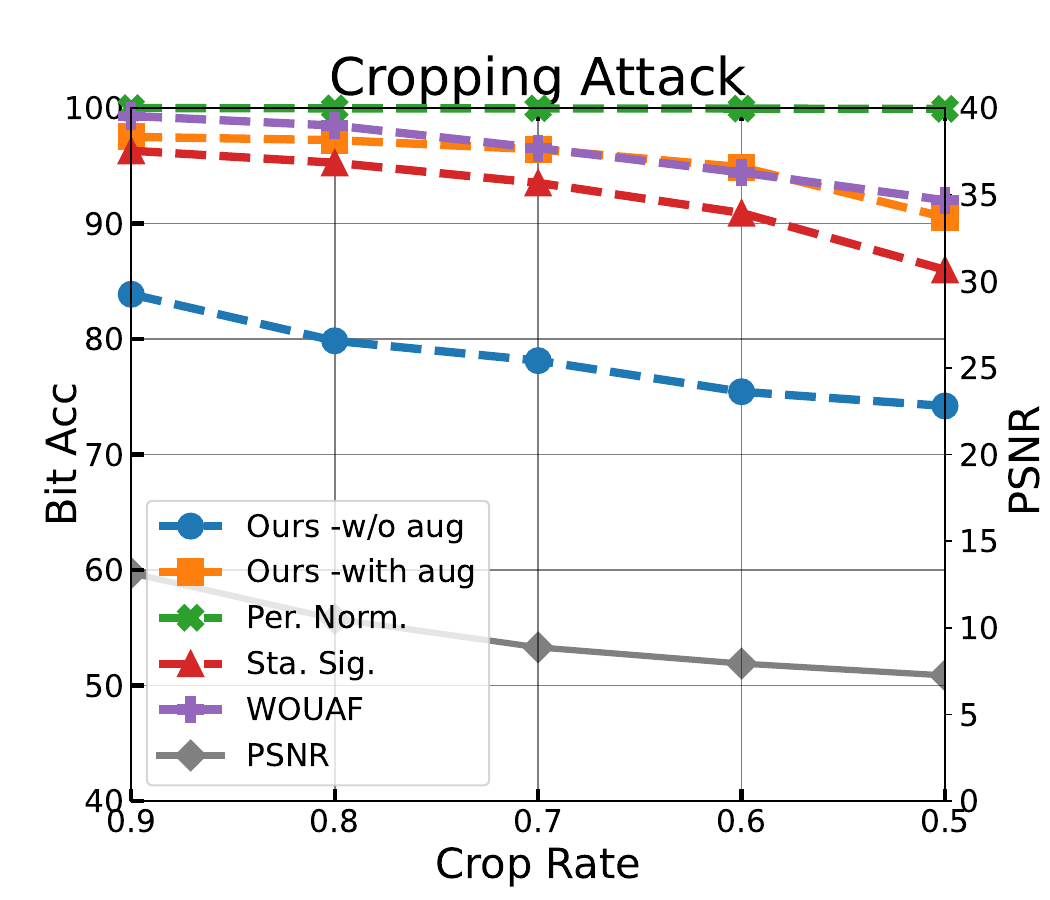}
    \end{minipage}\hfill
    \begin{minipage}{0.25\textwidth}
        \centering
        \includegraphics[width=\linewidth]{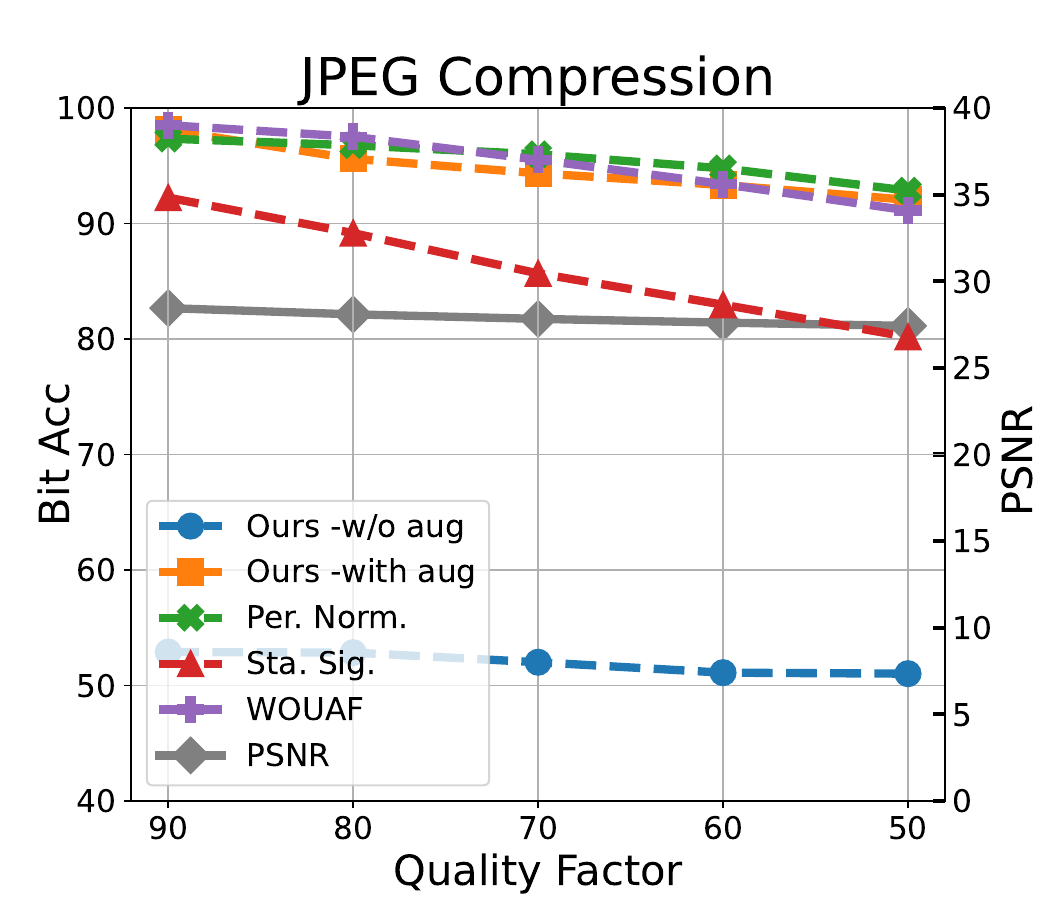}
    \end{minipage}
\caption{Robustness against different image-level attacks.
The grey line indicates the PSNR after the attack.}
\label{fig: image attack}
\end{figure*}

\textbf{Fingerprinting Performance.}
In the last column of Table~\ref{table: fid}, we report the Bit Acc (\%) achieved on all models and datasets.
The reported Bit Acc is averaged over 32 random fingerprints and 1,000 fingerprinted images for each distinct fingerprint (i.e., 32,000 images per dataset). Our method achieves a Bit Acc of around 99.5\% across all cases, demonstrating accurate fingerprint extraction. We observe that VAE–based methods achieve higher performance, whereas the Bit Acc of AquaLoRA is limited to 95\%. Since WatermarkDM is a trigger-based black-box watermarking scheme, the watermark is not extracted from model outputs but defined as a specific output pattern corresponding to trigger inputs. Therefore, comparison in terms of Bit Accuracy is not applicable.

We further evaluate the performance of our method in fingerprint verification. The results are presented in Fig.~\ref{fig: Threshold vs. TPR and FPR}. The behavior of TPR/FPR (Eq.~\ref{eq: detection}) is reported as a function of the threshold, defined as the number of matched bits (i.e., $\tau \times d$, where $d = 48$). The rates are evaluated on 1024k matches under $H_0$ and 32k matches under $H_1$. As the threshold increases from 35 to 47, the TPR remains consistently high (equal to 1.00 up to 41) and begins to decline thereafter, whereas the FPR is around $10^{-4}$ at 35 and drops sharply around 39 to 40. At a threshold of 41, the FPR is 0 in all cases, while the TPR is 1.0 for all cases. 
Therefore, in the subsequent evaluations, we fix $\tau = 0.85 (41/48)$. We also computed the TPR and FPR for this case, and all methods achieved perfect performance, with TPR = 1 and FPR = 0.

In Fig.~\ref{fig: Number of users M  vs. TPR/FPR/CIR}, we report the performance of fingerprint identification (see Section \ref{sec: Fingerprint Verification}) in the T2I COCO case (results are similar in the other cases). The TPR/FPR and CIR (defined in Eq.~\eqref{eq: identification}-\eqref{eq: identification2}) are reported as a function of the number of users, i.e., distinct fingerprinted models, when $\tau = 0.85$.
The rates under $H_1$ (i.e., TPR and CIR) are evaluated considering 10 fingerprinted images for each user, for a total of $10 \cdot M$ fingerprinted images, while the FPR is computed on 1,000 non-fingerprinted images. We see that the TPR and CIR remain consistently near 1.0, even as the number of users reaches $10^4$. Meanwhile, the FPR remains close to 0. Hence, given an image, the source model can be correctly identified among $10^4$ distinct fingerprinted models released to users. These results show that our method can achieve good identification performance.\footnote{It is worth stressing that the fingerprint associated with every user is chosen randomly. In principle, given a maximum number of users $M$ ($M \ll 2^d$) that the distributor wants to allocate, better performance can be achieved by choosing the $d$-bit fingerprints to be associated with every user by means of a suitable binary fingerprint code~\cite{tardos2008optimal}. This analysis goes beyond the scope of this paper.} 
Since the results obtained for the various tasks are very similar, in the following, we focus on  T2I COCO, and results are reported for this case, unless stated otherwise.

\begin{figure}[h]
    \centering
    \begin{minipage}{0.45\textwidth}
        \centering
        \includegraphics[width=\linewidth]{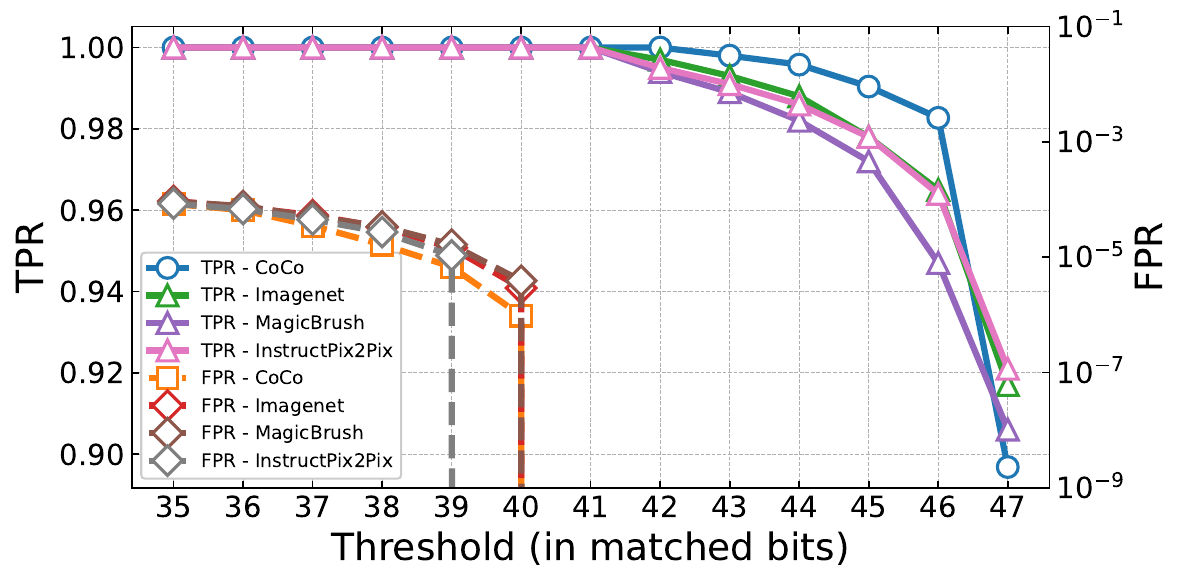}
        \subcaption{TPR and FPR vs Threshold ($\tau \times d$)} 
        \label{fig: Threshold vs. TPR and FPR}
    \end{minipage}
    \hfill
    \begin{minipage}{0.45\textwidth}
        \centering
        \includegraphics[width=\linewidth]{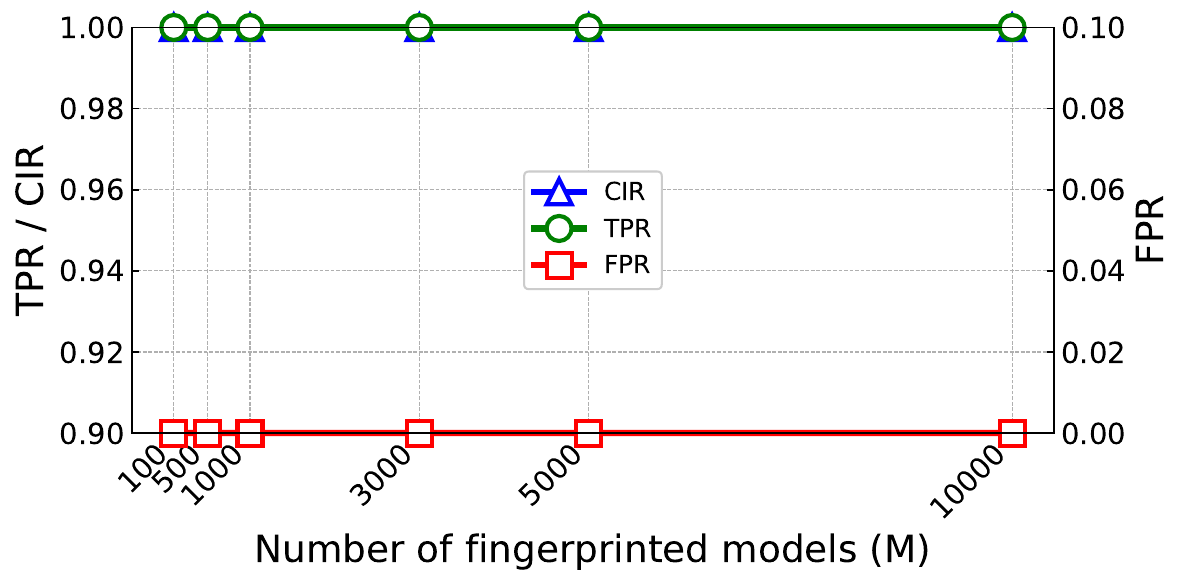}
        \subcaption{TPR,FPR and CIR vs M (number of users). 
        }
        \label{fig: Number of users M  vs. TPR/FPR/CIR}
    \end{minipage}
    \caption{Impact of the threshold on fingerprint verification (a) and of the number of users $M$ on fingerprint identification (b). 
    }
    \label{fig: collusion_plts}
\end{figure}

\subsection{Robustness Analysis} 
\label{Sec:Robustness Analysis}
\noindent \textbf{Robustness Against Image-level Attacks.}
Fig.~\ref{fig: image attack} shows the robustness against various image-level modifications. We report the Bit Acc under various attacks, as well as the PSNR between the processed images and the original ones. To show the positive impact of the inclusion of the noise layer (which is included for all the methods), we also report the performance of the proposed method without the noise layer augmentation (denoted by '-w/o aug'). We can observe that, as expected, in all the cases, the Bit Acc of the fingerprint decreases as the attack strength increases. The incorporation of the noise layer augmentation during fine-tuning considerably enhances the fingerprint robustness. When the attack is very strong, the Bit Acc reduces. However, the PSNR also drops, meaning that the attack strongly impairs the quality of the images, making them unusable. Regarding the comparison with the existing methods, Per. Norm.~\cite{fei2023robust} and our method are those that achieve the best performance, which is similar in most cases, while for Sta. Sig.~\cite{fernandez2023stable}, the robustness is significantly lower. A possible explanation for this lies in the two-stage procedure adopted by this method, in which the fingerprint decoder and the T2I VAE are optimized independently. The joint fine-tuning of the fingerprint decoder and the VAE in Per. Norm.~\cite{fei2023robust} and our method forces the VAE decoder to embed a more robust fingerprint, allowing the fingerprint decoder to recover it also from processed versions of the image.

We also evaluated our method against fingerprint purification attacks using deep learning-based image compression models (Cheng~\cite{cheng2020learned} and Ballé~\cite{balle2018variational}). As shown in Fig.~\ref{fig: vae}, our method shows moderate robustness, as attackers must degrade image quality with a PSNR drop of over 3 dB to reduce the Bit Acc to ~75\%. Moreover, we found that incorporating Cheng~\cite{cheng2020learned} as augmentation during fingerprint model fine-tuning is highly effective.  
It improves robustness by increasing Bit Acc by over 10\% under the same attack for comparable quality degradation, while also enhancing resistance to unseen purification methods such as the Ballé~\cite{balle2018variational} attack.

\begin{figure}[H]
\centering
    \includegraphics[width=0.48\textwidth]{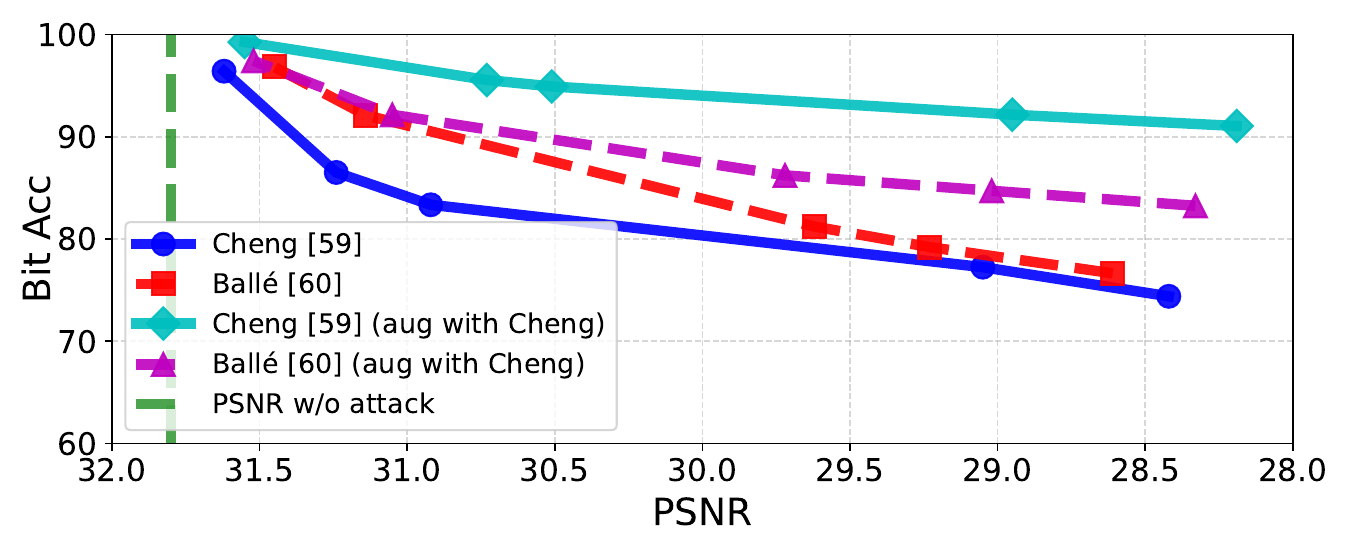}
    \caption{
    Robustness against purification attacks. The trade-off between Bit Acc and post-attack image quality is shown. Augmenting our method with Cheng~\cite{cheng2020learned} can also enhance robustness against both known (Cheng~\cite{cheng2020learned}) and unknown (Ballé~\cite{balle2018variational}) attacks.
    }
    \label{fig: vae}
\end{figure}

\begin{figure}[h]
    \centering
    \begin{minipage}{0.48\textwidth}
        \centering
        \includegraphics[width=\linewidth]{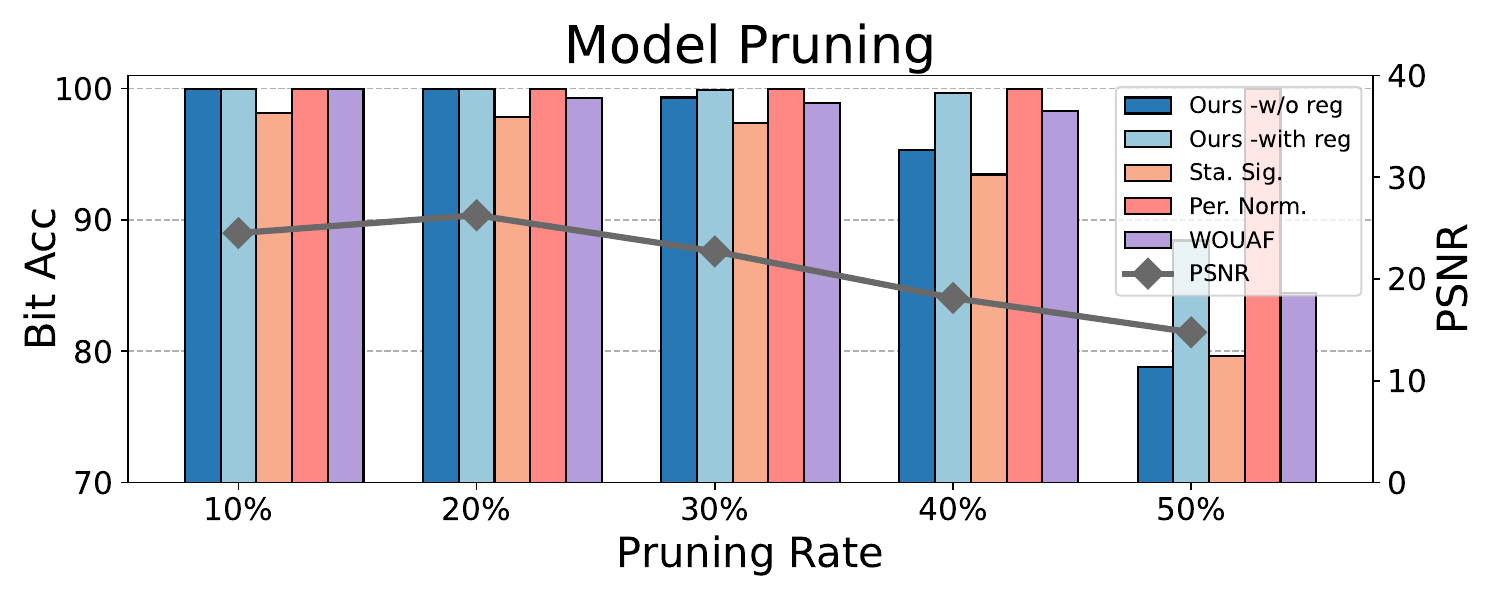}
    \end{minipage}

    \begin{minipage}{0.48\textwidth}
        \centering
        \includegraphics[width=\linewidth]{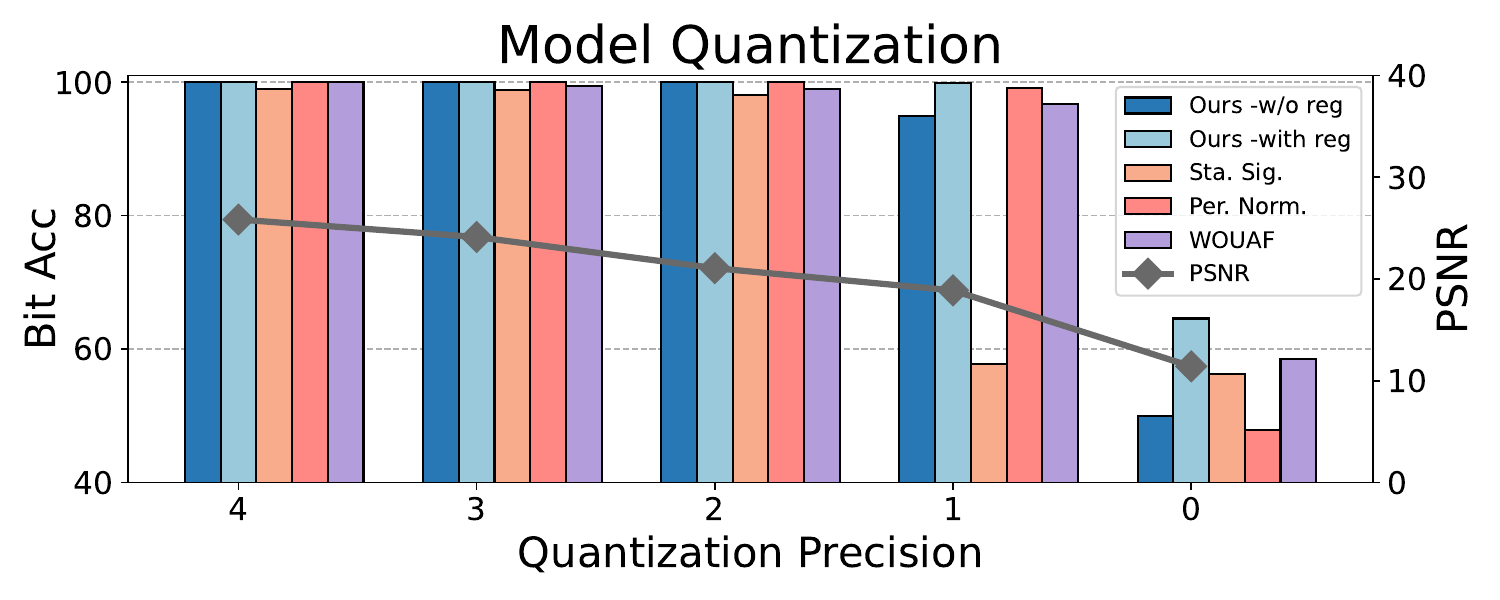}
    \end{minipage}

    \begin{minipage}{0.48\textwidth}
        \centering
        \includegraphics[width=\linewidth]{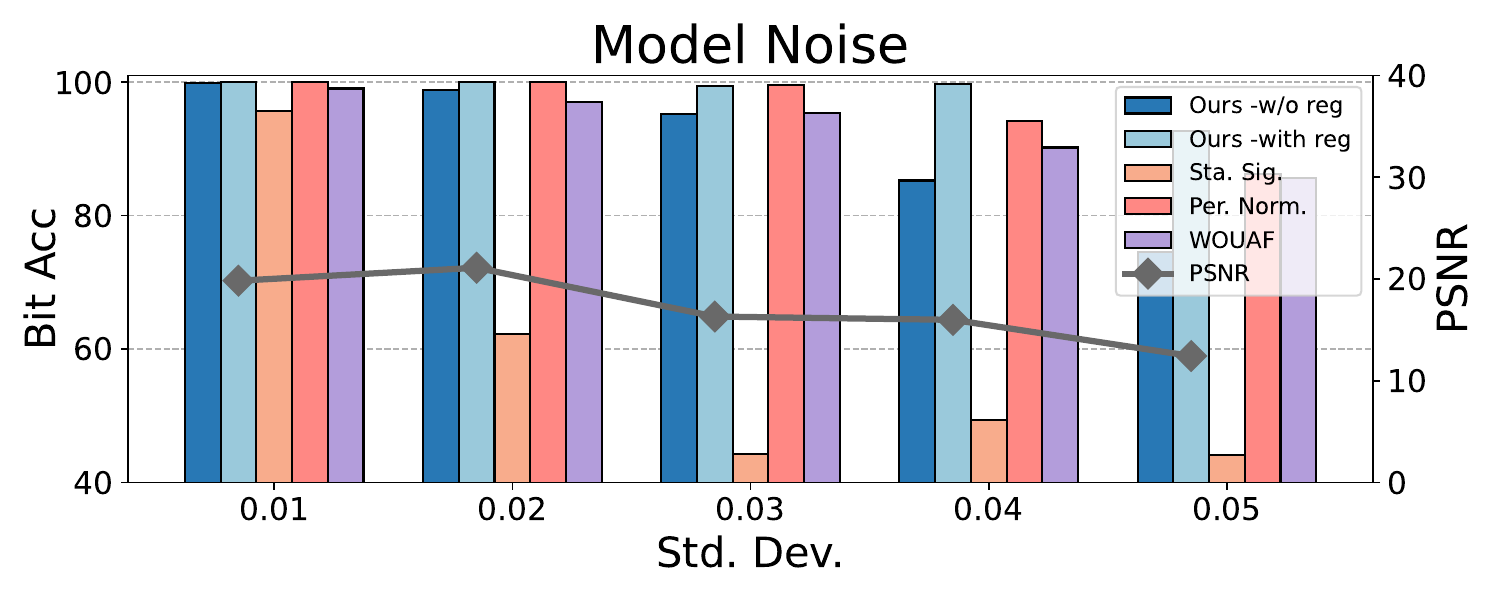}
    \end{minipage}
\caption{Robustness against different model-level attacks. The gray line indicates the PSNR after the attack.}
\label{fig: model attack}
\end{figure}

\begin{figure}[h]
\centering
    \includegraphics[width=0.48\textwidth]{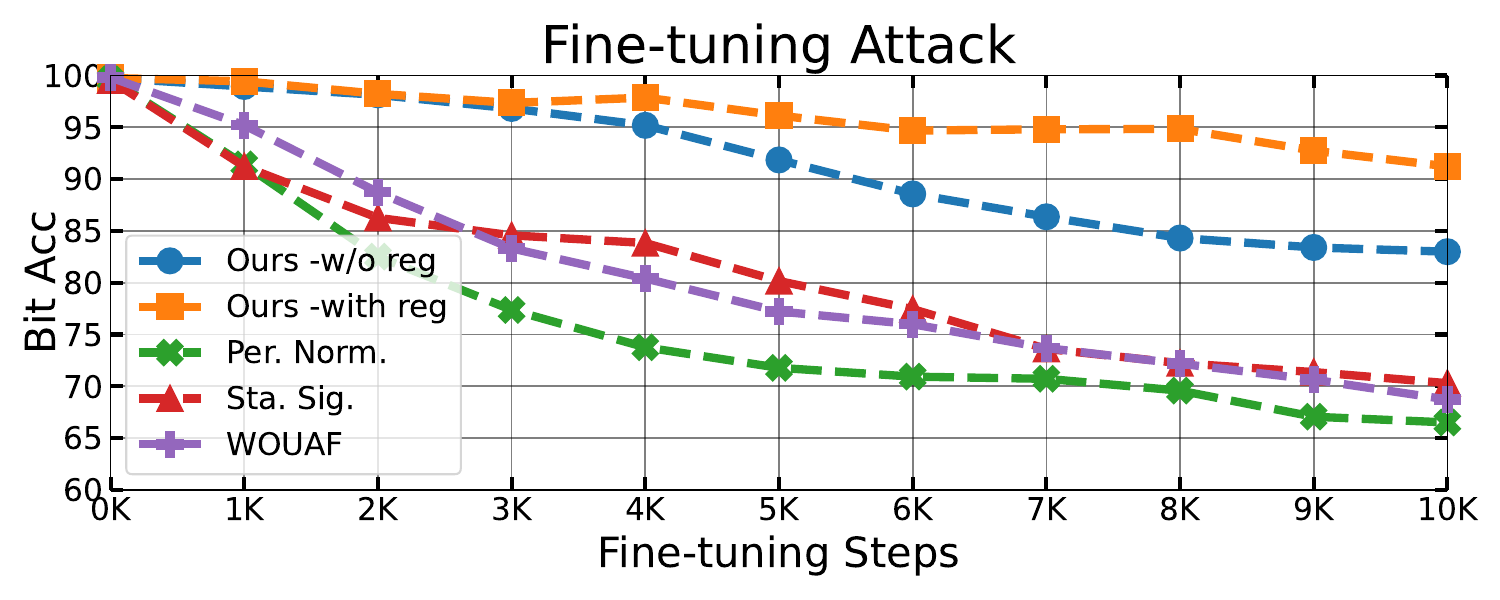}
    \caption{Robustness against fine-tuning attack.}
\label{fig: ft_atk}
\end{figure}

\noindent\textbf{Robustness Against Model-level Attacks.}
We assessed the robustness against modifications of the model in the case of pruning, compression, noise addition, and fine-tuning. 
In detail: 
1) for pruning, a percentage of the parameters having the smallest magnitude is set to zero; 
2) for quantization, we reduce the numerical precision of the model parameters, truncating to a fixed number of significant digits. 
3) for noise addition, the weights of the VAE are perturbed by adding Gaussian noise. 
The results are reported in Fig.~\ref{fig: model attack} (in the figure, levels 4, 3, 2, 1, and 0 correspond to a quantization precision of $10^{-4}$, $10^{-3}$, $10^{-2}$, $10^{-1}$, and $10^{0}$, respectively).
To show the benefit of the proposed regularization, the performance is also reported for our method in the case where the regularization loss ($\mathcal{L}_{reg}$) is not used during the fine-tuning of the VAE (i.e.,  setting $\lambda_{reg}  = 0$ in Eq. \eqref{eq:total_loss}).
The results show that, when the regularization loss is considered, our method outperforms all the other methods, except in the case of model pruning with 50\% of pruning rate, in which case, Per. Norm.~\cite{fei2023robust} gets larger Bit Acc. However, it is worth noticing that when the modification is so strong, the PSNR goes below 15, meaning that the attack severely impacts the model functionality. 

Fig.~\ref{fig: ft_atk} shows the robustness against fine-tuning, where the models are fine-tuned for various iterations without the fingerprinting loss. We observe that  Per. Norm., Sta.~\cite{fei2023robust} Sig.~\cite{fernandez2023stable}, and WOUAF~\cite{kim2024wouaf} all have limited robustness. In contrast, our method is very robust, and after 10,000 steps of fine-tuning, the Bit Acc is still above 90\% (and is approximately 7\% lower when the regularization term $\mathcal{L}_{reg}$ is not considered).

\begin{table}[htbp]
\centering
\caption{
Robustness of the fingerprinted model under structural attacks on the COCO dataset.
}
\label{tab:white_box_attacks}
\resizebox{0.9\linewidth}{!}{
\begin{tabular}{@{}lcccc@{}}
\toprule
\textbf{Attack} & \textbf{PSNR}$_b$ & \textbf{FID} & \textbf{Bit Acc} \\ 
\midrule
No Attack                 & 31.80       & 24.03       & 99.57                          \\
PNM Removal              & 8.45        & 185.22      & 51.20                     \\
PNM Reinitialization           & 12.21       & 95.9        & 50.30                \\
\bottomrule
\end{tabular}
}
\end{table}

Finally, Table~\ref{tab:white_box_attacks} evaluates the robustness against intentional structural attacks. We define PNM removal as bypassing the module during inference and reinitialization as resetting its weights to random noise. While both attacks reduce the fingerprint extraction accuracy to approximately 50\%, they induce a catastrophic collapse in generation quality. For instance, PNM removal degrades the reconstruction PSNR from 31.80 dB to 8.45 dB and surges the FID to 185.22. This strong degradation occurs because the PNM is entangled with the VAE's feature representations in the fingerprint fine-tuning phase. This indicates that attackers must either retain a functional but traceable model or destroy the fingerprint at the cost of rendering the model useless, while recovering such severely degraded image quality via fine-tuning demands considerable effort.

\subsection{Ablation Study}
\label{Sec:Ablation Study}
\textbf{Quality-Effectiveness Trade-off.}
We further study the trade-off between image quality and fingerprint effectiveness by changing the weight of the fingerprint loss during fine-tuning. Specifically, we set the weight of the fingerprinting loss $\lambda_m$ to 10, 1, and 0.1, and report the corresponding image quality and Bit Acc on the COCO dataset (T2I). From Table~\ref{table: acc_trade} we observe that a larger $\lambda_m$ results in a degradation of image quality, which is up to 1 dB of PSNR, passing from 10 to 0.1. In terms of fingerprinting performance, passing from $\lambda_m = 0.1$ to $10$ yields only a slight improvement in Bit Acc. However, using a large $\lambda_m$ is beneficial for fingerprint robustness, especially against JPEG compression, in which case the Bit Acc improvement is above 4\%.

\begin{table}[ht]
    \small
    \renewcommand{\arraystretch}{1.2}
    \centering
    \caption{Image reconstruction quality and Bit Acc (\%, $\uparrow$) for different fingerprinting loss weight $\lambda_{\text{m}}$.
    }
    \begin{tabular}{m{1.7cm}<{\centering}|m{1.cm}<{\centering}m{0.9cm}<{\centering}m{1.cm}<{\centering}m{0.8cm}<{\centering}m{0.8cm}<{\centering}}
    \toprule
    \textbf{$\lambda_{\text{m}}$}  & \textbf{PSNR}$_b$ &  \textbf{SSIM}$_b$ & \textbf{no Attack} & \textbf{Crop} & \textbf{JPEG} \\ \midrule
    10  & 30.92 &   0.86 & 99.85 & 91.42  & 94.65 \\
    1  & 31.80 &   0.88 & 99.57 & 90.53  & 92.04 \\
    0.1  & 31.95 &   0.88 & 99.45 & 90.25  & 90.37 \\
    \bottomrule
    \end{tabular}
    \label{table: acc_trade}
\end{table}

\textbf{Fingerprint Decoder Architecture.}
In all the experiments reported so far, the fingerprint decoder $\mathcal{W}$  is based on EfficientNet-B0.
In this section, we study the impact of using a different architecture to implement the decoder. Specifically, we additionally trained other fingerprint decoders by considering different architectures, including ResNet-18 and ResNet-50, and used them to fine-tune the VAE decoder. All decoders were initialized with ImageNet pre-trained weights. In Table~\ref{table: arc_acc}, we report the results in terms of image reconstruction quality and Bit Acc on COCO. It can be observed that the differences in image reconstruction performance of the VAE are minimal, with fluctuations in PSNR of approximately 0.1. The fingerprint extraction accuracy is also similar in the three cases. The EfficientNet-B0-based decoder exhibits a slight advantage in terms of robustness against JPEG compression, with a Bit Acc more than 1\% higher.

\begin{table}[ht]
    \small
    \renewcommand{\arraystretch}{1.}
    \centering
    \caption{Impact of fingerprint decoder architecture on image reconstruction quality and Bit Acc (\%, $\uparrow$)}
    \begin{tabular}{m{2.1cm}<{\centering}|m{.7cm}<{\centering}m{0.7cm}<{\centering}m{1.cm}<{\centering}m{0.7cm}<{\centering}m{0.7cm}<{\centering}}
    \toprule
    \textbf{Architecture}  & \textbf{PSNR}$_b$ &  \textbf{SSIM}$_b$ & \textbf{no Attack} & \textbf{Crop} & \textbf{JPEG} \\ \midrule
    EfficientNet-B0     & 31.80 &   0.88    & 99.57     & 90.53  & 92.04 \\
    ResNet18            & 31.72 &   0.88    & 99.56     & 90.15  & 91.20 \\
    ResNet50            & 31.88 &   0.88    & 99.52     & 90.40  & 91.07 \\
    \bottomrule
    \end{tabular}
    \label{table: arc_acc}
\end{table}

\textbf{Complexity Analysis.}
We evaluated the computational overhead of the proposed method by measuring the average image generation time over 10,000 images. On a single Nvidia A100 GPU, for the no-fingerprint model, the average time on COCO and ImageNet is 4,807 ms and 4,646 ms, respectively. For our PNM fingerprinted model, the corresponding times are 4,818 ms and 4,708 ms, with an increase of only 11 ms and 62 ms, i.e., within 1.5\%. Moreover, in terms of model parameters, PNM introduces only additional $4C$ parameters ($C$ parameters respectively for $\beta$, $\gamma$, conv1 and conv2).
\section{Anti-Collusion Analysis}
\label{Sec: Anti-collusion Evaluations}

In this section, we reveal the widespread systemic vulnerability in current methods: multiple attackers can easily synthesize a new model whose fingerprint does not reliably match that of any of the participating parties, via model collusion. 
Parameter-level collusion represents a unique
attack in modern open-weight or locally-licensed deployments. Since it requires no retraining and preserves generation quality, it represents a dangerous zero-cost threat to IPR protection. We demonstrate that the proposed ACT can address such a threat. Instead of relying on post-hoc accountability, ACT operates as a proactive defense, preventing colluders from synthesizing an effective model.
We first focus on the common case of collusion attacks using linear parameter averaging~\cite{wortsman2022model}, then we also consider the case in which more advanced non-linear collusion attacks are adopted.

\subsection{Collusion Attack Settings}
\label{sec:Collusion Attack Settings}
We consider both 2-party and multi-party collusion, as well as linear and non-linear attacks.
In the simple case of 2-party linear collusion, let $\mathcal{M}^{(a)}$ and $\mathcal{M}^{(b)}$ be the fingerprinted model distributed to users $a$ and $b$ with user-specific fingerprints $\bm{m}^{(a)}$ and $\bm{m}^{(b)}$. The attacker constructs a colluded model $\mathcal{M}_m^{(atk)}$ as $\mathcal{M}_m^{(atk)} = \alpha \mathcal{M}_m^{(a)} + (1-\alpha) \mathcal{M}_m^{(b)}$, where $0 <\alpha <1$. We vary $\alpha$ from 0 to 1 to simulate different interpolation scenarios and assess how the fingerprinting performance and image generation quality are affected. We further extend the evaluation to multi-party linear collusion, where $n$ users collaborate to produce a colluded model. In this case, the attacker computes a weighted sum of $N$ fingerprinted models: $\mathcal{M}^{(atk)}_m = \sum_{i=1}^{N} \alpha^i \mathcal{M}^{(i)}_m$, where $\sum_{i=1}^{N} \alpha^i = 1$ and $\alpha^i > 0$. We use $\alpha^i = 1/n$ to simulate equal contributions from all colluders. We also evaluate our method on non-linear collusion attacks, the details are presented in Sec.~\ref{sec:Non-linear Collusion}.

\subsection{Evaluation under Linear Collusion Attacks}
\label{sec:linear Collusion}
In this subsection, we evaluate the robustness of the proposed method against linear collusion attacks, considering both 2-party and multi-party scenarios. We provide a comparative analysis with state-of-the-art fingerprinting methods to demonstrate the effectiveness of the proposed ACT.

\subsubsection{Comparative Analysis on 2-Party Collusion}
\label{sec:Comparative Analysis on 2-Party Collusion}

We first focus on the 2-party collusion scenario with equal weights ($\alpha=0.5$), which represents the most straightforward attempt to remove fingerprints. 
We evaluate both the image generation quality (FID, based on the COCO dataset) and the Bit Acc between the two user-specific fingerprints  $\bm{m}^{(a)}$ and $\bm{m}^{(b)}$ and the fingerprint extracted from the colluded model, indicated with $\bm{m}^{(atk)}$. Let $p^{(a)} = \text{Acc}(\bm{m}^{(atk)}, \bm{m}^{(a)})$ and $p^{(b)} =  \text{Acc}(\bm{m}^{(atk)}, \bm{m}^{(b)})$.

\begin{table}[h]
\renewcommand{\arraystretch}{1.1}
\centering
\caption{Image generation (FID) and fingerprint verification performance of colluded models for different methods.}
\label{tab:collusion_eval}
\begin{tabular}{m{2.cm}<{\centering}|m{1.cm}<{\centering}m{1.cm}<{\centering}m{1.cm}<{\centering}m{1.cm}<{\centering}}
\toprule
\textbf{Method} & \textbf{FID} & $p^{(a)}$ & $p^{(b)}$ & \textbf{TPR} \\
\midrule
Per. Norm.~\cite{fei2023robust}     & 24.19 & 76.41 &76.26 & 0.534\\
Sta. Sig.~\cite{fernandez2023stable}     & 23.82 & 75.41 & 74.95 & 0.407 \\
WOUAF~\cite{kim2024wouaf}         &  23.65 & 76.77 & 75.20 & 0.469\\
AquaLoRA~\cite{feng2024aqualora}      & 27.79 & 70.59 & 71.42 & 0.451\\
\midrule
Ours w/o ACT  & 23.55 & 74.61 &74.73 & 0.468 \\
Ours with ACT & 79.51 & 75.92 &75.45 & 0.467 \\
\bottomrule
\end{tabular}
\end{table}

\textbf{Vulnerability of Existing Methods.} 
Table~\ref{tab:collusion_eval} shows the results, averaged over 20 pairs of randomly sampled fingerprints $(\bm{m}^{(a)},\bm{m}^{(b)})$. In addition, we compute the TPR of the colluding model under the claim-based verification protocol, at a fixed FPR $10^{-4}$ (the threshold $\tau = 0.85$ used before can also be adopted. The FPR we get would be even lower in this case). We observe that for all methods, two colluding attackers can 
substantially weaken the fingerprint information in the colluded model without degrading the quality of the generated images. Specifically, the matching accuracy with the fingerprints of participants $(a)$  and $(b)$  drops to around 75\%, while in the verification scenario, the TPR at an FPR of $1\times 10 ^{-4}$  falls to only 0.4–0.5. This indicates that the attackers can compromise existing generative model fingerprinting systems without any prior knowledge or sophisticated technical capabilities. 

\textbf{Effectiveness of ACT.} 
In contrast, our method with ACT demonstrates a proactive defense capability. As reported in Table~\ref{tab:collusion_eval}, although the fingerprint verification performance is similarly affected by parameter averaging (which is mathematically inevitable), the ACT causes the image quality to degrade drastically, with the FID surging by over 50 points compared to the baseline. This severe degradation renders the colluded model effectively unusable, thereby precluding 
the attack. Notably, without ACT, our method exhibits similar vulnerability to existing approaches, confirming that ACT is the key component for disruption.

\subsubsection{Impact of Interpolation Factor ($\alpha$)}
\label{sec:alpha_impact}
To further analyze the behavior of the models, we evaluate the performance under varying interpolation factors $\alpha \in [0, 1]$.

\textbf{Fingerprint Matching Analysis.} 
Fig.~\ref{fig:acc_vs_alpha} reports the Bit Accuracy between the extracted fingerprint $\bm{m}^{(atk)}$ and the users' original fingerprints. The results are averaged over 20 pairs, considering 1,000 images each. As expected, as $\alpha$ increases from 0 to 1, $\text{Acc}(\bm{m}^{(atk)}, \bm{m}^{(a)})$ decreases while $\text{Acc}(\bm{m}^{(atk)}, \bm{m}^{(b)})$ increases. Both metrics reach their minimum ($\approx 75\%$) at $\alpha=0.5$, confirming that the fingerprint signal is maximally obfuscated at the midpoint.

\textbf{Image Quality Analysis.} 
The crucial difference lies in the image quality preservation. Fig.~\ref{fig:psnr_vs_alpha} plots the PSNR of generated images against $\alpha$.
For existing methods and ours without ACT, the PSNR remains stable and high across the entire range of $\alpha$, confirming their vulnerability to collusion attacks.
For ours with ACT, the ACT disrupts this connectivity. As shown in Fig.~\ref{fig:psnr_vs_alpha} and the visual examples in Fig.~\ref{fig: vis_coll}, any deviation from the original parameters (i.e., $\alpha \neq 0, 1$) leads to a sharp deterioration in quality. Even a small interpolation (e.g., $\alpha=0.1$) causes severe degradation. This effect is consistent across diverse prompts.

\begin{figure}[H]
    \centering
    \begin{minipage}{0.24\textwidth}
        \centering
        \includegraphics[width=\linewidth]{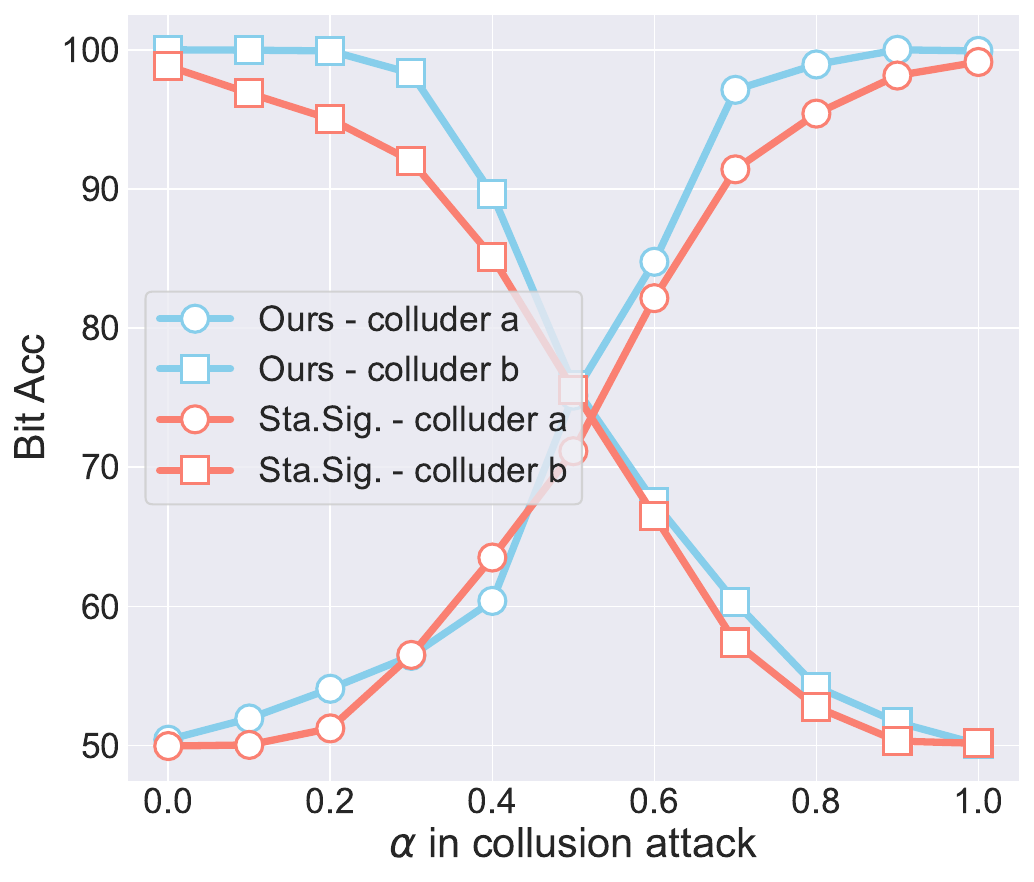}
        \subcaption{Bit Acc vs. $\alpha$}
        \label{fig:acc_vs_alpha}
    \end{minipage}
    \hfill
    \begin{minipage}{0.24\textwidth}
        \centering
        \includegraphics[width=\linewidth]{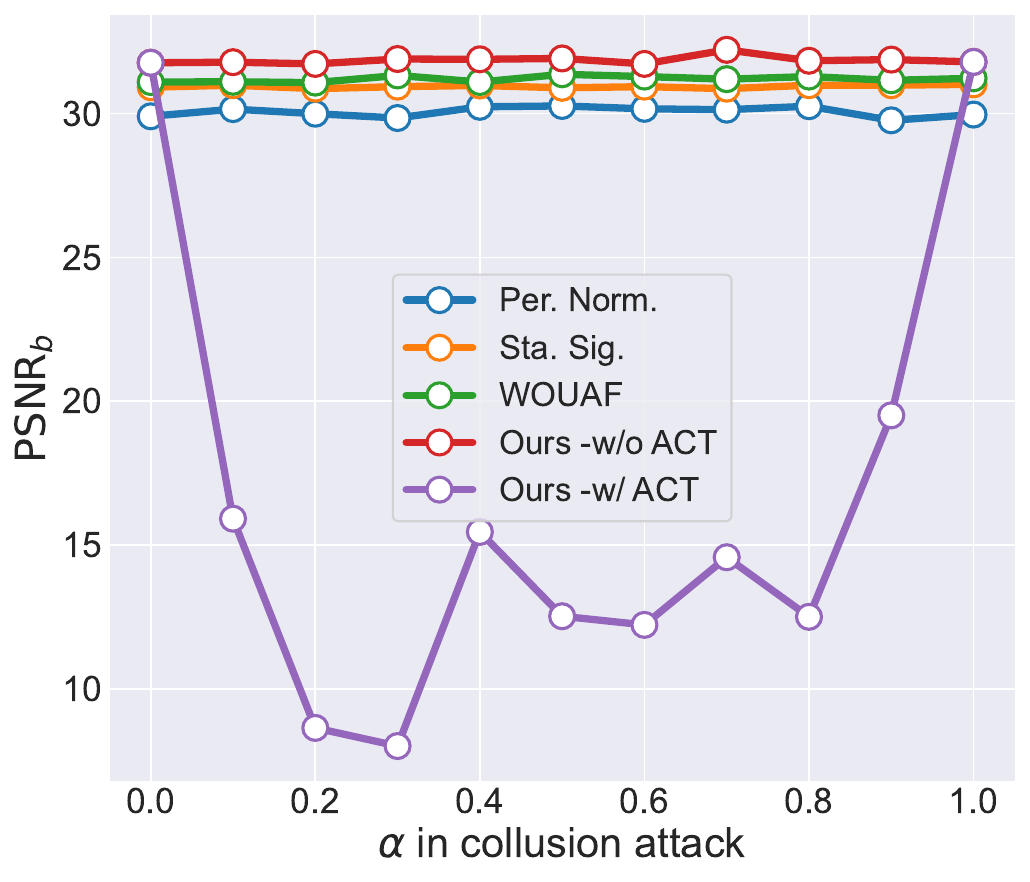}
        \subcaption{PSNR vs. $\alpha$}
        \label{fig:psnr_vs_alpha}
    \end{minipage}
    \caption{PSNR and fingerprint Bit Acc (\%) of colluded models for $\alpha$ (for sake of visibility, in (a), only our method and Sta. Sig. are reported. The behavior is similar in the other cases.}
    \label{fig: collusion_plts}
\end{figure}

\begin{figure}[h]
\centering
    \includegraphics[width=0.49\textwidth]{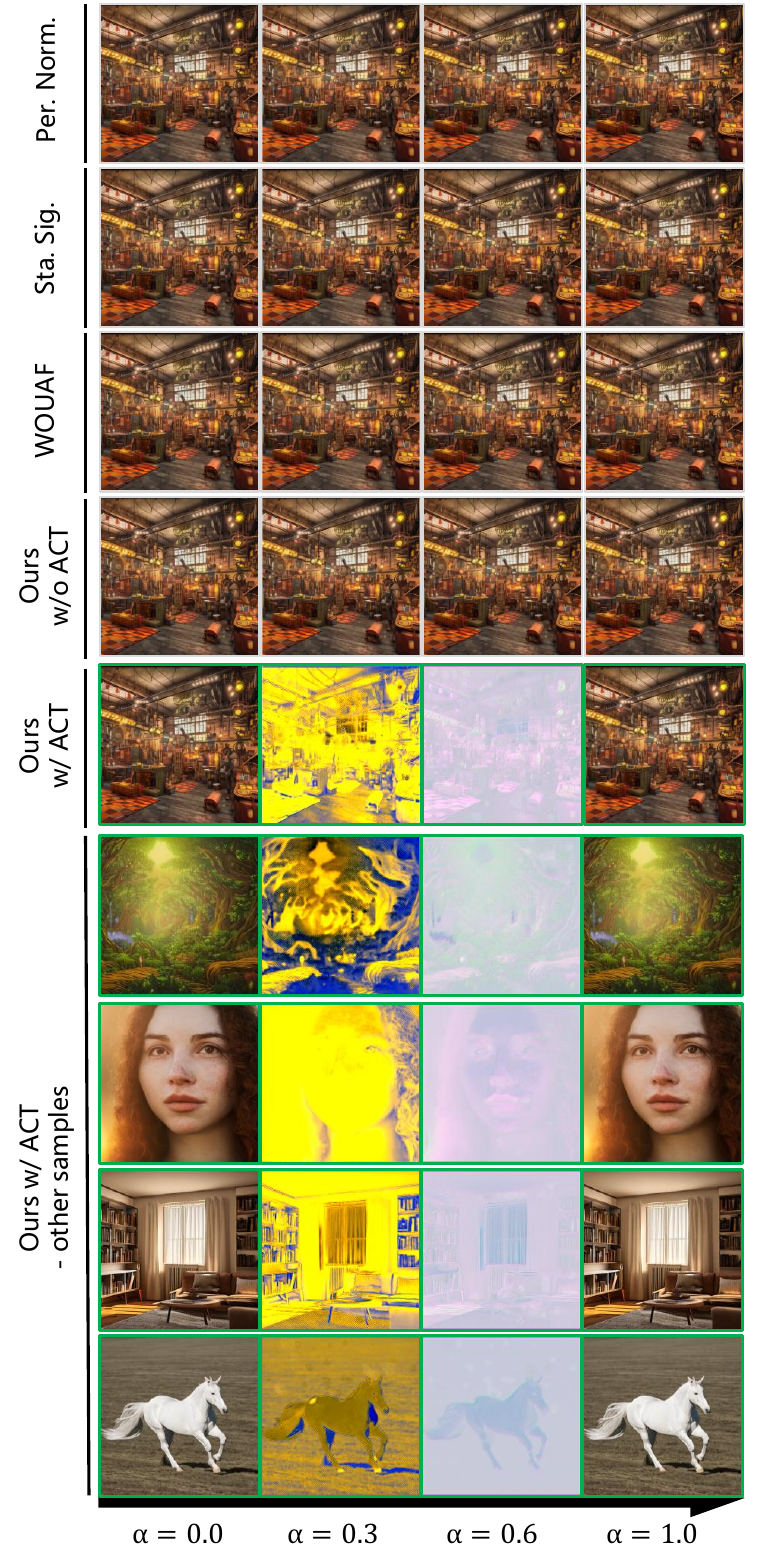}
    \caption{Images generated by 2-party colluded models with different $\alpha$. Row 1–2: Results from Sta. Sig.\cite{fernandez2023stable} and WOUAF\cite{kim2024wouaf}. Row 3: Ours without ACT. Row 4: Ours with ACT (same prompt). Row 5–7: Ours with ACT on different prompts. 
    }
\label{fig: vis_coll}
\end{figure}

\subsubsection{Comparative Analysis on Multi-party Collusion}
Table~\ref{tab:collusion_by_method} reports the results for different $N$.
Existing methods maintain high fidelity for the colluded models, with PSNR above 29 dB, which indicates that the colluded models preserve their functionality (LPIPS is also reported). However, their Bit Acc drops significantly (typically below 70\%), highlighting the general vulnerability of fingerprinting schemes to multi-party collusion attacks. 

In the case of our method, although the Bit Acc also decreases as $N$ increases, a strong degradation in generation capability is observed, which makes the colluded model unusable. For instance, with $N=10$, the PSNR drops to 11.44 dB, and LPIPS rises to 0.65, resulting in an unusable model. These observations further suggest that increasing the number of colluding parties does not enhance the image quality of the colluded model. Instead, the quality remains constrained by the ACT. As a result, the collusion attack remains ineffective, since attackers are unable to both preserve high-fidelity generation and simultaneously degrade fingerprint reliability. These observations further suggest that increasing the number of colluding parties does not enhance the image quality of the colluded model. 


\begin{table}[h]
\renewcommand{\arraystretch}{1.25}
\centering
\caption{Effect of different numbers of colluders ($N$) on generation quality and fingerprint extraction accuracy. Bit Acc: mean $\pm$ std over $N$ colluders. Quality: $\text{PSNR}\downarrow$ / $\text{LPIPS}\downarrow$}
\label{tab:collusion_by_method}
\begin{tabular}{@{}m{1.2cm}<{\centering}@{}|@{}m{1.1cm}<{\centering}@{}|m{1.2cm}<{\centering}m{1.2cm}<{\centering}m{1.2cm}<{\centering}m{1.2cm}<{\centering}}
\toprule
\textbf{Method} & \textbf{Metric} &\textbf{N=3} & \textbf{N=5} & \textbf{N=10} & \textbf{N=20} \\
\midrule

\multirow{2}{*}{\cite{fei2023robust}} & Bit Acc  & 75.69$\pm$8.72  & 68.75$\pm$9.59& 63.54$\pm$7.52 & 57.08$\pm$6.85 \\
  & Quality  & 30.05/0.06 & 29.88/0.06 & 29.97/0.06 & 30.12/0.07 \\ \midrule

\multirow{2}{*}{~\cite{fernandez2023stable} }  & Bit Acc  &71.52$\pm$6.87  &66.66$\pm$6.58 &62.08$\pm$8.05  & 56.04$\pm$8.98 \\ 
& Quality & 30.92/0.08 & 30.87/0.08 & 31.03/0.08 & 30.71/0.08 \\ \midrule

\multirow{2}{*}{~\cite{kim2024wouaf}} & Bit Acc   & 73.61$\pm$2.59 & 69.16$\pm$5.17& 60.83$\pm$7.44 & 57.60$\pm$4.85 \\ 
   & Quality       & 31.04/0.07 & 30.96/0.07 & 31.28/0.08 & 31.15/0.0 \\ \midrule

\multirow{2}{*}{Ours} & Bit Acc   &77.08$\pm$1.70  & 71.25$\pm$9.44&  62.08$\pm$5.57 & 58.85$\pm$9.52 \\ 
  & Quality  & \textbf{12.64/0.59} & \textbf{11.76/0.68} & \textbf{11.44/0.65} & \textbf{10.96/0.70} \\
\bottomrule
\end{tabular}
\end{table}

\begin{figure*}[ht]
    \centering
    \begin{minipage}{0.24\textwidth}
        \centering
        \includegraphics[width=\linewidth]{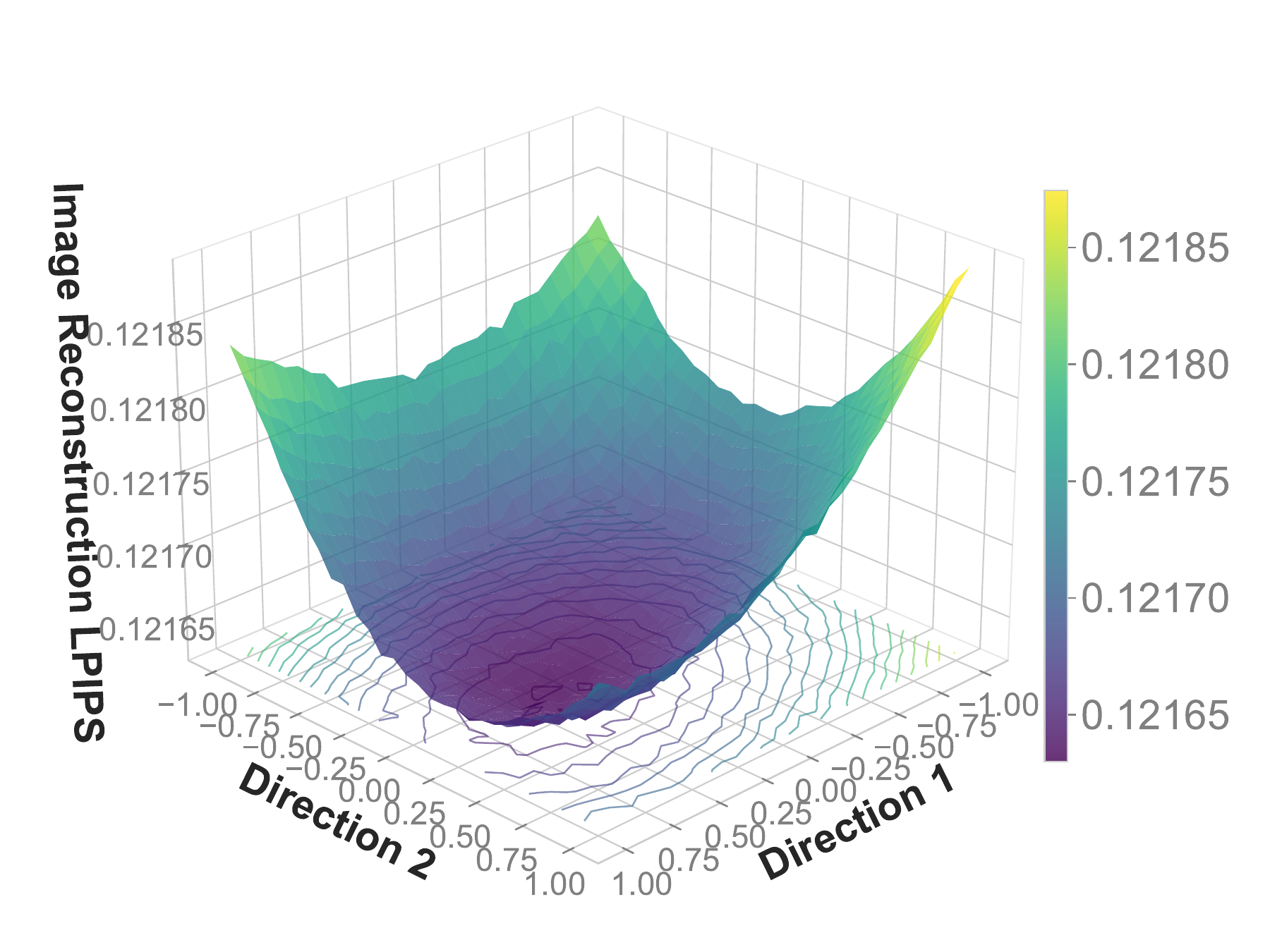}
        \subcaption{Fingerprint 1 (w/o ACT)}
    \end{minipage} \hfill
    \begin{minipage}{0.24\textwidth}
        \centering
        \includegraphics[width=\linewidth]{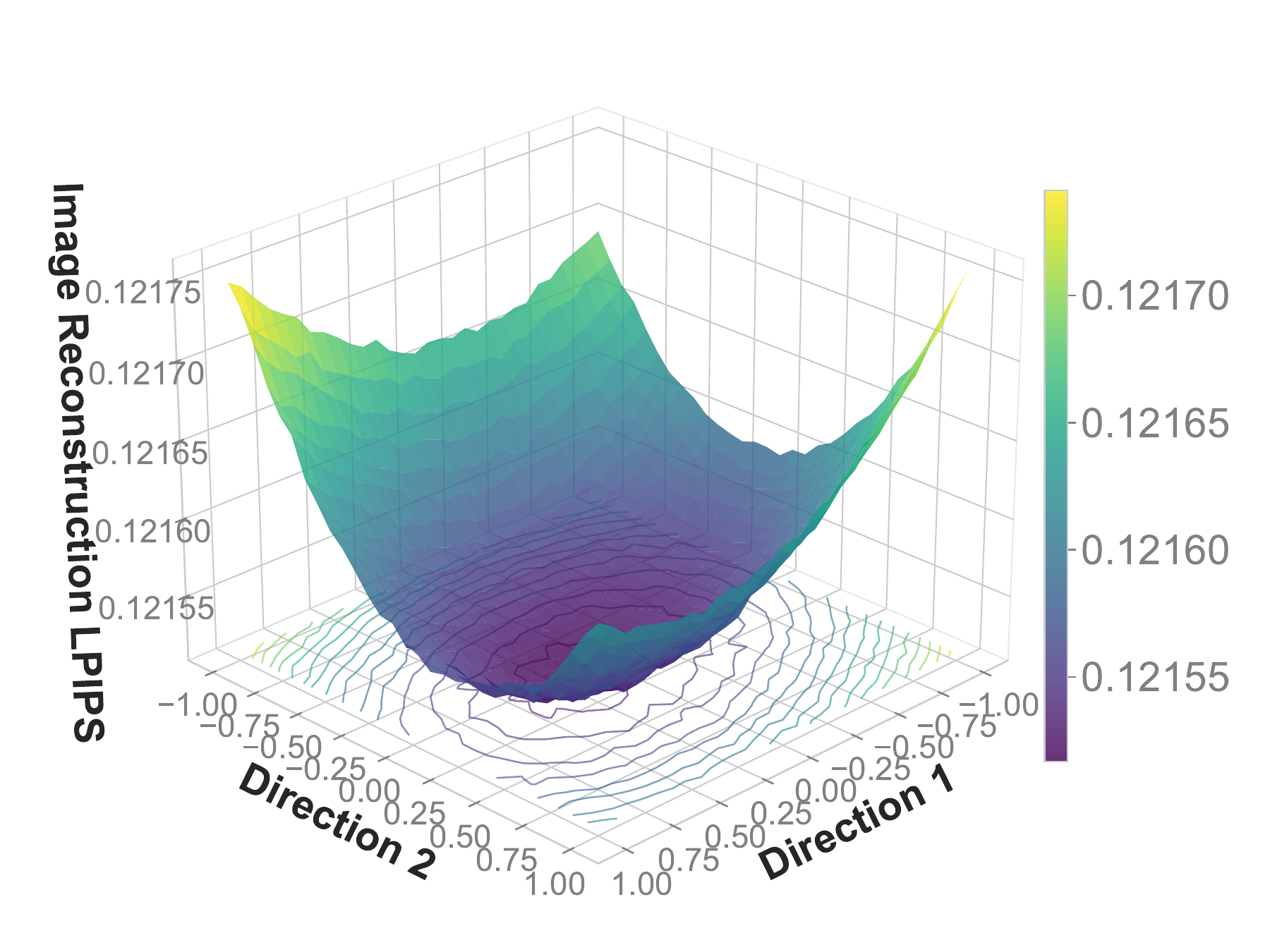}
        \subcaption{Fingerprint 2 (w/o ACT)}
    \end{minipage} \hfill
    \begin{minipage}{0.24\textwidth}
        \centering
        \includegraphics[width=\linewidth]{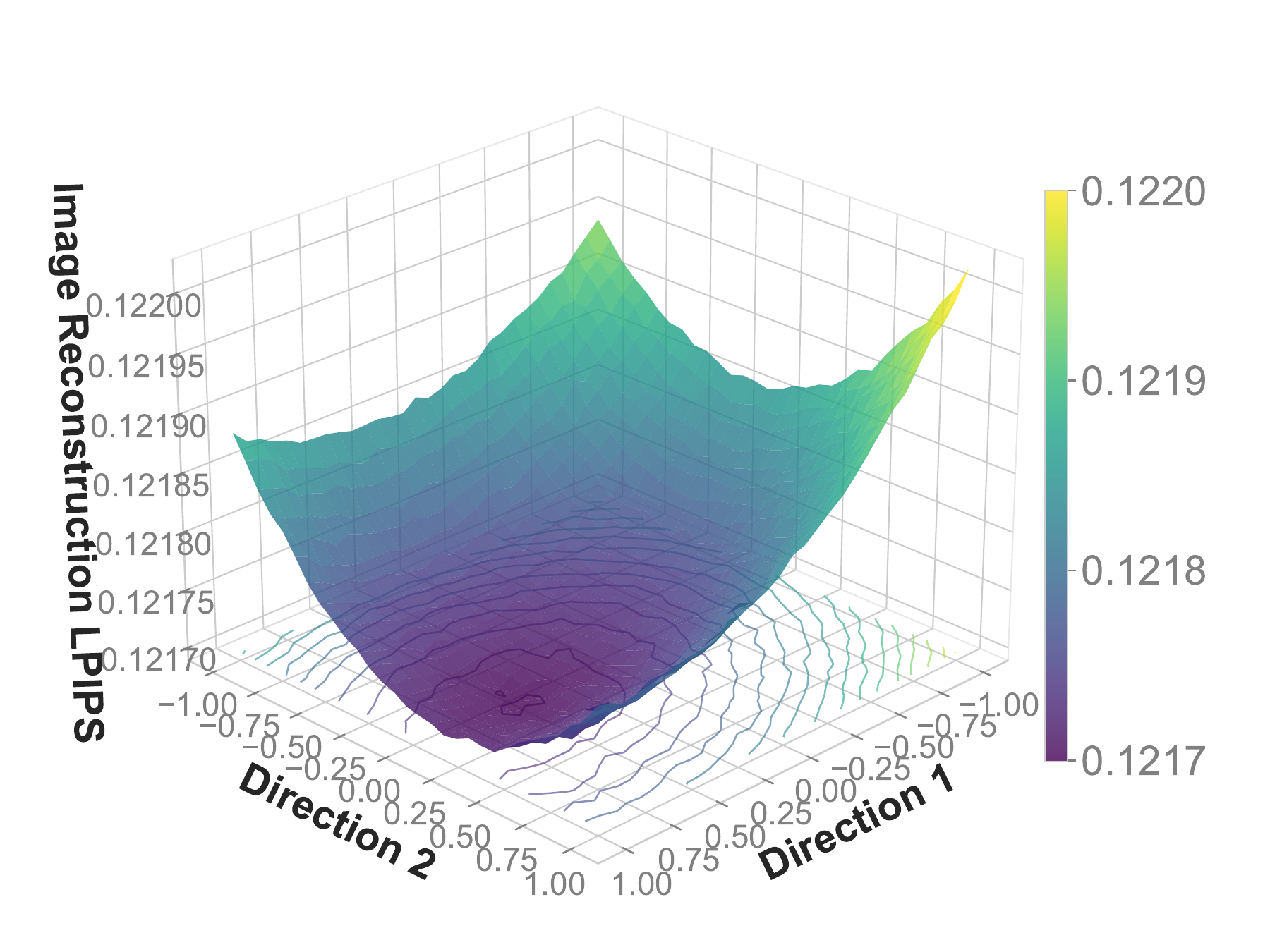}
        \subcaption{Fingerprint 3 (w/o ACT)}
    \end{minipage} \hfill
    \begin{minipage}{0.24\textwidth}
        \centering
        \includegraphics[width=\linewidth]{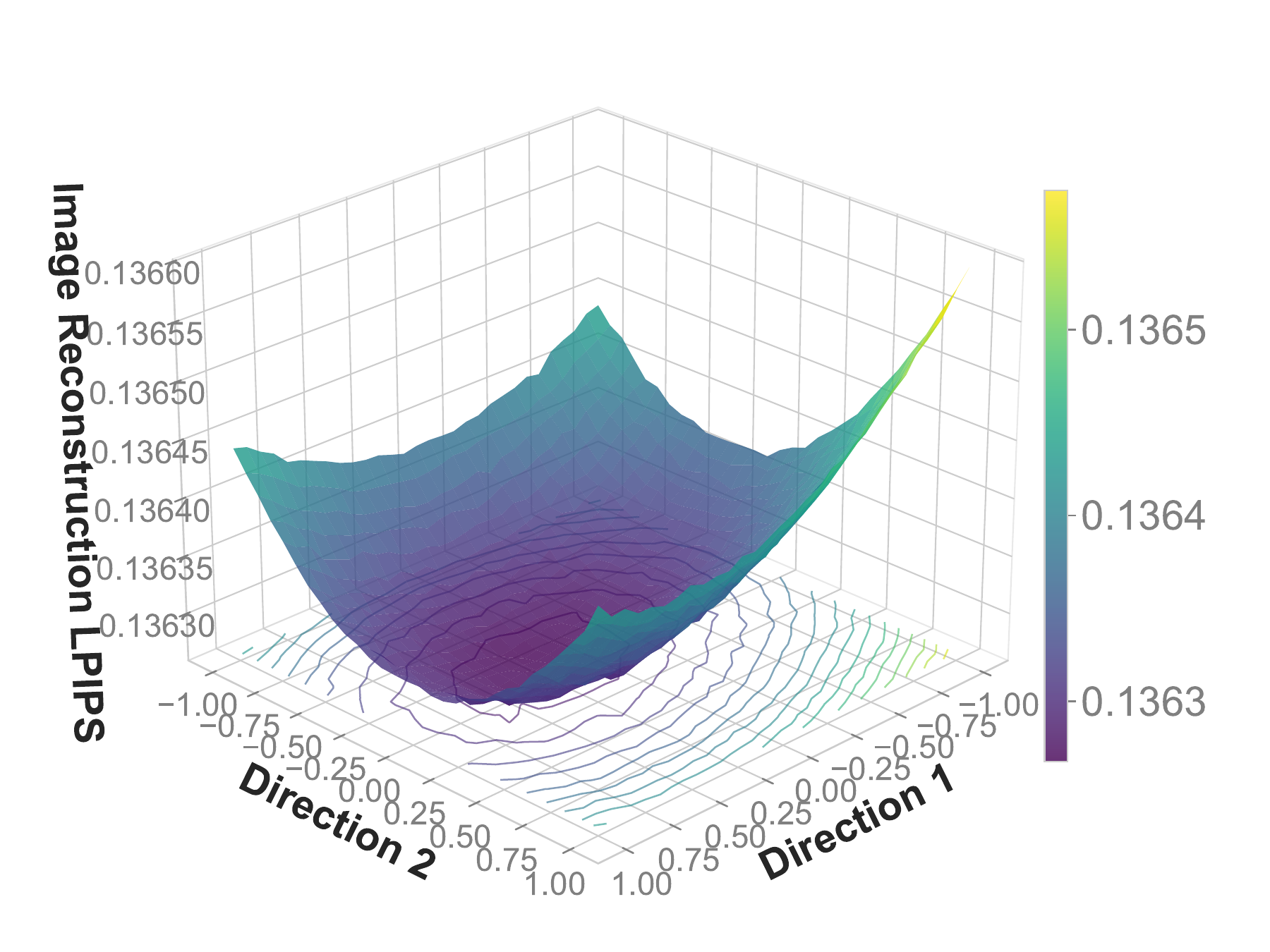}
        \subcaption{Fingerprint 4 (w/o ACT)}
    \end{minipage}

    \begin{minipage}{0.24\textwidth}
        \centering
        \includegraphics[width=\linewidth]{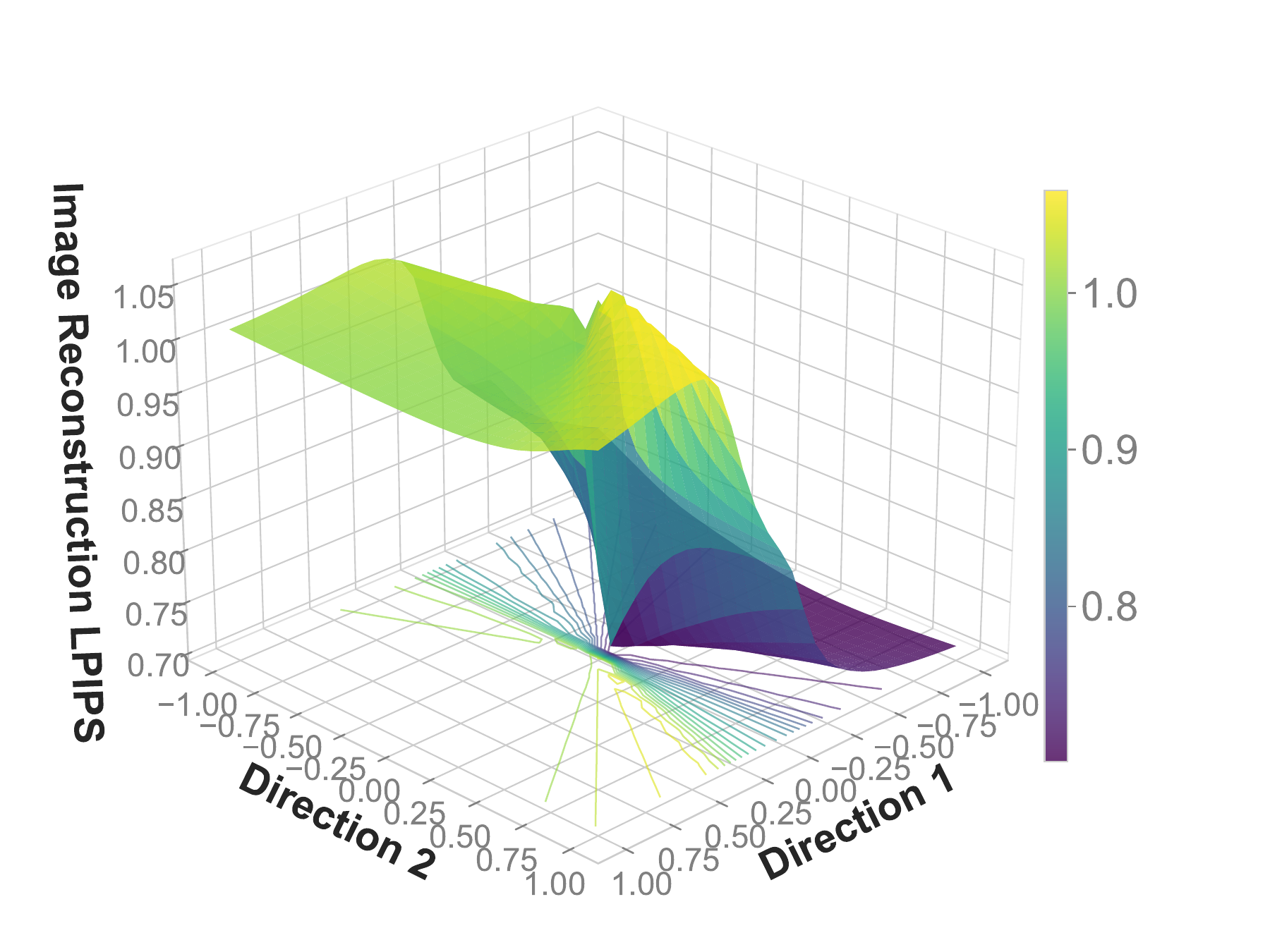}
        \subcaption{Fingerprint 1 (with ACT)}
    \end{minipage} \hfill
    \begin{minipage}{0.24\textwidth}
        \centering
        \includegraphics[width=\linewidth]{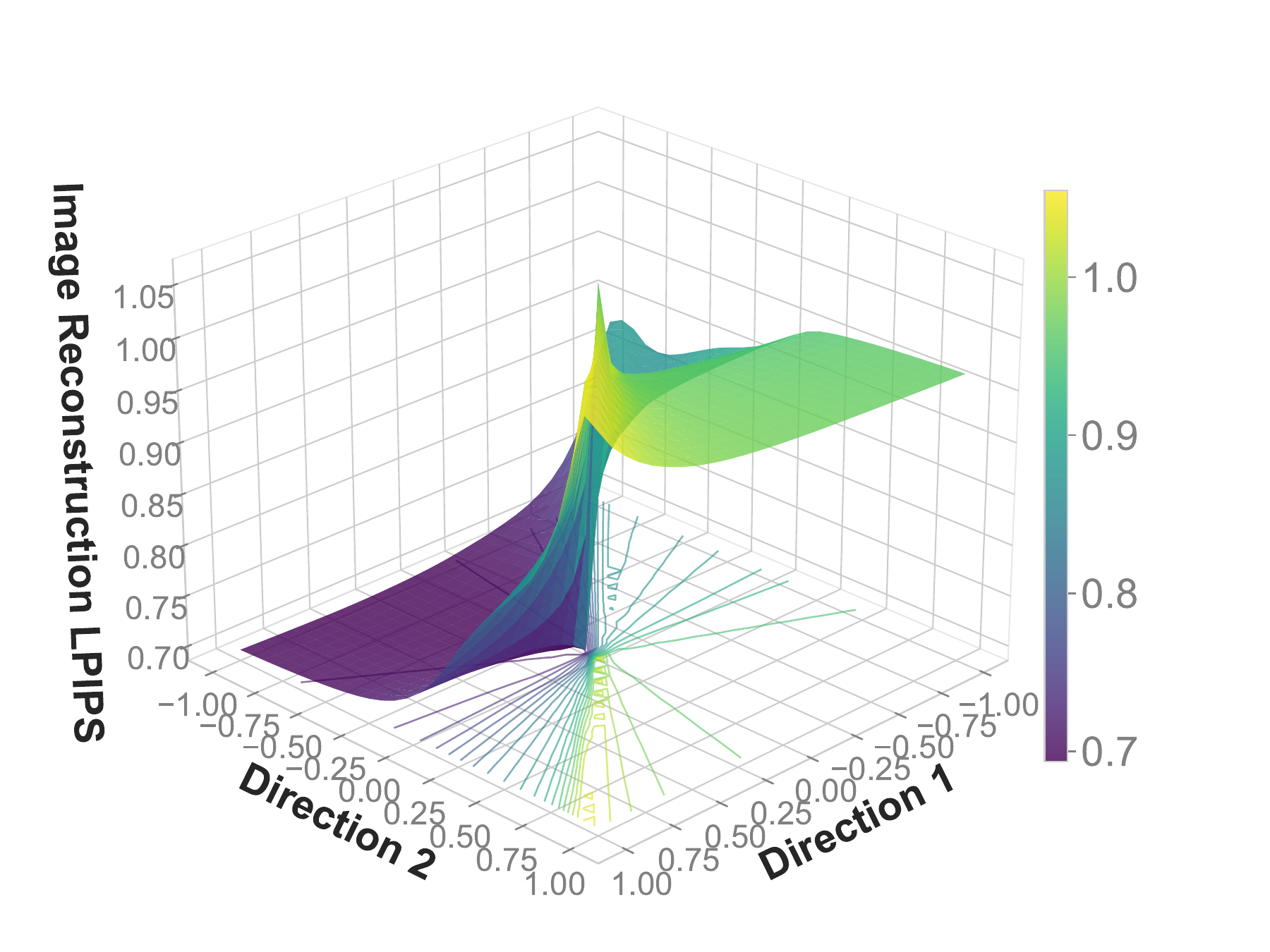}
        \subcaption{Fingerprint 2 (with ACT)}
    \end{minipage} \hfill
    \begin{minipage}{0.24\textwidth}
        \centering
        \includegraphics[width=\linewidth]{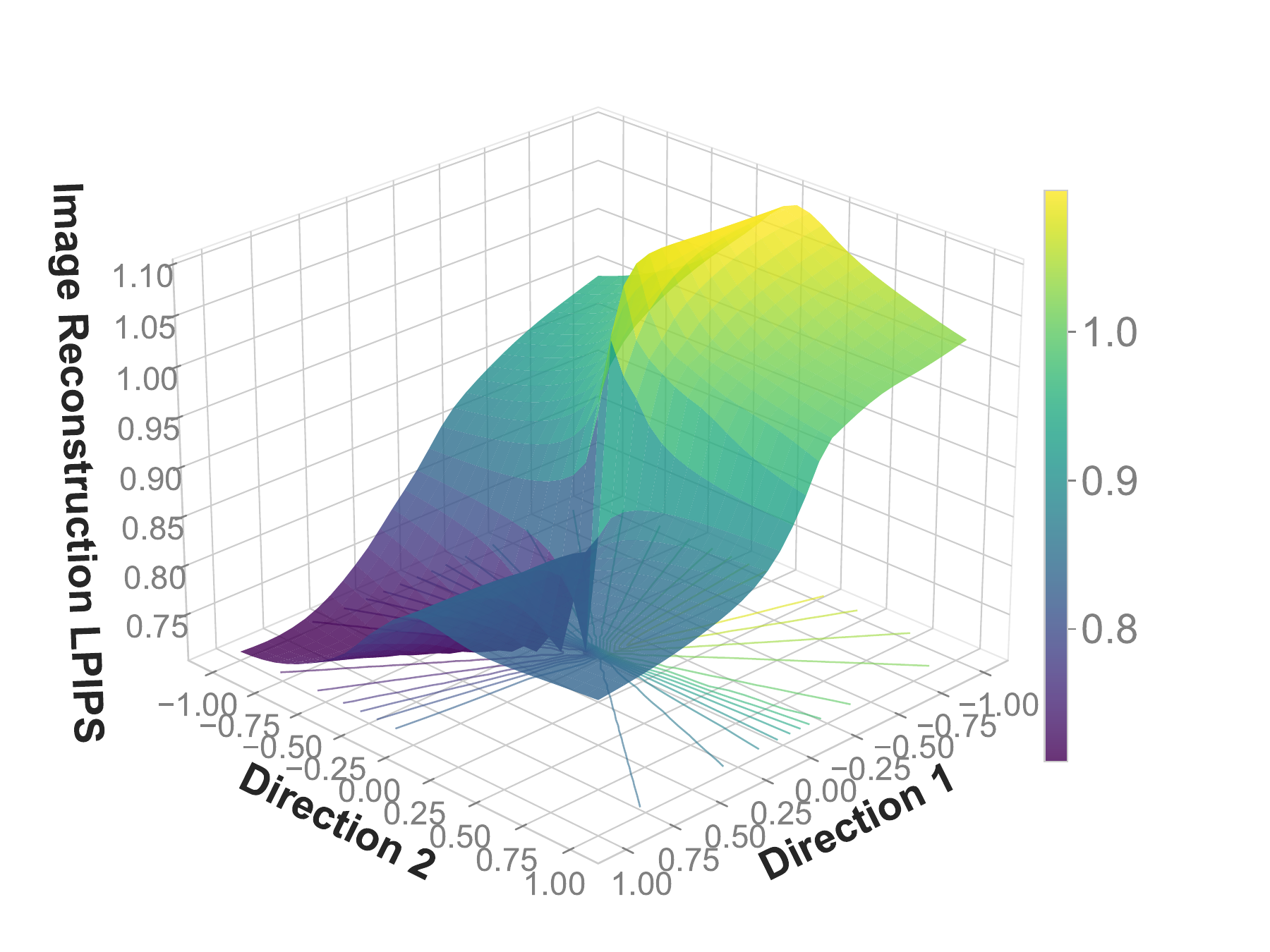}
        \subcaption{Fingerprint 3 (with ACT)}
    \end{minipage} \hfill
    \begin{minipage}{0.24\textwidth}
        \centering
        \includegraphics[width=\linewidth]{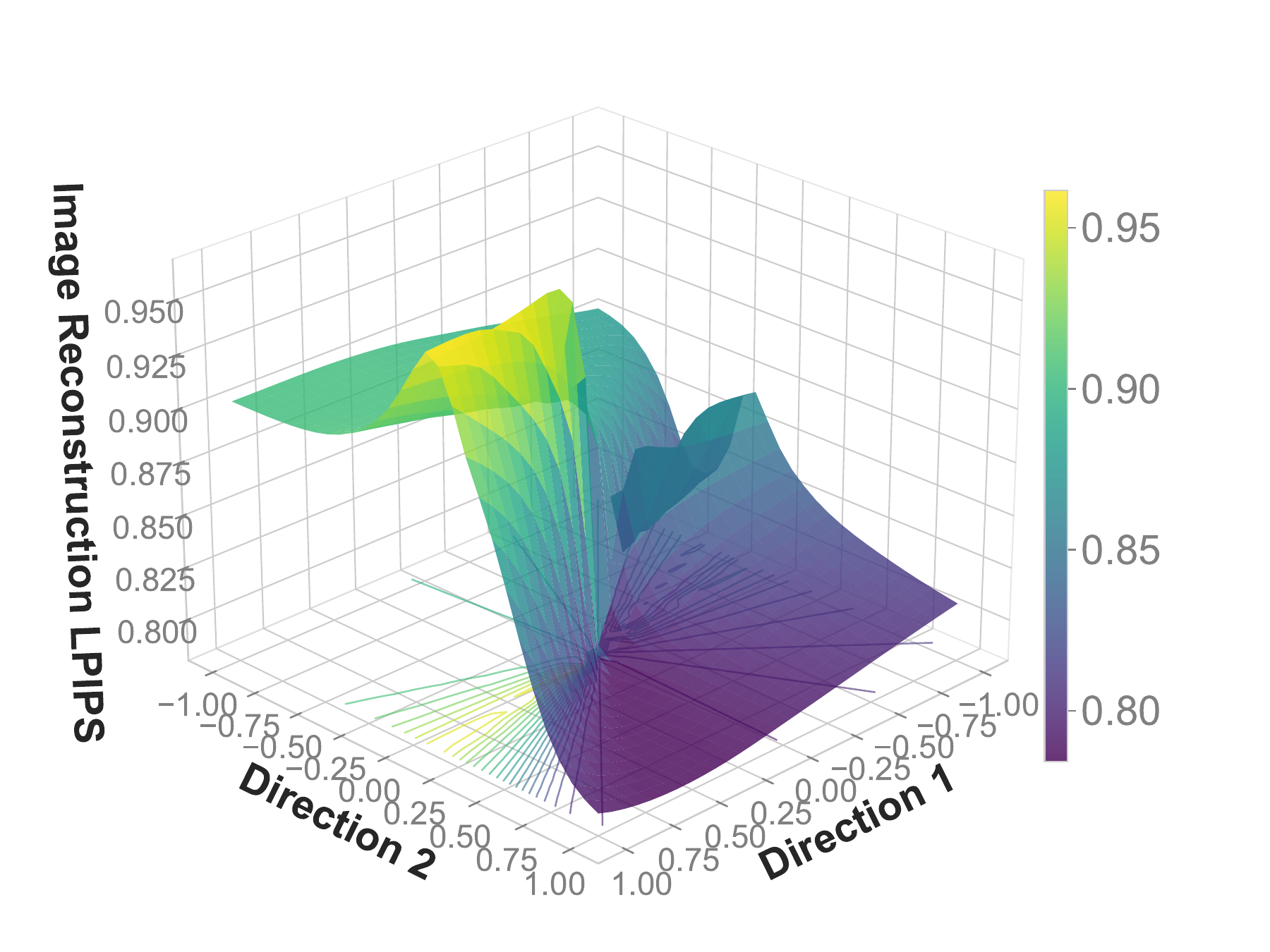}
        \subcaption{Fingerprint 4 (with ACT)}
    \end{minipage}
\caption{LPIPS landscapes of models with different fingerprints, plotted using~\cite{li2018visualizing}. First row: models with different fingerprints without ACT; second row: models with different fingerprints after ACT. Note that in both cases, the loss at (0,0) remains low, but with ACT it rises sharply with even slight perturbations, making this local minimum less visually discernible.}
\label{fig: land}
\end{figure*}

\subsection{Evaluation under Non-linear Collusion Attacks}
\label{sec:Non-linear Collusion}
We also evaluate more advanced non-linear collusion attack strategies, including the \textit{Random} strategy, the \textit{Power mean} based parameter merging strategy, and the \textit{max-absolute-value} parameter merging strategy, described below. 

(1) \textit{Random} strategy: each parameter is selected independently from either model $a$ or $b$ with equal probability.

(2) \textit{Power mean} based parameter merging strategy: Given two fingerprinted models with $\theta_1$ and $\theta_2$, the colluded model's parameters $\theta_{\text{atk}}$ are calculated as:
\begin{equation}
\theta_{\text{atk}} = \left( \frac{\theta_1^p + \theta_2^p}{2} \right)^{\frac{1}{p}},
\label{eq:power-mean}
\end{equation}
where $p > 0$ is a parameter controlling the non-linearity. 

(3) \textit{max-absolute-value} parameter merging strategy. For each parameter $i$, the colluded model selects the parameter with the larger absolute value:
\begin{equation}
\theta_{\text{atk}}^{(i)} = 
\begin{cases}
\theta_1^{(i)}, & \text{if } \left| \theta_1^{(i)} \right| \geq \left| \theta_2^{(i)} \right|, \\
\theta_2^{(i)}, & \text{otherwise}.
\end{cases}
\label{eq:max-abs-merge}
\end{equation}

\begin{table}[htbp]
\centering
\caption{Evaluation under non-linear 2-party collusion.}
\begin{tabular}{m{2.6cm}<{\centering}|m{0.7cm}<{\centering}m{0.7cm}<{\centering}m{1.cm}<{\centering}m{1.cm}<{\centering}}
\toprule
\textbf{Collusion Strategy} & $\textbf{PSNR}_b$ & $\textbf{LPIPS}_b$ & \textbf{$p^{(a)}$}  & \textbf{$p^{(b)}$}  \\
\midrule
Random                   & 5.21& 0.87 &74.51 &76.49 \\
Power Mean ($p=2$)       & 5.51& 0.82 &75.61 &74.72 \\
Max-Abs Merge            & 6.24& 0.84 &75.23 &74.49  \\
\bottomrule
\end{tabular}
\label{tab:nonlinear-collusion}
\end{table}

Table~\ref{tab:nonlinear-collusion} shows the results of our method under advanced non-linear 2-party collusion attacks. We evaluate both visual fidelity (PSNR, LPIPS) and the Bit Acc with respect to each colluding model, i.e., $p^{(a)}$, $p^{(b)}$.
We can see that all evaluated strategies fail to preserve both the visual quality and fingerprints, as indicated by PSNR values between 5 dB and 6 dB, high LPIPS scores ranging from 0.82 to 0.87. Overall, these results confirm that all tested non-linear collusion attacks are ineffective against our fingerprinting method.

\subsection{Impact of ACT on Anti-collusion Capability}
\label{sec:Anti Collusion Analysis}

\begin{table}[htbp]
\centering
\caption{The effectiveness of different transformations in ACT against 2-party and 10-party collusion, i.e., the performances of the colluded model (FID$\uparrow$; PSNR $\downarrow$; LPIPS$\uparrow$).}
\begin{tabular}{m{0.3cm}<{\centering}m{0.3cm}<{\centering}m{0.3cm}<{\centering}|m{0.6cm}<{\centering}m{0.6cm}<{\centering}m{0.6cm}<{\centering}m{0.6cm}<{\centering}m{0.6cm}<{\centering}m{0.6cm}<{\centering}}
\toprule
  \multirow{2}{*}{CP} &  \multirow{2}{*}{CS} &  \multirow{2}{*}{SF}   & \multicolumn{3}{c}{2-party} & \multicolumn{3}{c}{10-party} \\ \cmidrule(lr){4-6} \cmidrule(lr){7-9}
 & &  & FID &PSNR &LPIPS & FID &PSNR &LPIPS  \\
\midrule
 \checkmark&  &   & 72.49 &11.54 &0.51 & 81.21 &11.23& 0.54\\
 &\checkmark&       &76.45 &13.70& 0.45& 80.58& 12.68&0.52\\
 &&      \checkmark & 62.17& 12.92& 0.49&76.62 & 11.15 &0.51 \\
\bottomrule
\end{tabular}
\label{tab: abl_act}
\end{table}

\subsubsection{Ablation Study of ACT}
We analyze the impact of different transformations in ACT on anti-collusion performances. As shown in Table~\ref{tab: abl_act}, we evaluate linear collusion under 2-party and 10-party scenarios, and measure the generative and reconstruction capabilities of the colluded models based on the ImageNet. It can be observed that even a single transformation significantly degrades both image generation quality and reconstruction performance. Specifically, CP achieves the largest increase in FID and LPIPS under both collusion settings, indicating its strong disruption of image quality. These results suggest that each transformation contributes to anti-collusion.

\begin{figure}[h]
    \centering
    \begin{minipage}{0.24\textwidth}
        \centering
        \includegraphics[width=\linewidth]{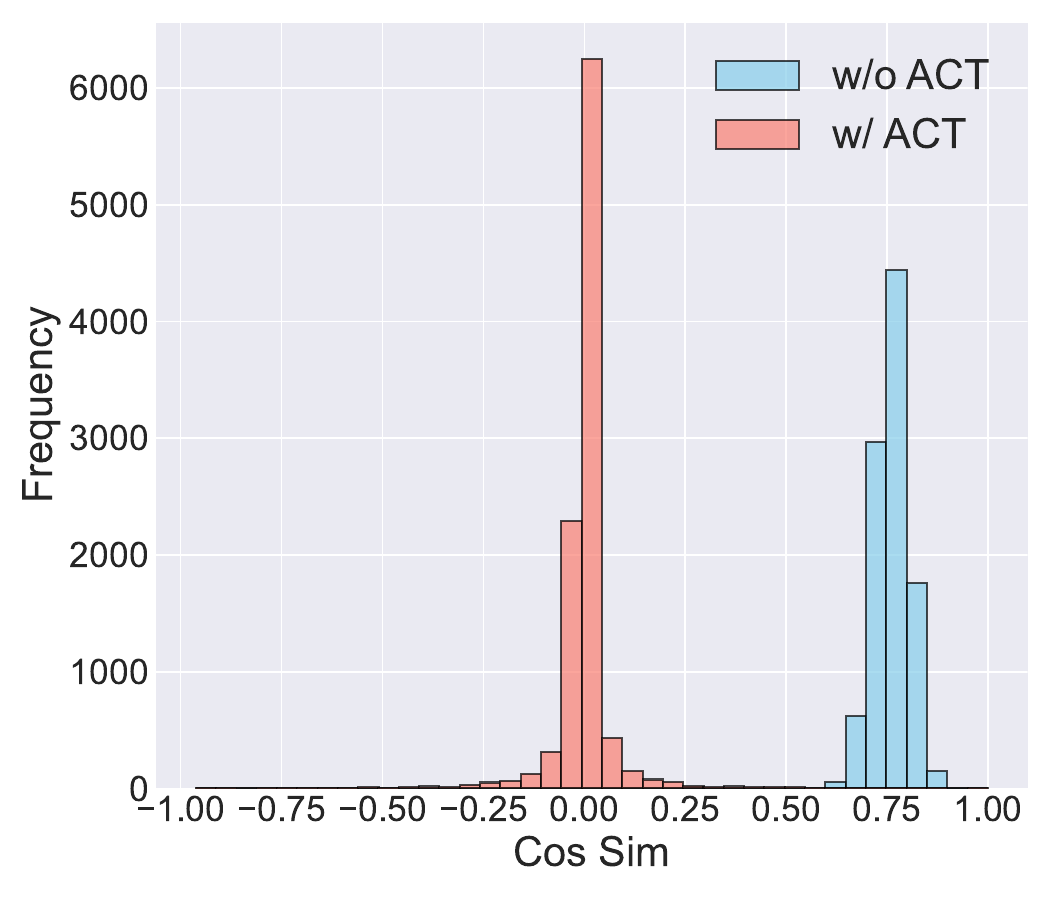}
        \subcaption{Distribution of cosine similarity for $\bm{\gamma}^{(a)}$ and $\bm{\gamma}^{(b)}$, with and without ACT.}
        \label{fig: cosa}
    \end{minipage}
    \hfill
    \begin{minipage}{0.24\textwidth}
        \centering
        \includegraphics[width=\linewidth]{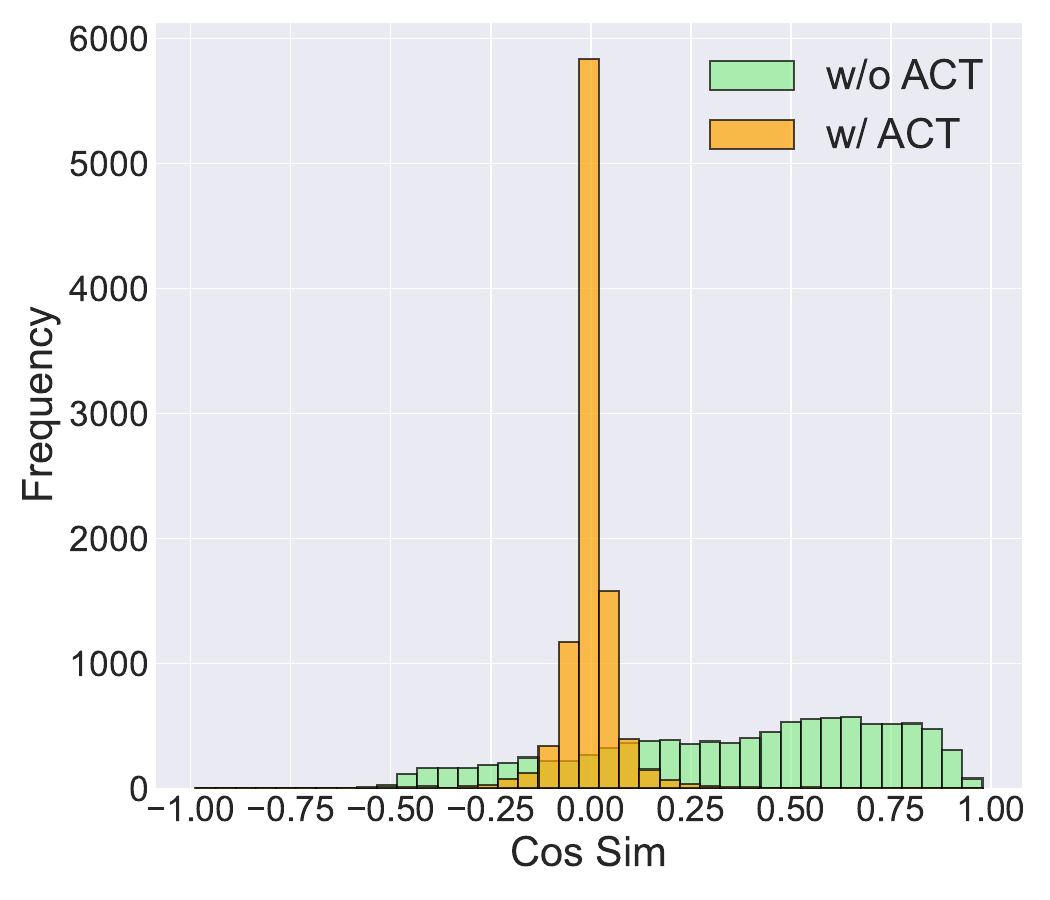}
        \subcaption{Distribution of cosine similarity for $\bm{\beta}^{(a)}$ and $\bm{\beta}^{(b)}$, with and without ACT.}
        \label{fig: cosb}
    \end{minipage}
    \caption{Similarity distributions of model parameters with and without ACT, computed on 10,000 random fingerprint pairs.}
    \label{fig: cos}
\end{figure}

\subsubsection{Analysis of the effectiveness of ACT}
We analyze why ACT is effective against collusion. Collusion attacks based on parameter interpolation rely on the principle of mode connectivity, which assumes that fingerprinted models reside on continuous low-loss manifolds~\cite{garipov2018loss}. The success of such attacks depends primarily on two factors: \textbf{(1)} high parameter similarity among fingerprinted models, and \textbf{(2)} the presence of a connected low-loss manifold enabling smooth interpolation.

ACT is designed to break this by applying user-specific, function-preserving transformations that remap the parameters of each user's model to a distant location in the parameter space. Although each transformed model remains functionally identical to its original version, the parameters for different users become highly dissimilar. Consequently, parameter interpolation of two or more ACT-protected models, which no longer share a connected low-loss path, yields parameters that fall outside any valid functional manifold. This interpolated point corresponds to a region in the loss landscape characterized by high loss values, leading to the observed failure to generate high-fidelity images.

\textbf{Parameter Similarity}. 
First, we analyze how ACT alters the similarity between model parameters. 
Fig.~\ref{fig: cos} illustrates the distributions of the cosine similarity between the normalization parameters of two models of users a and b,
averaged across 10,000 random fingerprint pairs $(\bm{m}^{(a)},\bm{m}^{(b)})$, both with and without ACT.
Specifically, Fig.~\ref{fig: cosa} shows the similarity between $\bm{\gamma}^{(a)}$ and $\bm{\gamma}^{(b)}$, while Fig.~\ref{fig: cosb} reports the same for $\bm{\beta}^{(a)}$ and $\bm{\beta}^{(b)}$. We can observe that without ACT, both fingerprinted parameters show high positive similarity, which makes the model susceptible to collusion attacks, as parameter interpolation between users can produce functional and colluded models.
In contrast, with ACT, the distributions are centered around zero. This demonstrates that ACT substantially reduces the similarity between fingerprinted parameters and enforces near-orthogonality between different user models.

\textbf{Loss Landscapes.}
We also study the influence of ACT on loss landscapes. Following~\cite{li2018visualizing}, we randomly sample two orthogonal directions to perturb PNM and compute the LPIPS loss between the reconstructed image and the input image. In Fig.~\ref{fig: land}, the first row presents the loss landscapes of four models with randomly selected fingerprints without ACT, while the second row shows those with ACT. Without ACT, the landscapes across different fingerprints exhibit high similarity, smooth variations, and continuous gradients.
In contrast, ACT induces substantially steeper surfaces, where loss values increase sharply within the same perturbation range.
These observations reveal that ACT substantially reshapes the loss landscape, making it steeper and more irregular, which affects the model's behavior in the presence of collusion attacks.

In the absence of ACT, the high parameter similarity ensures that an interpolated solution remains in a close neighborhood of the original models. Given the smoothness of loss surfaces, such a small deviation naturally incurs minimal loss, allowing colluding parties to remove fingerprints with little degradation in image quality. With ACT, the interpolation causes a large deviation in the parameter space, due to the near-orthogonality of the models. On the steep and irregular surfaces induced by ACT, such large deviations inevitably trigger large errors and severe quality degradation.

\subsection{Discussion}

\textbf{Practical Applicability.} 
LDM–based architectures are widely adopted in state-of-the-art models, such as SD3~\cite{esser2024scaling}. Although these models employ different diffusion backbones, they still rely on a decoder for image synthesis, which implies that our method can be applied to these cutting-edge models. For video generation models~\cite{zheng2024open}, our method can be extended to 3D convolutions. Moreover, recent model customization methods~\cite{ruiz2023dreambooth} mainly adjust the diffusion backbone while leaving the VAE unchanged, thereby avoiding any conflicts or degradation of fingerprints.

\textbf{Capacity.}
In our experiments, we randomly selected the fingerprint associated with each user. In principle, given a maximum number of users $M$ that the distributor wants to allocate, better performance can be achieved by choosing the $d$-bit fingerprints to be associated with every user by means of a suitable binary fingerprint code~\cite{tardos2008optimal}. In particular, the length of the fingerprint $d$ has an important impact on $M$. We employed a 48-bit fingerprint, yielding $2^{48}$ distinct fingerprints. 
To ensure safety and obtain some error-correction capability, e.g., ensuring that fingerprints from different users differ in at least 12 positions, according to the Hamming bound~\cite{macwilliams1977theory}, we can allocate up to $1.98\times 10 ^{7}$ unique models using a suitable code. Although this is only a bound, it indicates that the 48-bit length can already be sufficient to meet real-world requirements for fingerprinting.


\section{Conclusion}
\label{sec: Conclusion}
We introduced a novel, efficient, and robust framework for image diffusion model fingerprinting, which addresses the challenge posed by collusion attacks. By integrating fingerprint embedding directly into the model parameters through a PNM module, our method enables efficient and scalable creation of uniquely fingerprinted model copies. 
Strong fingerprint robustness against model-level modifications, and, notably, fine-tuning, is achieved by minimizing the fingerprinting loss against worst-case parameter perturbations.
The method incorporates a user-specific ACT inside the PNM, ensuring that any attempt to collude, multiple models result in a non-functional model, thus effectively preventing unauthorized redistribution. 
Future research will focus on optimizing user fingerprint allocation and on the evaluation of the performance in a real scenario using binary fingerprint codes. The optimization of ACT transformation strategies may also be considered.

\bibliographystyle{IEEEtran}
\bibliography{ref}

\begin{IEEEbiography}
[{\includegraphics[width=1in,height=1.25in, clip]{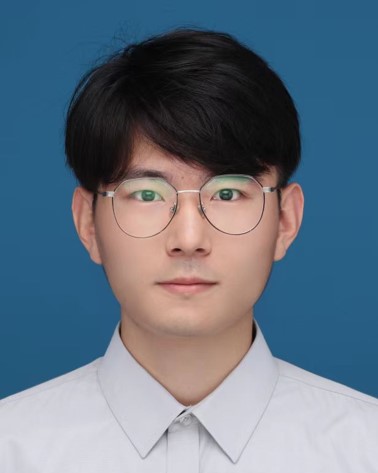}}]
{Jianwei Fei} (Member, IEEE) received the Ph.D. degree in information security from Nanjing University of Information Science and Technology in 2023. From 2024 to 2025, he was a Post-Doctoral Fellow at the University of Macau. He is currently a Post-Doctoral Fellow at the University of Florence, Italy. His research interests include the security of generative AI, proactive and passive attribution of generative models, and digital forensics.
\end{IEEEbiography}

\begin{IEEEbiography}
[{\includegraphics[width=1in,height=1.25in, clip]{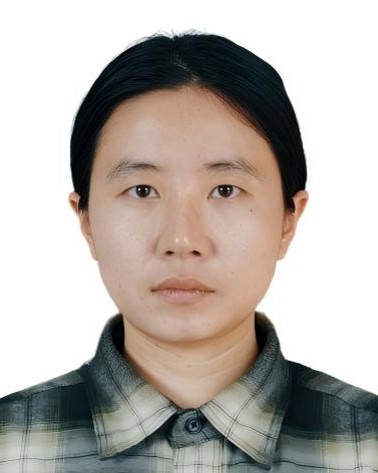}}]
{Yunshu Dai} (Student Member, IEEE) is currently pursuing the Ph.D. degree with the School of Cyber Science and Technology, Sun Yat-sen University. Her research interests include AI security, watermarking, and multimedia forensics.
\end{IEEEbiography}

\begin{IEEEbiography}
[{\includegraphics[width=1in,height=1.25in, clip]{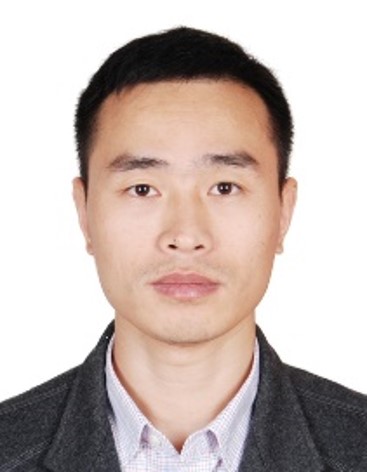}}]
{Zhihua Xia} (Member, IEEE) received his Ph.D. degree in computer science and technology from Hunan University, China, in 2011, and worked successively as a lecturer, an associate professor, and a professor with College of Computer and Software, Nanjing University of Information Science and Technology. He is currently a professor with the College of Cyber Security, Jinan University, China. He was a visiting scholar at New Jersey Institute of Technology, USA, in 2015, and was a visiting professor at Sungkyunkwan University, Korea, in 2016. He serves as a managing editor for IJAACS. His research interests include AI security, cloud computing security, and digital forensics. He has been a member of the IEEE since Mar. 1, 2014.
\end{IEEEbiography}

\begin{IEEEbiography}
[{\includegraphics[width=1in,height=1.25in,clip]{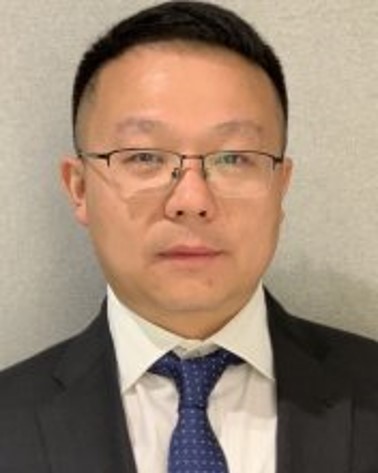}}]
{Xiaochun Cao} (Senior Member, IEEE) received the B.E. and M.E. degrees in computer science from Beihang University (BUAA), China, and the Ph.D. degree in computer science from the University of Central Florida, USA. After graduation, he spent about three years at ObjectVideo Inc., as a Research Scientist. From 2008 to 2012, he was a Professor at Tianjin University. He was a Professor at the Institute of Information Engineering, Chinese Academy of Sciences. He is a Professor at the School of Cyber Science and Technology, Sun Yat-sen University. He has authored and co-authored more than 200 journal and conference papers. He is a fellow of the IET. In 2004 and 2010, he was a recipient of the Piero Zamperoni Best Student Paper Award at the International Conference on Pattern Recognition. His Ph.D. dissertation was nominated for the University-Level Outstanding Dissertation Award. He was on the Editorial Board of IEEE TRANSACTIONS ON CIRCUITS AND SYSTEMS FOR VIDEO TECHNOLOGY and IEEE TRANSACTIONS on MULTIMEDIA. He is on the Editorial Board of IEEE TRANSACTIONS ON IMAGE PROCESSING and IEEE TRANSACTIONS ON PATTERN ANALYSIS AND MACHINE INTELLIGENCE.
\end{IEEEbiography}

\begin{IEEEbiography}
[{\includegraphics[width=1in,height=1.25in,clip]{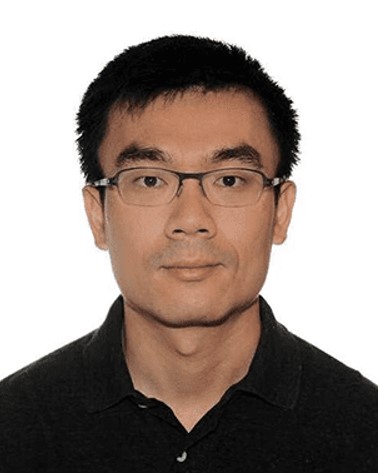}}]
{Jiantao Zhou} (Senior Member, IEEE) received the B.Eng. degree from the Department of Electronic Engineering, Dalian University of Technology, in 2002, the M.Phil. degree from the Department of Radio Engineering, Southeast University, in 2005, and the Ph.D. degree from the Department of Electronic and Computer Engineering, Hong Kong University of Science and Technology, in 2009. He held various research positions with the University of Illinois at Urbana-Champaign, Hong Kong University of Science and Technology, and McMaster University. He is now the Head and Professor with the Department of Computer and Information Science, Faculty of Science and Technology, University of Macau. His research interests include multimedia security and forensics, multimedia signal processing, artificial intelligence, and big data. He holds four granted U.S. patents and two granted Chinese patents. He has coauthored two papers that received the Best Paper Award at the IEEE Pacific-Rim Conference on Multimedia in 2007 and the Best Student Paper Award at the IEEE International Conference on Multimedia and Expo in 2016. He is serving as an Associate Editor for the IEEE TRANSACTIONS ON IMAGE PROCESSING, the IEEE TRANSACTIONS ON MULTIMEDIA, and the IEEE TRANSACTIONS ON DEPENDABLE AND SECURE COMPUTING. 
\end{IEEEbiography}

\begin{IEEEbiography}
[{\includegraphics[width=1in,height=1.25in,clip]{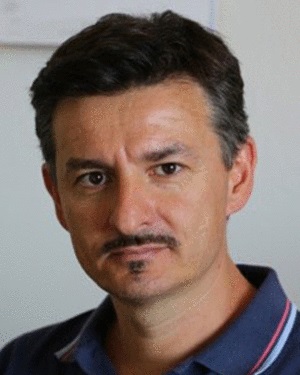}}]
{Alessandro Piva} (Fellow, IEEE) is currently an Associate Professor with the Department of Information Engineering, University of Florence. He is also the Head of the FORLAB, the Multimedia Forensics Laboratory, University of Florence. His research interests include information forensics and security, and image and video processing. In the first topic, he was interested in data hiding, signal processing in the encrypted domain, and image and video forensic techniques. In the second area, he was interested in the design of image and video processing and analysis techniques for cultural heritage, medical, and industrial applications. In the above research topics, he is the coauthor of more than 70 articles published in international journals and 130 papers published in international conference proceedings, with an H-index of 43 according to Scopus.
\end{IEEEbiography}

\begin{IEEEbiography}
[{\includegraphics[width=1in,height=1.25in,clip]{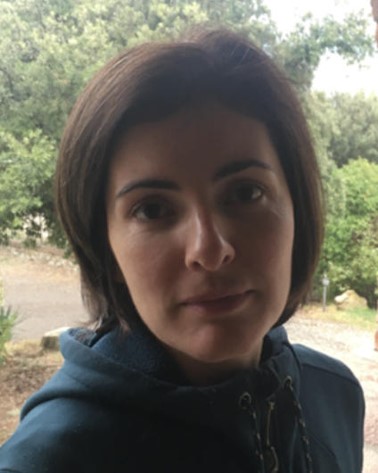}}]
{Benedetta Tondi} (Senior Member, IEEE) received the master’s degree (cum laude) in electronics and communications engineering and the Ph.D. degree in information engineering and mathematical sciences from the University of Siena, Siena, Italy, in 2012 and 2016, respectively. She is currently an Associate Professor with the Department of Information Engineering and Mathematics, University of Siena. From October 2014 to February 2015, she was a Visiting Student with the University of Vigo, Vigo, Spain. Since January 2019, she has been a member of the Information Forensics and Security Technical Committee of the IEEE Signal Processing Society. She is part of the IEEE Young Professionals and IEEE Signal Processing Society and a member of the National Inter-University Consortium for Telecommunications. She has been Area Chair of several IEEE conferences and workshops and  Technical Program Chair of ACM IHMMSEC 2022 and 2026. She is an Associate Editor  of IEEE SIGNAL PROCESSING LETTERS and a Senior Area Editor of IEEE TRANSACTIONS ON INFORMATION FORENSICS AND SECURITY.
\end{IEEEbiography}

\clearpage

\twocolumn[
\begin{center}
    \huge \bfseries Appendix for ``Efficient, Robust, and Anti-Collusion Fingerprinting of Image Diffusion Models''\par
    \vspace{2em}
\end{center}
]

\appendices

\section{Theoretical Proofs of Output Invariance for ACT}
\label{app:invariance_proofs}
We provide the detailed mathematical proofs demonstrating that 
Channel Permutation, Parameter Scaling, and Sign Flip, included in the proposed Anti-Collusion Transformations (ACT),  strictly preserve the input-output mapping of the Personalized Normalization Module (PNM).

\subsection{Proof of Channel Permutation Invariance}
\label{app:proof_cp}
We first formalize the invariance property of the channel-wise parameter permutation (CP).

\begin{theorem}[Channel Permutation Invariance]
\label{theorem:channel_perm}
For any permutation $\pi: [C] \to [C]$ and input feature map $\bm{F}^{(0)} \in \mathbb{R}^{C \times h \times w}$, the channel-wise parameter permutation transformation preserves the PNM output. That is, if $(\widetilde{\bm{W}}^{(1)}, \widetilde{\bm{W}}^{(2)}, \widetilde{\bm{\gamma}}, \widetilde{\bm{\beta}})$ are the permuted parameters defined by Eq.~\eqref{eq:cp1}, then
\begin{equation}
\begin{aligned}
\widetilde{\bm{F}}^{(2)}_c = \bm{F}^{(2)}_c, \quad \forall c \in [C].
\end{aligned}
\end{equation}
\end{theorem}

\begin{proof}
We analyze the forward computation under permuted parameters.
For the intermediate features after $\text{Conv}_1$ and normalization, $\forall c \in [C]$:
\begin{equation}
\begin{aligned}
\widetilde{\bm{F}}^{(1)}_c &= \widetilde{\bm{\gamma}}_c \cdot \sum_{i=1}^{C} \bm{F}^{(0)}_i \ast \widetilde{\bm{W}}_{c,i}^{(1)} + \widetilde{\bm{\beta}}_c \\
&= \bm{\gamma}_{\pi(c)} \cdot \sum_{i=1}^{C} \bm{F}^{(0)}_i \ast \bm{W}_{\pi(c),i}^{(1)} + \bm{\beta}_{\pi(c)} = \bm{F}^{(1)}_{\pi(c)}.
\end{aligned}
\end{equation}
For the final output after $\text{Conv}_2$:
\begin{align}
\widetilde{\bm{F}}^{(2)}_c &= \sum_{j=1}^{C} \widetilde{\bm{F}}^{(1)}_j \ast \widetilde{\bm{W}}_{c,j}^{(2)} = \sum_{j=1}^{C} \bm{F}^{(1)}_{\pi(j)} \ast \bm{W}_{c,\pi(j)}^{(2)}.
\end{align}
Since $\pi$ is a bijection, as $j$ ranges over $[C]$, $\pi(j)$ also ranges over $[C]$. Therefore:
\begin{align}
\widetilde{\bm{F}}^{(2)}_c = \sum_{j=1}^{C} \bm{F}^{(1)}_{\pi(j)} \ast \bm{W}_{c,{\pi(j)}}^{(2)} = \bm{F}^{(2)}_c.
\end{align}
Hence, the input-output mapping is preserved.
\end{proof}

\subsection{Proof of Scaling Invariance}
\label{app:proof_sc}
In this section, we show that the parameter scaling (SC) operation does not alter the overall functionality of the PNM.

\begin{theorem}[Scaling Invariance]
\label{theorem:scaling_invariance}
Let $\bm{F}^{(0)} \in \mathbb{R}^{C \times h \times w}$ be an input feature map and let the scaled parameters $(\widetilde{\bm{W}}^{(1)}, \widetilde{\bm{W}}^{(2)}, \widetilde{\bm{\gamma}}, \widetilde{\bm{\beta}})$ be defined as in Eq. \eqref{eq:scal1}, with constraints given by Eq. \eqref{eq:con1}. Then the output of the PNM remains unchanged: $\widetilde{\bm{F}}^{(2)} = \bm{F}^{(2)}$.
\end{theorem}

\begin{proof}
We examine the forward pass through the PNM under the scaled parameters. The intermediate features after $\text{Conv}_1$ and normalization are:
\begin{equation}
\small
\begin{aligned}
\widetilde{\bm{F}}^{(1)}_c 
&= \widetilde{\bm{\gamma}}_c \cdot \ \sum_{i=1}^{C} \bm{F}^{(0)}_i \ast \widetilde{\bm{W}}^{(1)}_{c,i} + \widetilde{\bm{\beta}}_c  \\
&= \bm{\alpha}_c^{(3)} \bm{\gamma}_c \cdot  \sum_{i=1}^{C} \bm{F}^{(0)}_i \ast \bm{\alpha}_c^{(1)} \bm{W}^{(1)}_{c,i} + \bm{\alpha}_c^{(4)} \bm{\beta}_c \\
&= \bm{\alpha}_c^{(1)} \bm{\alpha}_c^{(3)} \bm{\gamma}_c \cdot \sum_{i=1}^{C} \bm{F}^{(0)}_i \ast \bm{W}^{(1)}_{c,i} + \bm{\alpha}_c^{(4)} \bm{\beta}_c,
\end{aligned}
\end{equation}
$\forall c \in [C]$.
Then, the final output after the second convolution is
\begin{equation}
\small
\begin{aligned}
\widetilde{\bm{F}}^{(2)}_c 
&= \sum_{j=1}^{C} \widetilde{\bm{F}}^{(1)}_j \ast \widetilde{\bm{W}}^{(2)}_{c,j} \\
&= \sum_{j=1}^{C} \left( \bm{\alpha}_{j}^{(1)} \bm{\alpha}_{j}^{(3)} \bm{\gamma}_j \sum_{i=1}^{C} \bm{F}^{(0)}_i \ast \bm{W}_{j,i}^{(1)} + \bm{\alpha}_{j}^{(4)} \bm{\beta}_j \right) \ast \bm{\alpha}_{j}^{(2)} \bm{W}_{c,j}^{(2)} \\
&= \sum_{j=1}^{C} \bm{\alpha}_{j}^{(1)} \bm{\alpha}_{j}^{(2)} \bm{\alpha}_{j}^{(3)} \cdot \bm{\gamma}_j \sum_{i=1}^{C} \left( \bm{F}^{(0)}_i \ast \bm{W}_{j,i}^{(1)} \right) \ast \bm{W}_{c,j}^{(2)} \\
&\quad + \sum_{j=1}^{C} \bm{\alpha}_{j}^{(2)} \bm{\alpha}_{j}^{(4)} \cdot \bm{\beta}_j \ast \bm{W}_{c,j}^{(2)},
\end{aligned}
\end{equation}
$\forall c \in [C]$.
Applying the constraints from Eq.\eqref{eq:scal1}, where $\bm{\alpha}_{i}^{(1)} \bm{\alpha}_{i}^{(2)} \bm{\alpha}_{i}^{(3)} = 1$ and $\bm{\alpha}_{i}^{(2)} \bm{\alpha}_{i}^{(4)} = 1$ for $ \forall i$, we obtain:
\begin{equation}
\begin{aligned}
\widetilde{\bm{F}}^{(2)}_c &= \sum_{j=1}^{C} \left[\bm{\gamma}_j \sum_{i=1}^{C} \left(\bm{F}^{(0)}_i \ast \bm{W}_{j,i}^{(1)}\right) +  \bm{\beta}_j\right] \ast \bm{W}_{c,j}^{(2)} \\
&= \sum_{j=1}^{C} \bm{F}^{(1)}_j \ast \bm{W}_{c,j}^{(2)} = \bm{F}^{(2)}_c.
\end{aligned}
\end{equation}
Therefore, $\widetilde{\bm{F}}^{(2)} = \bm{F}^{(2)}$, completing the proof.
\end{proof}

\subsection{Proof of Sign Flip Invariance}
\label{app:proof_sf}
Building upon the scaling invariance, we briefly formalize the output preservation for the sign flip operation.

\textbf{Sign Flip (SF).}
We observed that the sign flip is a special case of parameter scaling where the scaling factors are constrained to $\{-1, 1\}$.
Then, from Theorem \ref{theorem:scaling_invariance}, output preservation is guaranteed.

\section{Security Analysis of ACT}
\label{app:security_analysis}
Having established the function-preserving nature of ACT, we present a formal security analysis and derive theoretical complexity bounds against collusion attacks. We first define the threat model and analyze the complexity required for an adversary to successfully carry out such an attack.

\noindent\textbf{Definition 1} (Attacker Model). \textit{We consider colluders who possess $K \ge 2$ fingerprinted model copies $\{\mathcal{M}^{(1)}, \dots, \mathcal{M}^{(K)}\}$ and have full knowledge of the model architecture and the ACT mechanism (White-box setting). However, the attacker does not have access to the private user-specific seeds used to generate the ACT keys. The goal of the attacker is to find a set of transformation parameters to align the models into a parameter space, such that their interpolation yields a model with high generation quality.}

The security of ACT depends on the difficulty of achieving \textit{Parameter Alignment} across different users. We formulate this hardness in the following theorem.

\textbf{Theorem 3} (Complexity of Parameter Alignment). \textit{Given two distinct models $\mathcal{M}^{(a)}$ and $\mathcal{M}^{(b)}$ protected by ACT with channel size $C$, the complexity to align their parameters for a successful collusion attack is $\Omega(C! \cdot 2^{2C})$.}

\textit{Proof.} Since the attacker aims to perform parameter averaging (or interpolation), they must find a transformation mapping $\mathcal{T}$ such that $\mathcal{T}(\theta^{(b)}) \approx \theta^{(a)}$ (or map both to a common space). Since the ACT keys are generated randomly and independently for each user, aligning two models is equivalent to guessing the relative transformation differences between them. The search space is determined by the combined transformations:

\begin{enumerate}
    \item \textbf{CP:} CP applies a random permutation $\pi : [C]\to[C]$. To align the channels of $\mathcal{M}^{(b)}$ with $\mathcal{M}^{(a)}$, the attacker must identify the correct permutation out of all possible mappings. The size of this search space is $|\mathcal{S}_{CP}| = C!$.
   
    \item \textbf{SF and SC:} For SF, the attacker must determine the relative sign differences for each channel. Since the sign flip vectors are composed of $2C$ independent binary elements (for the two convolution layers in PNM), there are $2^{2C}$ possible patterns. For SC,  the scaling factors are continuous values, hence, guessing the exact value is impossible for the attacker. Even if we relax this by quantizing the scaling factors into $Q$ levels ($Q>1$), the search space expands by a factor of $Q^{2C}$.
\end{enumerate}

Combining these factors, the minimum size of the hypothesis space that the attacker must explore is:
\begin{equation}
    |\mathcal{H}| \ge |\mathcal{S}_{CP}| \cdot |\mathcal{S}_{SF}| = C! \cdot 2^{2C}.
\end{equation}
Thus, the complexity is $\Omega(C! \cdot 2^{2C})$.
\hfill $\square$

\textbf{Remark.} For a typical configuration with $C=128$, the search space $|\mathcal{H}|$ is extremely large, rendering brute-force attacks computationally infeasible. Furthermore, verifying each guess requires a full forward pass to inspect the image quality, which introduces significant computational overhead. Consequently, ACT effectively prevents attackers from achieving the parameter alignment necessary for collusion.

\end{document}